\documentclass[Afour,sageh,times]{sagej}

\usepackage{moreverb,url}

\usepackage[colorlinks,
bookmarksopen,
bookmarksnumbered,
citecolor=red,
urlcolor=red]{hyperref}
\usepackage{amsthm}
\usepackage{amssymb}
\usepackage{mathtools}
\usepackage{mathrsfs}
\usepackage[font=small, justification=justified]{caption}
\usepackage{booktabs}
\usepackage{multirow}
\usepackage[table]{xcolor}
\usepackage[ruled,linesnumbered]{algorithm2e}
\usepackage{algpseudocode}
\usepackage{subfigure}
\usepackage{float}
\usepackage{bbm}
\usepackage{soul, xcolor}
\usepackage{xcolor}
\usepackage[title]{appendix}
\usepackage{diagbox}
\usepackage{threeparttable}
\usepackage{arydshln}
\usepackage{natbib}
\setcitestyle{authoryear,round}
\usepackage{mathabx}
\usepackage{siunitx}

\theoremstyle{definition}
\allowdisplaybreaks

\newtheorem{theorem}{Theorem}
\newtheorem{lemma}{Lemma}
\newtheorem{assumption}{Assumption}
\newtheorem{remark}{Remark}
\newtheorem{definition}{Definition}

\newcommand\BibTeX{{\rmfamily B\kern-.05em \textsc{i\kern-.025em b}\kern-.08em
		T\kern-.1667em\lower.7ex\hbox{E}\kern-.125emX}}

\def\volumeyear{2022}

\begin{document}
	\setstcolor{red}
	
	\runninghead{Lu \textit{et al}.}
	
	\title{Trajectory Generation and Tracking Control for Aggressive Tail-Sitter Flights}
	
	\author{Guozheng Lu, Yixi Cai, Nan Chen, Fanze Kong, Yunfan Ren
		and Fu Zhang}

	\affiliation{Department of Mechanical Engineering, The University of Hong Kong.}
	\corrauth{Fu Zhang, Department of Mechanical Engineering, The University of Hong Kong, HW 7-18,  Pokfulam, Hong Kong.}
	
	\email{fuzhang@hku.hk}

	\begin{abstract}
		We address the theoretical and practical problems related to the trajectory generation and tracking control of tail-sitter UAVs. Theoretically, we focus on the differential flatness property with full exploitation of actual UAV aerodynamic models, which lays a foundation for generating dynamically feasible trajectory and achieving high-performance tracking control. We have found that a tail-sitter is differentially flat with accurate (not simplified) aerodynamic models within the entire flight envelope, by specifying coordinate flight condition and choosing the vehicle position as the flat output. This fundamental property allows us to fully exploit the high-fidelity aerodynamic models in the trajectory planning and tracking control to achieve accurate tail-sitter flights. Particularly, an optimization-based trajectory planner for tail-sitters is proposed to design high-quality, smooth trajectories with consideration of kinodynamic constraints, singularity-free constraints and actuator saturation. The planned trajectory of flat output is transformed  into state trajectory in real-time with optional consideration of wind in environments. To track the state trajectory, a global, singularity-free, and minimally-parameterized on-manifold MPC is developed, which fully leverages the accurate aerodynamic model to achieve high-accuracy trajectory tracking within the whole flight envelope. The proposed algorithms are implemented on our quadrotor tail-sitter prototype, ``Hong Hu", and their effectiveness are demonstrated through extensive real-world experiments in both indoor and outdoor field tests, including agile SE(3) flight through consecutive narrow windows requiring specific attitude and with speed up to \SI{10}{m/s}, typical tail-sitter maneuvers (transition, level flight and loiter) with speed up to \SI{20}{m/s}, and extremely aggressive aerobatic maneuvers (Wingover, Loop, Vertical Eight and Cuban Eight) with acceleration up to \SI{2.5}{g}. The video demonstration is available at \url{https://youtu.be/2x_bLbVuyrk}.
	\end{abstract}
	
	
	\maketitle
	
	\renewcommand*{\thefootnote}{\arabic{footnote}}
	
	\section{Introduction}
	\label{sec_intro}
	A tail-sitter unmanned aerial vehicle (UAV) is a type of vertical takeoff and landing (VTOL) flying machine that takes off and lands vertically on its tail while tilts the entire airframe in a near horizontal attitude for forward flight. Its hybrid fixed-wing and rotary-wing design combines advantages of the VTOL capability, aerodynamic efficiency, and hence extends the power endurance and flight range. Compared to other hybrid designs of VTOL UAVs, like tilt-rotors \citep{carlson2014hybrid, ozdemir2014design}, tilt-wings \citep{ccetinsoy2011design}, rotor-wing \citep{mckenna2007one}, and dual-systems \citep{park2014arcturus, gu2017development}, tail-sitters have rotors fixed to the wing and use their thrust in all flight conditions, leading to a mechanically simple, lightweight and efficient airframe configuration, which is particularly important for small-scale, low-cost, portable UAVs. Such UAVs hold immense potentials for a wide range of industrial and civil applications, such as infrastructure inspection, geological surveying, environment mapping, and  post-disaster search and rescue. These exciting opportunities have attracted intensive research interests and led to the development of a variety of tail-sitter UAV prototypes, such as the single-propeller configuration \citep{frank2007hover, wang2017design, de2018design}, the shoulder-mounted twin-engine configuration \citep{bapst2015design, ritz2017global, sun2018design}, and the quadrotor configuration \citep{oosedo2013development, gu2018coordinate}.
	
	 To accommodate the escalating demand of  real-world applications, tail-sitter UAVs must be able to execute highly aggressive  maneuvers, including forward transition to level flight, back transition break, and quickly bank turns. The agile maneuverability is crucial for the UAV to navigate at high-speed through obstacle-dense environments.  Unlike conventional airplanes that fly in open space,  tail-sitter UAVs are subjected to more challenging flight conditions of fast-varying speed and attitude, asking a holistic design of the trajectory generation and tracking control, where the former aims to plan a smooth, dynamically-feasible, and collision-free trajectory and the latter should track the planned trajectory with small errors. 
	
	While the planning and control of multicopter UAVs have been comprehensively studied by leveraging the differential flatness property of the systems \citep{mellinger2011minimum, faessler2017differential}, thus stimulating a wealth of practical applications, like flying through narrow gaps \citep{mellinger2012trajectory, falanga2017aggressive, ren2022online}, perching on structures \citep{mellinger2012trajectory, hang2019perching}, autonomous safe navigation \citep{shen2011autonomous, zhou2019robust, zhang2020falco}, and drone racing \citep{foehn2021time},  the equivalent techniques for tail-sitter UAVs are relatively underdeveloped. The  differential flatness for tail-sitters,  which resolves the full states and inputs of the system from finite flat outputs and their derivatives, has not been rigorously investigated. A significant hurdle confronting this task lies in the complex nonlinear aerodynamics inherent to tail-sitter UAVs. While the wings of a tail-sitter can produce the desired lift force to enhance power efficiency, they also introduce highly nonlinear aerodynamic forces into the system dynamics.  Unlike fixed-wing airplanes that are primarily confined to a conservative level flight regime where the wing aerodynamics are well understood as being linear, tail-sitters usually operate within a large flight envelope with a wide range of angle of attack (AoA), where the wing aerodynamics exhibit extreme nonlinearity. Consequently, the study on differential flatness, as well as high-precision planning and control of tail-sitter UAVs that fully exploit aerodynamic models are significantly complicated and still remains an open question.
	
	Besides the theoretical difficulty, trajectory generation and tracking of tail-sitter UAVs are also confronted with  practical challenges.  For example, during outdoor long-range missions, a tail-sitter UAV often suffers from  model uncertainties and considerable wind disturbances. Other constraints such as actuator saturation, sensor noise and limited onboard computation resource also ask for high robustness and computation efficiency of the designed algorithms. 
	
	\subsection{Contributions} 
	In this work, we address the challenge of high-quality trajectory generation and high-performance tracking control of tail-sitters by leveraging the differential flatness property, aiming to enable agile tail-sitter flights within the whole envelope in real-world environments. Specifically, our contributions are outlined as follows.
	\begin{itemize}
		\item [1)]  We show that the tail-sitter is differentially flat in the coordinated flight condition, in considering the actual aerodyanamic model without any simplifications. 
		
		\item [2)] Based on the differential flatness, we develop an optimization-based trajectory generation  method  enabling aggressive flights while taking account of actuator constraints, flight time, dynamical feasibility, and singularity conditions in coordinated flight.
		
		\item [3)]   We propose a two-stage control architecture.  The first stage transforms the planned flat-output trajectory into a state-input trajectory while compensating wind effect and treating singularities. The second stage is a real-time state trajectory tracking controller.
		
		\item [4)] For the second stage, we develop a global, model-based, minimally-parameterized and singularity-free  model predictive control (MPC) for trajectory tracking within the entire tail-sitter flight envelope. 
		
		\item [5)] We demonstrate and validate our algorithms via extensive real-world experiments on an actual quadrotor tail-sitter prototype in both indoor and outdoor environments.  To our best knowledge, it is the first tail-sitter demonstration of flying through narrow tilted windows and outdoor aerobatics.
		
	\end{itemize}
	
	\subsection{Outline}
	The outline of the rest of the paper is as follows. Section \ref{sec_related_work} reviews the related work. The system dynamics including the aerodynamic model are introduced in Section \ref{sec_flight_dyn}. The fundamental property of differential flatness is proved in Section \ref{sec_diff_flat}. Section \ref{sec_ctrl_archi} describes the system architecture including high-level trajectory generation and tracking, and low-level control. Section \ref{sec_traj_gen} presents the optimization-based trajectory generation and its solver. Section \ref{sec_ctrl} derives the error-state dynamics along the reference trajectory leading to an on-manifold MPC. Section \ref{sec_experiment} presents  real-world experiments validating our approach. Finally, Section \ref{sec_diss_and_con} concludes the paper with  extensions and limitations.
	
	\section{Related work}
	\label{sec_related_work}
\begin{table*}[t!]
    \centering
    \small
    \setlength\tabcolsep{4.5pt}
    \caption{Comparison of the stat-of-the-art global control methods for tail-sitter UAVs. }
    \begin{threeparttable} 
	\begin{tabular}{|l l l l l l l|}
		\hline 
		\multirow{2}{*}{Study} & \multirow{2}{*}{Methodology} & \multirow{2}{*}{Aerodynamic}  & \multirow{2}{*}{Singularity} & \multirow{2}{*}{Flight} &  \multirow{2}{*}{Wind} & \multirow{2}{*}{Demo}\\
		& & Model &   &Condition  & Compensation & Flights \\ 
		\hline 
		\multirow{2}{*}{Ours} & \multirow{2}{*}{MPC $\&$} & \multirow{2}{*}{Classic} & \multirow{2}{*}{Specific Airspeed} & \multirow{2}{*}{Coordinated}  & \multirow{2}{*}{On Reference} & \multirow{2}{*}{$*,\dagger, \ddagger$}\\ 
		& Differential Flatness &  &   &  &  &  \\
		&  &  &   & &  &\\ [-2em]
		\multirow{2}{*}{\cite{tal2022global}} & \multirow{2}{*}{Cascaded PD $\&$ INDI $\&$ } & \multirow{2}{*}{$\phi$-Theory} & \multirow{2}{*}{Specific Airspeed} & \multirow{2}{*}{No Restriction}  & \multirow{2}{*}{On Input }  &\multirow{2}{*}{$*, \ddagger$}\\
		& Differential Flatness &  &   &  &   &  \\	
		&  &  &   & &  &\\ [-2em]
		\multirow{2}{*}{\cite{lustosa2017phi}} & \multirow{2}{*}{Scheduled LQR } & \multirow{2}{*}{$\phi$-Theory} & \multirow{2}{*}{None} & \multirow{2}{*}{No Restriction} & \multirow{2}{*}{None}  &\multirow{2}{*}{$*$}\\
		&  &  &   & &  &\\ [-1em]
		\multirow{2}{*}{\cite{ritz2017global}} & \multirow{2}{*}{Cascaded PID} & \multirow{2}{*}{Classic } & \multirow{2}{*}{Specific Airspeed} & \multirow{2}{*}{Coordinated} & \multirow{2}{*}{None}  &\multirow{2}{*}{$*$}\\
		&  &  &  & & &\\ [-1em]
		\multirow{2}{*}{\cite{smeur2020incremental}} & \multirow{2}{*}{INDI} & \multirow{2}{*}{Quasi-Static} & \multirow{2}{*}{Euler Angle} & \multirow{2}{*}{Not Specified} & \multirow{2}{*}{On  Input}  &\multirow{2}{*}{$*$}\\
		&  &  &  & & &\\ [-1em]
		\multirow{2}{*}{\cite{cheng2022transition}} & \multirow{2}{*}{Adaptive Control} & \multirow{2}{*}{-} & \multirow{2}{*}{Euler Angle} & \multirow{2}{*}{Not Specified} & \multirow{2}{*}{None}  &\multirow{2}{*}{$*$}\\
		& &  &   &   &  & \\
		\hline
	\end{tabular}
        \small 
        Symbols $*,\dagger, \ddagger$ indicate three different demonstrated maneuvers:  the typical maneuvers $*$ include common tail-sitter flights such as transition, level flight and loiter; the $SE(3)$ maneuvers $\dagger$ denote a whole-body flying motion with specified pose and velocity; the aerobatic maneuvers $\ddagger$ denotes aggressive maneuvers with large attitude variation and flight speeds. 
    \end{threeparttable} 
\label{tab_control_comparison}
 \vspace{-3mm}
\end{table*}

	\subsection{Tail-sitter control}	
	There is a wealth of research on tail-sitter control which can be generally categorized into two main strategies: the separated control strategy, which consists of several isolated controllers designed for respective flight phases, and the global control that regulates the vehicle maneuvers within the entire envelope under a unified controller. We will discuss these control approaches in the following content.
	
	Since the tail-sitter dynamics reduce to a rotary-wing model and a fixed-wing model in low-speed vertical flight and high-speed level flight respectively, separated control methods \citep{frank2007hover, oosedo2013development, lyu2017hierarchical} usually divide the flight process into three phases -- vertical flight (including takeoff, landing and hovering), transition and level flight -- and design controllers separately for each phase.  The vertical flight dynamics are linearized at the stationary hovering equilibrium \citep{frank2007hover, matsumoto2010hovering, lyu2017hierarchical}, and controlled by means of established control methods for quadrotors, such as loop-shaping, \citep{zhou2018frequency}, robust control \citep{lyu2018disturbance}, and MPC \citep{li2018development}. The level flight controllers are usually borrowed from the fixed-wing airplanes and UAVs, such as the total energy control system \citep{lambregts1983vertical} which is widely used in the open-source autopilot PX4 \citep{meier2015px4}.                  
	
	Transition control is a key challenge for the separated control strategy and there is rich literature addressing this issue. The aerodynamics become highly nonlinear during the transition due to the dramatic change of AoA. An intuitive linear control method is to feed a pre-designed profile of linearly decreasing or increasing pitch angle to the attitude controller with a constant altitude command \citep{verling2016full, lyu2017design}, forcing the vehicle to pitch down or up until triggering the mode-switching condition. Because of the nonlinear dynamics, gain-scheduling techniques \citep{kita2010transition, jung2012development} could be applied to enhance the stability margin. However, this linear method is not always dynamically feasible and usually  results in undesired altitude deviation. The altitude control performance can be improved either by a well-designed transition planner \citep{naldi2011optimal, oosedo2017optimal, wang2017model, li2020transition}  using accurate aerodynamic models or  a sophisticated altitude controller, such as iterative learning control \citep{xu2019acceleration}. A  limitation of these methods is  their focus on the altitude and pitch control to transit a tail-sitter to the level flight phase,  often neglecting the lateral motion or any maneuvers (e.g., bank turns) during the transition, which are necessary for obstacle avoidance in low-altitude cluttered environments. 

	To sum up, although the separated strategy eases the controller design and has widespread use in practice, the controller switching usually causes unexpected transient response, thereby degrading  control performance. Given that a tail-sitter would frequently enter and exit the transition regime (i.e., a specified range of pitch angle and airspeed) when performing aggressive maneuvers, a global control strategy that uses a unified system model and control law serving for the whole envelope without switching among different flight phases (e.g., hovering, transition, and level flight) is  more preferable. This direction has prompted a significant amount of research. 
	
	Model-free global control methods for tail-sitters do not rely on vehicle aerodynamic models, but manage to approximate the aircraft dynamics locally and stabilize the local approximation by using linear theory. For example, \cite{barth2020model} proposed a cascaded model-free global control framework based on quasi-static assumptions. The vehicle system is decoupled, approximated and estimated locally as a group of second-order piece-wise linear systems, and thus the reference thrust and attitude can be solved from the desired body velocity.  Similarly, \cite{cheng2022transition} employed an adaptive control law with an IMU-based thrust-attitude decoupling method, assuming zero gradient for the aerodynamic forces. Although model-free methods can estimate and compensate the unmodeled aerodynamics, they apply small control input  at each step to maintain the effective region of the state-input linearization. These approaches are not ideal to agile flights requiring more aggressive control inputs. Consequently, the control performance degrades (i.e., altitude error exceeds \SI{1}{m} during transition) during highly agile maneuvers as demonstrated in \cite{barth2020model, cheng2022transition}.
	
	To further improve the control performance, model-based global controllers of varying sophistication have been proposed.  For instance, \cite{ritz2017global} used a classic aerodynamics model to derive the desired attitude and thrust from the acceleration commands,  by specifying the coordinated flight condition. To enable real-time implementation on a low-cost microcontroller, the aerodynamic model was simplified based on first-principles derivations, leading to considerable tracking errors. \cite{zhou2017unified} also calculated the desired attitude but by solving a non-convex optimization using an accurate aerodynamic model obtained from wind tunnel test.  However, this controller is computational demanding, which precludes real-time implementation. When the airspeed is zero, the definitions of angle of attack and sideslip angle become invalid, introducing singularity into the classic aerodynamic model used by these two research. There are studies employing alternative aerodynamic characterizations to avoid this singularity.   \cite{pucci2013nonlinear} transformed a 2-dimension (2-D) planner VTOL (PVTOL) vehicle into an orientation-independent model, separating the computation of the vehicle thrust and orientation, thereby leading to a unified controller design \citep{pucci2012flight}. The author also derived the conditions, \textit{spherical equivalency}, that  airfoil aerodynamic characteristics must satisfy for the transformation to hold.  \cite{lustosa2017phi} proposed a polynomial-like global aerodynamic parameterization, termed as the $\phi$-theory model, and developed a linear quadratic regulator (LQR) based on the model. Their experiment results show that the LQR gain must be scheduled during the transition to avoid the instability in pitch angle caused by the model errors of $\phi$-theory.  Alternatively,  \cite{smeur2020incremental} designed a global incremental nonlinear dynamic inversion (INDI) controller by linearizing the system at the current control inputs. To design the INDI controller, it requires knowing the current aerodynamic force (and moment) applied to the UAV and its gradient with respect to (w.r.t.) the control input increment (pitch angle and velocity): the former one is obtained from inertial measurement units (IMUs), which suffer from either large measurement noise caused by constant propeller rotation or considerable filter delay; the latter is derived from a simple, heuristic aerodynamic model at a quasi-static condition where the flight path angle is zero. More recently, \cite{tal2021global} integrated the aforementioned $\phi$-theory model and INDI technique into a global cascade PD controller  applicable to both coordinated and uncoordinated flight. They also introduced  feedforward jerk and yaw rate to improve the tracking performance and demonstrated indoor aerobatics \citep{tal2022global}. Compared to the previous INDI method \citep{smeur2020incremental} with an over-simplified aerodynamic model, the $\phi$-theory model used in \cite{tal2021global, tal2022global} can significantly increase the control accuracy. However, similar issues, i.e., significant measurement noise or filter delay still persist in  INDI-based methods. Moreover, the $\phi$-theory models have limited fitting capability, leading to larger model errors  compared with the classic model, as confirmed by the authors in \cite{lustosa2019global}. To sum up, the existing mode-based global control methods typically make compromise between model fidelity and computational load. While  high-fidelity models are costly and impractical for real-time implementation, simplified models are relatively easy to be estimated from limited experimental data, but  tend to degrade the control accuracy to varying extents.

	Our method aims to fully exploit the UAV's actual aerodynamics to achieve high-accuracy and real-time control performance. Compared to the existing works, our proposed control scheme has the following advantages. \textbf{1)} Existing works either give up the vehicle aerodynamic model (e.g., model-free methods \citep{barth2020model, cheng2022transition}) or compromise to simplified models (e.g., simplified classic model \citep{ritz2017global}, spherical equivalence model \citep{pucci2013nonlinear}, steady-level-flight model \citep{smeur2020incremental} and $\phi$-theory model \citep{lustosa2017phi, tal2022global}), while our proposed controller leverages classic aerodynamic models without any simplification on its aerodynamic coefficients. The use of high-fidelity aerodynamic model is crucial to achieve higher control accuracy. \textbf{2)} Existing controllers either ignore wind effect in the environment (e.g. \cite{ritz2017global, lustosa2017phi}), or compensate the disturbance through incremental control updates from increased control error (e.g. \citep{smeur2020incremental, tal2022global}), while our proposed approach incorporates wind effect by adjusting the reference trajectory (e.g., attitude) to maintain coordinated flight based on the differential flatness, and then tracks the adjusted trajectory in real time. Given the  considerable aerodynamic efficiency loss of tail-sitter in windy conditions \citep{vourtsis2023wind}, our proposed feedforward strategy compensates the wind effect in an pre-emptive way before the control error actually accumulates. \textbf{3)} Existing controllers \citep{ritz2017global, tal2022global} simultaneously track trajectories and process singularities, while our work decouples singularities from the tracking controller, by the two-statge architecture. Such separation isolates the singularity treatment from the state tracking controller. \textbf{4)} An on-manifold MPC is proposed for trajectory tracking in high accuracy. MPC tracks full states by solving a finite-horizon optimization at each step to yield the best future behavior based on the system model \citep{borrelli2017predictive}. Its predictive nature that exploits the information of the future reference trajectory, contributes to a higher control bandwidth for trajectory tracking. In contrast, existing works \citep{ritz2017global, barth2020model, smeur2020incremental, tal2022global} usually track the position trajectory in a cascaded control structure (e.g., a position controller followed by an attitude controller), which simplifies the outer loop design, but simultaneously constrains the outer loop's bandwidth. Admittedly, MPC is more computationally demanding and its  convergence is challenging to guarantee, but its predictive nature and constraint handling capability have led to a wealth of successful robotic applications, such as the leading-edge  Boston Dynamics Atlas humanoid robot \citep{Marion2021Flipping}, drone racing \citep{foehn2021time} and aerobatics \citep{kaufmann2020RSS, lu2022manifold}.  In summary, a comparison of our work with those existing state-of-the-art global controllers is presented in Table \ref{tab_control_comparison}. 

    \subsection{Tail-sitter trajectory generation}
	Depending on the control strategy reviewed above, there are different trajectory generation algorithms for tail-sitters in literature. For separated control strategies, trajectories are generated separately for each phase. When the tail-sitter dynamics reduce to a rotary-wing model in low-speed vertical flight, well-established trajectory generation methods for quadrotors (or multicopters) \citep{mellinger2011minimum, mueller2015computationally} are applicable directly. Trajectory planners for quadrotors can be also applied for high-level autonomy, such as obstacle avoidance and autonomous navigation. Similarly, traditional fixed-wing planners \citep{park2004new, chitsaz2007time} can be adapted for tail-sitter in level flights. For example, the L1 guidance proposed by \cite{park2004new} has been widely used in prototype verification \citep{frank2007hover, jung2012development, verling2016full} and commercial Autopilots \citep{meier2015px4} for tail-sitter level flights. 
	
	Generating a transition trajectory between vertical and level flights is relatively challenging due to the nonlinear aerodynamics during the transition. The intuitive linear transition method, which designs linearly increasing or decreasing pitch angle and constant altitude command \citep{verling2016full, lyu2017design} as mentioned before, does not consider the dynamical feasibility, thus requiring a lot of empirical trials and errors. To incorporate dynamic feasibility,  trajectory generation is usually formulated into nonlinear optimization problems subject to different control objectives and constraints. For instance, \cite{kita2010transition} calculated a pitch angle and thrust profile achieving the shortest transition time. \cite{naldi2011optimal} considered a minimum-time and minimum-energy optimal transition problem, while \cite{oosedo2017optimal} and \cite{li2020transition} respectively minimized the flight time and  energy to maintain a constant altitude during the transition flight. However, solving these non-convex optimization problems is computationally expensive, preventing from onboard implementation and online replanning. These methods are also confined to straight-line transition that cannot be extended to other maneuvers like transition with banked turns to avoid obstacles.  Simplified dynamic models like the point-mass model \citep{mcintosh2022transition} can be used to expedite obstacle-free planning, but again, the dynamical feasibility is omitted. Overall,  existing separated trajectory generation approaches  generate  simple trajectories with  limited maneuverability, making them only suitable for flights in open areas. The aerodynamic simplification and kinodynamic limitation  prevent them from being extended to dynamically feasible and agile flights in cluttered environments.
 
	Compared to the separated trajectory generation, designing a dynamically feasible trajectory that spans the entire envelope is a significantly more complex task because the tail-sitter  is an under-actuated system with extremely nonlinear aerodynamics. For under-actuated mechanical systems, such as tail-sitters, the differential flatness is an essential property that can significantly ease trajectory generation. If a dynamic system is differentially flat, its full states and inputs can be determined by algebraic functions of flat outputs and their derivatives \citep{fliess1995flatness, murray1995differential}. This property simplifies the trajectory generation problem to a set of algebraic operations  in the flat-output space. This is a significant reduction in complexity compared to the state-space planning, which usually has to take into account on-manifold kinematic constraints. For example, the differential flatness property of quadrotors \citep{mellinger2011minimum, faessler2017differential} has been thoroughly studied and enabled a variety of applications in trajectory planning.

	Research on the differential flatness of tail-sitter UAVs is  scarce due to the extremely complicated, nonlinear aerodynamics mentioned above. Early research based on simplified models can trace back to 1990s. \cite{hauser1992nonlinear, martin1996different} studied the differential flatness and control of a simple 2-D PVTOL aircraft. \cite{van1998rapid} simply considered the transition dynamics as a nominal flat system where the aerodynamics are treated as perturbations. Recently,  \cite{tal2021global}  showed the differential flatness based on the $\phi$-theory aerodynamic model. The vehicle position and yaw angle are chosen as flat outputs, which allow for a global framework of trajectory optimization \citep{tal2022aerobatic}. The optimization could then be solved efficiently in the flat-output space, and the flatness transformation provides state projections (e.g., mapping acceleration to attitude) in a cascaded controller. Yet, this framework has certain theoretical limitations. First, the differential flatness is built on the coarse $\phi$-theory aerodynamic model, the model errors of which degrade the trajectory quality and the resultant control performance. Second, the $\phi$-theory model assumes a windless condition that only considers the vehicle attitude and velocity w.r.t. the fixed inertial frame, rather than the aerodynamic angles and airspeed. Lastly, the method must assume that the vehicle has no body or vertical rudder that produce side forces. Hence, this differential flatness property is not applicable to outdoor environments commonly with external winds or more general tail-sitter airframes.
	
	Contrasted with early studies based on simplified 2-D models \citep{hauser1992nonlinear, martin1996different},  our work considers the full 3D model of a real tail-sitter UAVs. Furthermore, in comparison to recent research that used simplified aerodynamic model, such as the spherical equivalence model \citep{pucci2013nonlinear} and the $\phi$-theory model \citep{lustosa2017phi}, or that required particular airframe, such as configurations without vertical surfaces necessitated by \citep{tal2022aerobatic}, we prove the differential flatness property on accurate aerodynamic models and more general tail-sitter airframes. Based on the proved differential flatness, we propose a systematic trajectory generation framework for tail-sitter UAVs. High-quality trajectories are optimized subjecting to actuator constraints, flight time and dynamical feasibility. 	

	It is interesting to note that, both \cite{tal2022global}, which assumes no vertical surfaces but with uncoordinated flight, and ours, which assumes coordinated flight, eventually lead to the same effect of avoiding lateral forces. The lateral force would dramatically complicate the solving of the UAV state (i.e., attitude and thrust) due to the highly nonlinear aerodynamic forces.  \cite{zhou2017unified} solves these highly nonlinear constraints by leveraging numerical approach, leading to high computational complexity not suitable for real-time implementation. Instead, avoiding such lateral force could effectively isolate and solve the angle of attack in our work (or Pitch angle in \cite{tal2022global}), hence the rest UAV states.

	\section{Flight dynamics}
	\label{sec_flight_dyn}
	This section introduces the dynamic models that describe the motion of tail-sitters. We define coordinate frames for tail-sitter modeling, trajectory generation, and tracking control in Section \ref{sec_coordinate_def}. The dynamic model of the tail-sitter is presented in Section \ref{sec_airframe_dyn} and Section \ref{sec_aerodyn} introduces the classic aerodynamic models.
	
	\subsection{Coordinate frames}
	\label{sec_coordinate_def}
	As shown in Fig. \ref{fig_coordinate}, the definition of coordinate frames follows the convention of traditional fixed-wing aircraft. The world frame $\{\mathbf{O}, \mathbf{x}, \mathbf{y}, \mathbf{z}\}$ denoted North-East-Down (NED), is considered as the inertial frame. The body frame $\{\mathbf{O}_b, \mathbf{x}_b, \mathbf{y}_b, \mathbf{z}_b\}$ is defined as Forward-Right-Down where the body axis $\mathbf{x}_b$ points along the nose of the aircraft and $\mathbf{O}_b$ is the vehicle center of gravity.  
	\begin{figure}[t!] 
		\centering
		\includegraphics[width=0.8\linewidth]{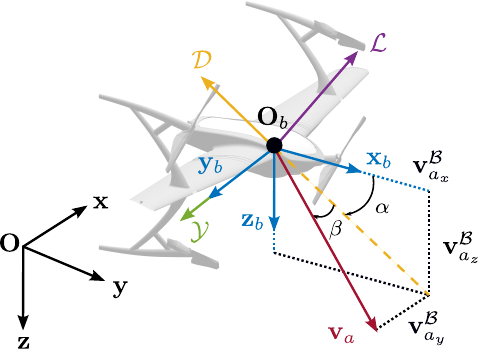}
		\caption{Coordinate frames: the world frame, body frame and aerodynamic forces.} 
		\label{fig_coordinate}
		\vspace{-3mm}
	\end{figure}
	
	\subsection{Airframe dynamics}
	\label{sec_airframe_dyn}
	We view the whole body of the tail-sitter as a rigid body.  Referring to the Newton-Euler equations, the translational and rotational dynamics of the aircraft is modeled as follows:
    \begin{subequations}
    \vspace{-4mm}
	\begin{align}
        \dot{\mathbf{p}} &= \mathbf{v} \label{e_translation_kin} \\
		\dot{\mathbf{v}} &= \mathbf{g} + a_T\mathbf{R}\mathbf{e}_1 + \frac{1}{m}\mathbf{R}\mathbf{f}_a  \label{e_translation_dyn} \\
		\dot{\mathbf{R}} &= \mathbf{R}\lfloor\boldsymbol{\omega}\rfloor \label{e_rotation_kin} \\
		\mathbf{J}\dot{\boldsymbol{\omega}} &= \boldsymbol{\tau} + \mathbf{M}_a - \boldsymbol{\omega} \times \mathbf{J} \boldsymbol{\omega} 
		\label{e_rotational_dyn}
	\end{align}
	\label{e_vehicle_dyn}
    \end{subequations}
	where $\mathbf{p}$ and $\mathbf{v}$ are respectively the vehicle position and velocity in the inertial frame, $\boldsymbol{\omega}$ is the angular velocity in the body frame, $\mathbf{R}$ denotes the rotation from the inertial frame to the body frame, $m$ is the total mass of the aircraft, $\mathbf{J}$ is the inertia matrix and $\mathbf{g} = [0 \ 0\ 9.8]^T$ is the gravity vector in the inertial frame. $a_T$ and $\boldsymbol{\tau}$ denote the thrust acceleration scalar and control moment vector produced by actuators (e.g., four motors for a quadrotor tail-sitter). $\mathbf{f}_a$ and $\mathbf{M}_a$ are the aerodynamic force and moment in the body frame, respectively. The notation $\lfloor \mathbf a \rfloor$ converts a 3-D vector $\mathbf a$ into a skew-symmetric matrix such that $\mathbf a \times \mathbf b = \lfloor \mathbf a\rfloor \mathbf b, \forall \mathbf a, \mathbf b \in \mathbb{R}^3$. $\mathbf{e}_1 = [1 \ 0 \ 0]^T, \mathbf{e}_2 = [0 \ 1 \ 0]^T, \mathbf{e}_3 = [0 \ 0 \ 1]^T$ are unit vectors used in the remaining of the paper. 

 Collecting all the state and input elements of the dynamics (\ref{e_vehicle_dyn}) leads to the system state and input below: 
	\begin{subequations}
		\begin{align}
			\mathbf x_{\rm full} &= (\mathbf p, \mathbf v, \mathbf R, \boldsymbol{\omega}) \in \mathbb R^3 \times  \mathbb R^3 \times SO(3)  \times \mathbb R^3 \\
			\mathbf u_{\rm full} &= (a_T, \boldsymbol{\tau}) \in \mathbb R \times \mathbb R^3
		\end{align}
		\label{e_full_xu}
		\vspace{-3mm}
	\end{subequations}

	Note that in the above model, we assume that the thrust direction is aligned to the body X axis $\mathbf{x}_b$. For cases where the thrust has a fixed installation angle, it can be trivially handled by re-defining the body frame. 

	\subsection{Aerodynamics}
	\label{sec_aerodyn}

	Referring to \citep{etkin1959dynamics}, the aerodynamic force  $\mathbf{f}_a$ is modeled in the body frame as follows:
	\begin{align}
		\label{e_aerodynmics}
		\mathbf{f}_a =  \begin{bmatrix}
			\mathbf f_{a_x} \\
			\mathbf f_{a_y} \\
			\mathbf f_{a_z}
		\end{bmatrix} = \begin{bmatrix}
			-\cos \alpha && 0 && \sin \alpha \\
			0 && 1 && 0 \\
			-\sin \alpha && 0 && -\cos \alpha
		\end{bmatrix} \begin{bmatrix}
			\mathcal{D} \\ \mathcal{Y} \\ \mathcal{L}
		\end{bmatrix}  
	\end{align}
	where $\alpha$ is the angle of attack. The force components $\mathcal{L}, \mathcal{D}, \mathcal{Y}$ are the  lift, drag, and side force, respectively. The aerodynamic moment vector $\mathbf{M}_a$  consists of rolling $L$, pitching $M$ and yawing  $N$  moment along the body axis $\mathbf{x}_b, \mathbf{y}_b, \mathbf{z}_b$:
	\begin{equation}
		\mathbf{M}_a = \begin{bmatrix}
			L & M & N
		\end{bmatrix}^T
	\end{equation}
 
	 \begin{figure}[t!] 
	 	\centering
	 	\includegraphics[width=0.85\linewidth]{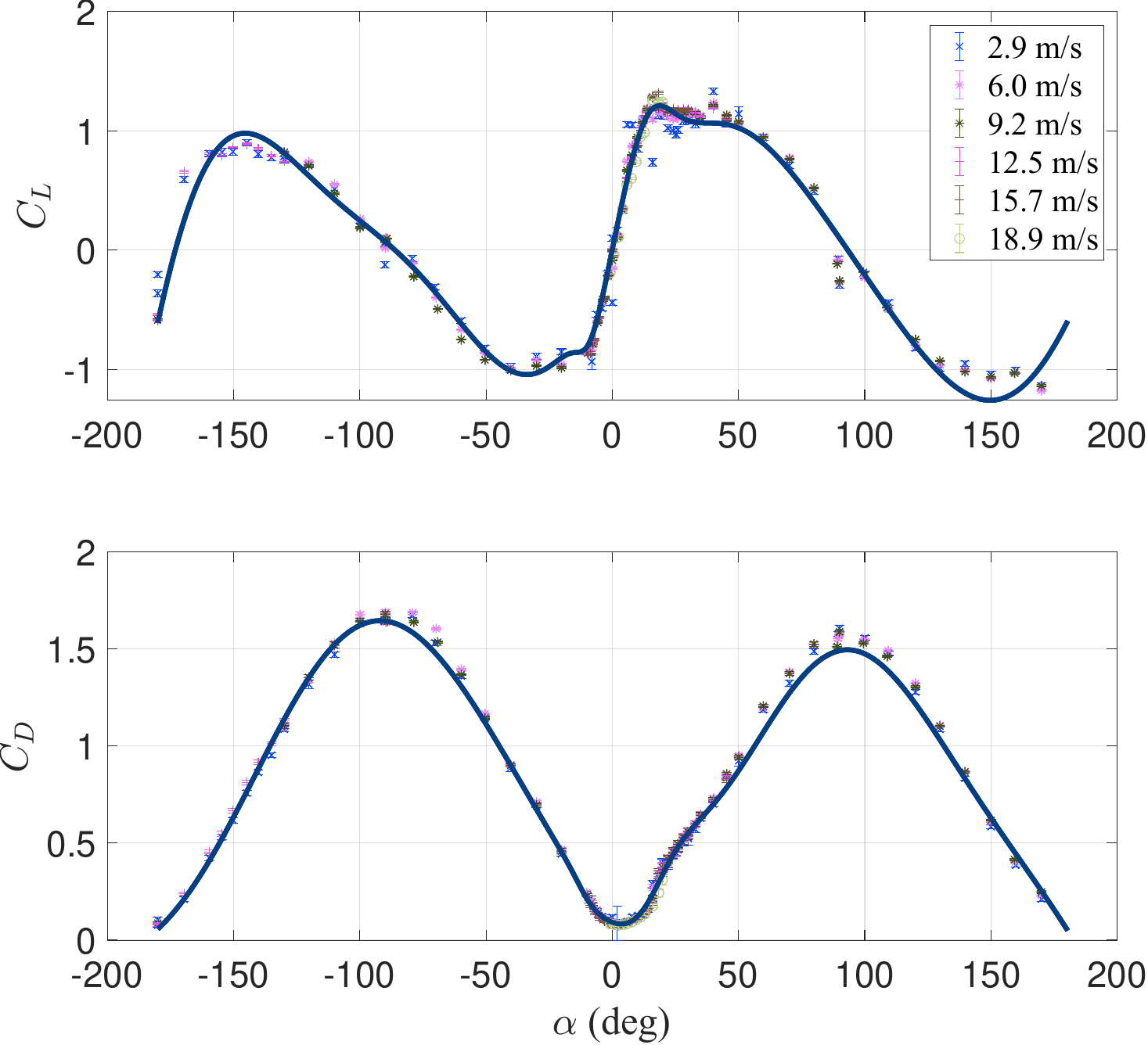} 
	 	\caption{Longitudinal aerodynamic coefficients $C_L$ and  $C_D$ of our previous quadrotor tail-sitter UAV prototype, identified by wind tunnel tests \citep{lyu2018simulation}.} 
	 	\label{fig_aero_coef}
	 	\vspace{-2mm}
	 \end{figure}

	The force and moment components $\mathcal{L}, \mathcal{D}, \mathcal{Y}, L, M, N$  can be written as products of non-dimensional coefficients, dynamic pressure $\frac{1}{2} \rho V^2$, the reference area $S$ (e.g., the wing area), and the characteristic length $\bar{c}$ (e.g., the mean aerodynamic chord), as follows: 
	\begin{equation}
		\begin{aligned}
			\mathcal{L} &= \frac{1}{2} \rho V^2 SC_L(\alpha,\beta) \\
			\mathcal{D} &= \frac{1}{2} \rho V^2 SC_D(\alpha,\beta) \\
			\mathcal{Y} &= \frac{1}{2} \rho V^2 SC_Y(\alpha,\beta)  
		\end{aligned} \quad,
		\begin{aligned}
		M &= \frac{1}{2} \rho V^2 S\bar{c}C_l(\alpha,\beta) \\
		N &= \frac{1}{2} \rho V^2 S\bar{c}C_m(\alpha,\beta) \\
		L &= \frac{1}{2} \rho V^2 S\bar{c}C_n(\alpha,\beta)
		\end{aligned}
		\label{e_aerodyn}
	\end{equation}
	where $\rho$ is the air density and $V = \| \mathbf v_a \|$ is the norm of the airspeed. $C_L, C_D, C_Y$ are the lift, drag, and side force coefficients, while $C_l, C_m, C_n$ are the rolling, pitching, and yawing moment coefficients. The aerodynamic coefficients are functions of the angle of attack $\alpha$ and the sideslip angle $\beta$, depending on the design of the airfoil profile and the overall airframe. The accurate aerodynamic coefficients are usually identified by wind tunnel tests \citep{lyu2018simulation}. For readability, the total aerodynamic force $\mathbf{f}_a$ in (\ref{e_aerodynmics}) can be rewritten as
	\begin{equation}
		\mathbf{f}_a = \frac{1}{2} \rho V^2 S\mathbf{c}(\alpha,\beta)
		\label{e_aero_force}
	\end{equation} 
	where
	\begin{subequations}
	    \begin{align}
		\mathbf{c}(\alpha,\beta) &= \begin{bmatrix}
			\mathbf c_x(\alpha,\beta) &
			\mathbf c_y(\alpha,\beta) &
			\mathbf c_z(\alpha,\beta)
		\end{bmatrix}^T \\
		\mathbf c_x(\alpha,\beta) &= 	-C_D(\alpha, \beta)\cos\alpha + C_L(\alpha, \beta)\sin\alpha \\
		\mathbf c_y(\alpha,\beta) &= C_Y(\alpha, \beta) \\
		\mathbf c_z(\alpha,\beta) &= -C_D(\alpha, \beta)\sin\alpha - C_L(\alpha, \beta)\cos\alpha
	\end{align}
	\label{e_c_dfn}
	\end{subequations}

	Given the vehicle ground velocity $\mathbf{v}$ and wind speed $\mathbf{w}$ defined in the inertial frame, the airspeed $\mathbf{v}_a$, the angle of attack $\alpha$ and the sideslip angle $\beta$ are calculated as follows:
	\begin{align}
		\mathbf{v}_a &= \mathbf{v} - \mathbf{w} ,\ 
        \mathbf{v}_a^\mathcal{B} = \mathbf{R}^T\mathbf{v}_a = \begin{bmatrix}
			\mathbf{v}_{a_x}^\mathcal{B} &
			\mathbf{v}_{a_y}^\mathcal{B} &
			\mathbf{v}_{a_z}^\mathcal{B}
		\end{bmatrix}^T,  \\
		V &= \| \mathbf v_a \|, \
		\alpha = \tan^{-1} \left( \frac{\mathbf{v}_{a_z}^\mathcal{B}}{\mathbf{v}_{a_x}^\mathcal{B}}\right)  ,\  
		\beta = \sin^{-1} \left(\frac{\mathbf{v}_{a_y}^\mathcal{B}}{V} \right) 
        \label{e_V_alpha_beta}
	\end{align}

    We further assume that the airframe is symmetric to the body X-Z plane, which implies 
    \begin{subequations}
    	\begin{align}
    		C_L(\alpha, \beta) &= C_L(\alpha, - \beta),\  \forall \alpha, \beta \\
    		C_D(\alpha, \beta) &= C_D(\alpha, - \beta),\  \forall \alpha, \beta \\
    		C_Y(\alpha, \beta) &= -C_Y(\alpha, - \beta),\  \forall \alpha, \beta
    	\end{align}
    \end{subequations}
	and hence $\forall \alpha$
	\begin{subequations}\label{e_axial_symmetric}
		\begin{align}
			&C_Y\! (\alpha, 0) \! = \!  0,  \left.\frac{\partial C_L \! (\alpha, \beta)}{\partial \beta} \! \right\vert_{\beta = 0} \!=\! \left.\frac{\partial C_D \! (\alpha, \beta)}{\partial \beta} \! \right\vert_{\beta = 0} \!=\! 0, \\
			&\frac{\partial \mathbf{c}(\alpha, \beta)}{\partial \beta}\rvert_{\beta = 0} = \begin{bmatrix}
				0 & \left.\frac{\partial \mathbf c_y (\alpha, \beta)}{\partial \beta}\right\vert_{\beta = 0} & 0
			\end{bmatrix}^T, \label{e_pc_palpha} \\
			&\frac{\partial \mathbf{c}(\alpha, \beta)}{\partial \alpha}\rvert_{\beta = 0} = \begin{bmatrix}
				\frac{\partial \mathbf{c}_x(\alpha, 0)}{\partial \alpha} & 0 & \frac{\partial \mathbf{c}_z(\alpha, 0)}{\partial \alpha}
			\end{bmatrix}^T.
		\end{align}
	\end{subequations}

	\section{Differential flatness in coordinated flight}
	\label{sec_diff_flat}
	In this section, we aim to investigate the fundamental differential flatness property which is the theoretical foundation for trajectory generation and tracking control.  We prove that the tail-sitter is differentially flat in a flight condition known as the coordinated flight.
	
	\subsection{The coordinated flight}
	An aircraft in coordinated flight indicates a flight condition without sideslip (e.g., $\beta = 0, \mathbf v_{a_y}^{\mathcal{B}} = 0$) \citep{clancy1975aerodynamics}. This flight condition does not restrict the degree-of-freedom of the tail-sitter, which is still able to reach any position in the entire 3-D space. Moreover, the coordinated flight is usually preferred over uncoordinated flight \citep{stevens2015aircraft} for several practical reasons: $\bf 1)$ the coordinated flight condition ideally achieves  maximum aerodynamic efficiency and also minimizes  undesirable aerodynamic moment that could cause spins. $\bf 2)$ it is naturally required when the navigation sensors (e.g., cameras) mounted on the vehicle's nose have a limited FoV. $\bf 3)$ restricting the sideslip angle around zero reduces the efforts for aerodynamic model identification by only requiring the longitudinal aerodynamic coefficients around $\beta = 0$ (see Fig. \ref{fig_aero_coef}). 

    \begin{figure}[t] 
	\centering		
        \includegraphics[width=1\linewidth]{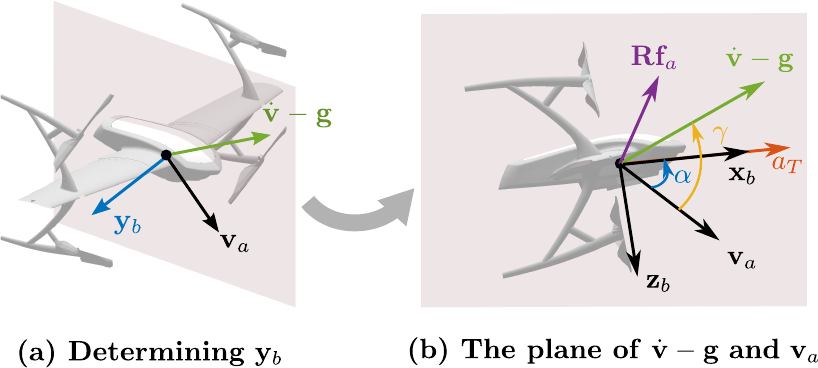} 
        \caption{A tail-sitter UAV in coordinated flight: (a) the axis $\mathbf{y}_b$ is perpendicular to both of $\mathbf{v}_a$ and $\dot{\mathbf{v}} - \mathbf{g}$; (b) the angle of attack $\alpha$ and thrust acceleration $a_T$ are determined on the longitudinal plane by the fact that the total acceleration comprising the drag acceleration $\mathcal{D}/m$, the lift acceleration $\mathcal{L}/m$, the $a_T$, and gravity $\mathbf g$ is equal to $\dot{\mathbf{v}}$.}
        \label{fig_coordinated_flight}
    \end{figure}

	\subsection{The differential flatness}
	\label{sec_flat_func}
	\begin{definition}
		\citep{fliess1995flatness}
			A system $\dot{\mathbf x} = \mathbf f \left(\mathbf x, \mathbf u \right), \mathbf x \in \mathbb R^n, \mathbf u \in \mathbb R^m,  \frac{\partial \mathbf f(\mathbf x, \mathbf u)}{\partial \mathbf u} = m$,  is differentially flat, if  there exists a flat output  $\mathbf y \in \mathbb R^m$ of the form
			\begin{equation}
				\mathbf y = \mathbf{y} \left(\mathbf x, \mathbf u,  \dot{\mathbf u}, \cdots, \mathbf u^{(p)} \right)
			\end{equation}
			such that the system state can be expressed explicitly by functions of the flat output and a finite number of its derivatives:
			\begin{align}
				\mathbf x &= \mathbf{x} \left(\mathbf y, \dot{\mathbf y}, \cdots, \mathbf y^{(q)} \right) \\
				\mathbf u &= \mathbf{u} \left(\mathbf y, \dot{\mathbf y}, \cdots, \mathbf y^{(q)} \right)
		\end{align} 
	\vspace{-5mm}
	\end{definition}
	
	The definition of differential flatness formally requires an equal dimension of the control input and the selected flat output for a system with independent inputs. However, the control input $\mathbf u_{\rm full}$ defined in (\ref{e_full_xu}) is not independent, due to the coordinated flight condition.
	
	\vspace{-2mm}
	\begin{theorem}
		\label{theorem_rank_pfpu}
		Given the system dynamics in (\ref{e_vehicle_dyn}) and definition of state and input in (\ref{e_full_xu}), when the UAV performs coordinated flight, it holds that
		\begin{equation}
			\frac{\partial \mathbf f(\mathbf x_{\rm full}, \mathbf u_{\rm full})}{\partial \mathbf u_{\rm full}} = 3 
			\label{e_rank_pfpu}
		\end{equation}
	\end{theorem}
	\begin{proof}
		The proof is given in Appendix \ref{app_rank_pfpu}.
	\end{proof}	

	It is seen in Theorem \ref{theorem_rank_pfpu} and its proof that two elements of the body angular velocity and consequently the control moment $\boldsymbol{\tau}$ are coupled, and the control input $\mathbf u_{\rm full}$ reduced by one degree-of-freedom to maintain the coordinated flight condition. The reduced input dimension decreases the rank of derivative $\frac{\partial \mathbf f(\mathbf x_{\rm full}, \mathbf u_{\rm full})}{\partial \mathbf u_{\rm full}} $ by one, resulting in a flat output vector with a dimension of three only.
		
	Our choice of the flat output is the vehicle position $\mathbf{p} \in \mathbb R^3$ in the inertial frame. In the following, we prove that all of the vehicle states $\mathbf x_{\rm full}$ and inputs $\mathbf u_{\rm full}$ can be expressed by functions of $\mathbf{p} $ and its derivatives.
	
	The position $\mathbf p$ and velocity $\mathbf v$ are simply $\mathbf{p}$ itself and its first-order derivatives, respectively. To express the attitude $\mathbf{R}$ as a function of $\mathbf p$ and its derivatives, we observe that in the coordinate flight, $\bf 1)$ there is no airspeed along the body Y axis, implying that $\mathbf{y}_b$ is perpendicular to the airspeed $\mathbf{v}_a$; and $\bf 2)$ because the aerodynamic sideslip  force $\mathcal Y$ is zero (due to coordinated flight and symmetric airframe) and the thrust is in the body X-Z plane, there is no force (and hence acceleration) except gravity along the body Y axis. That is being said, the total acceleration excluding gravity, $\dot{\mathbf{v}} - \mathbf{g}$, has no projection on the body Y axis (i.e., $\mathbf{y}_b$ is perpendicular to $\dot{\mathbf{v}} - \mathbf{g}$). As shown in Fig. \ref{fig_coordinated_flight}\textcolor{red}{(a)}, being perpendicular to both $\dot{\mathbf{v}} - \mathbf{g}$ and $\mathbf{v}_a$, $\mathbf y_b$ can only be in one of two opposite directions. We choose the one closest to the body Y axis determined at the previous time step, denoted as $\mathbf{y}_b^{\rm prev}$,  to prevent drastic attitude change: 
    \begin{align}   
            r &= \text{sign}\left( \left(\mathbf{v}_a \times \left(\dot{\mathbf{v}} - \mathbf{g}\right)\right)\cdot \mathbf{y}_b^{\rm prev} \right)\\
            \mathbf{y}_b &= r \frac{\mathbf{v}_a \times \left(\dot{\mathbf{v}} - \mathbf{g}\right)}{\Vert \mathbf{v}_a \times \left(\dot{\mathbf{v}} - \mathbf{g}\right)\Vert} ,\quad \text{if}\ \Vert \mathbf{v}_a \times \left(\dot{\mathbf{v}} - \mathbf{g}\right)\Vert \neq 0
            \label{e_yb}
    \end{align}
	where $\text{sign}(a)$ denotes the sign of $a \in \mathbb{R}$ and the scalar $r$ denotes the direction of the body Y axis, ensuring that $\mathbf{y}_{b} \cdot \mathbf{y}_b^{\rm prev} \geq 0$ (the angle between $\mathbf y_b$ and $\mathbf{y}_b^{\rm prev}$ is always less than $90^{\circ}$). $\Vert \mathbf{v}_a \times \left(\dot{\mathbf{v}} - \mathbf{g}\right)\Vert = 0$ is a singularity condition that will be discussed in Section \ref{sec_singularities}. 
	
	Next, we show how to solve the body Z axis $\mathbf z_b$ and body X axis $\mathbf x_b$. We note that the sideslip force is zero due to the coordinated flight, hence the aerodynamic force $\mathbf f_a$ reduces to $\mathbf f_a = \begin{bmatrix}
	    \mathbf f_{a_x} & 0 & \mathbf f_{a_z}
	\end{bmatrix}^T$ and $\mathbf R \mathbf f_a = \mathbf x_b \mathbf f_{a_x} + \mathbf z_b \mathbf f_{a_z}$. Substituting $\mathbf R \mathbf f_a$ into (\ref{e_translation_dyn}) leads to: 
	\begin{equation}
		 a_T \mathbf x_b + \frac{\mathbf{f}_{a_x}}{m} \mathbf x_b + \frac{\mathbf{f}_{a_z}}{m} \mathbf z_b + \mathbf g = \dot{\mathbf v}
		\label{e_force_balance}
	\end{equation}
    Decomposing the equation along the direction of $\mathbf{x}_b$ and $\mathbf z_b$ respectively, we have (see Fig. \ref{fig_coordinated_flight}\textcolor{red}{(b)})
    \begin{subequations}
        \begin{align}
            &a_T \!=\! \mathbf x_b^T (\dot{\mathbf{v}} - \mathbf{g}) - \mathbf{f}_{a_x} / m
            &\mathbf z_b^T (\dot{\mathbf{v}} - \mathbf{g}) =  \mathbf{f}_{a_z}/m         
        \end{align}
    	\label{e_decomp_force_xz}
    \end{subequations}
	Since  $\mathbf x_b$, $\mathbf z_b$, $\dot{\mathbf v} - \mathbf g$ and $\mathbf v_a$ are all perpendicular to $\mathbf y_b$, they should lie in the same plane (see Fig. \ref{fig_coordinated_flight}\textcolor{red}{(b)}). Hence we have $\mathbf x_b^T (\dot{\mathbf{v}} - \mathbf{g}) = \Vert\dot{\mathbf{v}} - \mathbf{g}\Vert \cos\left(\gamma \!-\! \alpha\right)$, $\mathbf z_b^T (\dot{\mathbf{v}} - \mathbf{g}) = \Vert\dot{\mathbf{v}} - \mathbf{g}\Vert \sin\left(\gamma \!-\! \alpha\right)$, and
    \begin{subequations}
        \begin{align}
            &a_T \!=\! \Vert\dot{\mathbf{v}} - \mathbf{g}\Vert \cos\left(\gamma \!-\! \alpha\right) - \mathbf{f}_{a_x} / m
    		\label{e_aT}\\
            &\Vert\dot{\mathbf{v}} - \mathbf{g}
    		\Vert \sin\left(\gamma - \alpha\right) =  -\mathbf{f}_{a_z}/m
    		\label{e_alpha_func}       
        \end{align}
    	\label{e_decomp_force}
    \end{subequations}
    where
    \begin{equation}   
        \gamma =  
            r \cdot {\rm atan2} \left( \| (\dot{\mathbf{v}} \!-\! \mathbf{g}) \! \times \! \mathbf{v}_a \|, (\dot{\mathbf{v}} \!-\! \mathbf{g}) \cdot \mathbf{v}_a \right), \text{if}\ \Vert \mathbf{v}_a \Vert \neq 0 
        \label{e_gamma}
    \end{equation}
    and $r$ denotes the angle direction (the positive direction of $\gamma$ and $\alpha$ is defined such that rotating $\mathbf v_a$ along $\mathbf y_b$ will reach $\dot{\mathbf v} - \mathbf g$ and $\mathbf x_b$, respectively), while $\Vert\mathbf{v}_a\Vert \neq 0$ has been specified in $\Vert \mathbf{v}_a \times \left(\dot{\mathbf{v}} - \mathbf{g}\right)\Vert \neq 0$ above. 
    
	\begin{figure}[t] 
		\centering
		\includegraphics[width=1\linewidth]{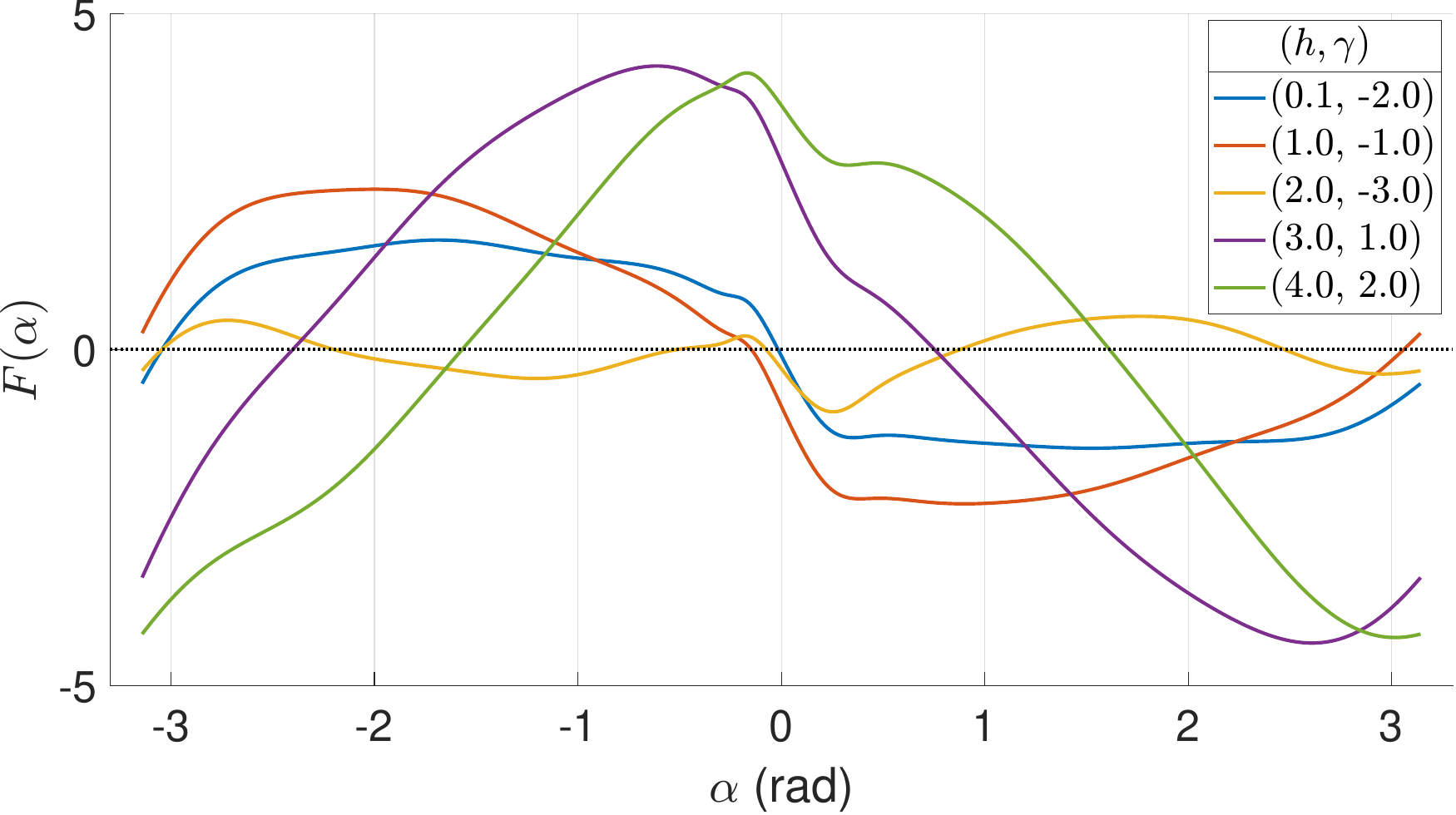} 
		\caption{Numerical examples of the root-finding problem of $F(\alpha) = 0$ in (\ref{e_solve_alpha}) for five pairs of $(h, \gamma)$ and the longitudinal aerodynamic coefficients shown in Fig. \ref{fig_aero_coef}.} 
		\label{fig_solve_alpha}
	\end{figure}

	It is noticed that (\ref{e_alpha_func}) only involves the known flat derivatives and the angle of attack $\alpha$, which can hence be solved. Specifically, (\ref{e_alpha_func}) can be written as a nonlinear root-finding problem in terms of $\alpha$:
    \begin{equation}
        F(\alpha)  = h \sin(\gamma - \alpha) + \mathbf{c}_z(\alpha, 0) = 0 
        \label{e_solve_alpha}
    \end{equation}
    where
    \begin{align}
        h = \frac{2m\Vert\dot{\mathbf{v}} - \mathbf{g}\Vert}{\rho V^2 S}
    \end{align}
	
	In the function of $F(\alpha)$, the variables $h$ and $\gamma$ is completely determined by the flight trajectory (and wind gust), while $\mathbf{c}_z(\alpha, 0)$ is the third element of $\mathbf{c}$ in (\ref{e_c_dfn}),  which is completely determined by the actual aerodynamic configuration of the UAV. It should be also noted that $\gamma$ and $h$ are independent because they are respectively the angle and length ratio between $\dot{\mathbf{v}} - \mathbf{g}$ and $\mathbf{v}_a$. These properties allow us to investigate the shape of $F(\alpha)$, hence the solution of $\alpha$, for a given pair of $(h, \gamma)$. An example of such function $F(\alpha)$ is presented in Fig. \ref{fig_solve_alpha}. As can be seen, the equation $F(\alpha) = 0$ is highly nonlinear due to the nonlinear aerodynamic model $\mathbf c_z(\alpha, 0)$, hence no closed-form solution can be found in general. In practice, the equation can be solved numerically, such as {Newton–Raphson method using $\mathbf c_z (\alpha)$ and $\partial \mathbf c_z(\alpha, 0) / \partial \alpha$ identified in advance}. Moreover, the extreme nonlinearity in $F(\alpha)$ also results in multiple solutions of $\alpha$ in most cases. To avoid the ambiguity and prevent drastic change of $\alpha$,  $\alpha^{\rm prev}$, the value of $\alpha$ determined at the previous time step, could be used as the initial guess for the numerical solver, to find a solution close to $\alpha^{\rm prev}$. 
		
	With the solved angle of attack $\alpha$, the body X axis $\mathbf{x}_b$, and hence the rotation matrix $\mathbf R$, can be determined as
	\begin{subequations}
	\label{eq:bodyX_attitude}
	\begin{align}
		\mathbf{x}_b = \text{Exp}\left(\alpha \mathbf{y}_b \right) \frac{\mathbf{v}_a}{\Vert\mathbf{v}_a\Vert}, \quad \text{if}\ \Vert\mathbf{v}_a\Vert \neq 0, \label{e_xb} \\
		\mathbf{R} = \begin{bmatrix}
			\mathbf{x}_b & \mathbf{y}_b & {\mathbf z_b}
		\end{bmatrix}, \quad \mathbf z_b = \mathbf{x}_b \times\mathbf{y}_b.
		\label{e_determine_R}
	\end{align}
	\end{subequations}
	where $\text{Exp}(\cdot)$ is the exponential map on $SO(3)$ and $\Vert\mathbf{v}_a\Vert \neq 0$ has been specified in the singularity condition $\Vert \mathbf{v}_a \times \left(\dot{\mathbf{v}} - \mathbf{g}\right)\Vert \neq 0$ above. With the solved $\alpha$ and $\beta = 0$, the aerodynamic force $\mathbf f_a$ and system input $a_T$ are determined by (\ref{e_aero_force}) and (\ref{e_aT}), respectively.
	
	Next, to show that the body angular velocity $\boldsymbol{\omega}$ is a function of the flat output, we take the time derivative of the translational dynamics (\ref{e_translation_dyn}) as follows:
	\begin{equation}
		\begin{aligned}
			\ddot{\mathbf{v}} &= \left(\dot{a}_T\mathbf{R}  + a_T\mathbf{R} \lfloor \boldsymbol{\omega} \rfloor \right)\mathbf{e}_1  \\
			&\quad + \frac{1}{m}\mathbf{R}\left(\lfloor \boldsymbol{\omega} \rfloor \mathbf{f}_a + \frac{\partial \mathbf{f}_a}{\partial \left(\mathbf{R}^T\mathbf{v}_a\right)} \frac{d}{dt}\left(\mathbf{R}^T\mathbf{v}_a\right)\right)  \\
			&= \frac{1}{m}\mathbf{R}\frac{\partial \mathbf{f}_a}{\partial \mathbf{v}_a^\mathcal{B}}\mathbf{R}^T\dot{\mathbf{v}}_a + \dot{a}_T\mathbf{R}\mathbf{e}_1 \\
			&\quad +\mathbf{R}\left(-\left \lfloor \left( a_T \mathbf{e}_1 + \frac{\mathbf f_a }{m} \right) \right\rfloor  + \frac{1}{m} \frac{\partial \mathbf{f}_a}{\partial \mathbf{v}_a^\mathcal{B}} \lfloor \mathbf{v}_a^\mathcal{B} \rfloor \right)\boldsymbol{\omega}
		\end{aligned}
		\label{e_d_trans_dyn}
	\end{equation}
	where $\partial \mathbf{f}_a / \partial \mathbf{v}_a^\mathcal{B}$ is evaluated at $\beta = 0$ and can be obtained by taking derivative of (\ref{e_aero_force}) as below.
	\begin{theorem}
		\label{theorem_pfa_pvb}
		Given the aerodynamic coefficients $\mathbf c(\alpha, \beta)$ of a symmetric airframe configuration satisfying (\ref{e_axial_symmetric}), the partial derivative $\partial \mathbf{f}_a / \partial \mathbf{v}_a^\mathcal{B}$ at $\beta = 0$ is
	   \begin{equation}
                \frac{\partial \mathbf{f}_a}{\partial \mathbf{v}_a^\mathcal{B}} = \frac{\rho S}{2}\left(2\mathbf{c} \mathbf{v}_a^{\mathcal{B}^T} +  \frac{\partial \mathbf{c}}{\partial \alpha}\mathbf{v}_a^{\mathcal{B}^T} \lfloor \mathbf{e}_2 \rfloor 
                 + V\frac{\partial \mathbf{c}}{\partial \beta}\mathbf{e}_2^T \right)
		\label{e_pfa_pvb}
	   \end{equation}
	\end{theorem}
	\begin{proof}
		The proof is given in Appendix \ref{app_pfa_pvb}.
	\end{proof}
	
	With the $\ddot{\mathbf v}, \dot{\mathbf v}_a, \mathbf {R}, \mathbf f_a, a_T$ and $\partial \mathbf{f}_a / \partial \mathbf{v}_a^\mathcal{B}$ solved above, the equation (\ref{e_d_trans_dyn}) forms three linear functions for $\dot{a}_T$ and $\boldsymbol{\omega}$. To solve $\dot{a}_T$ and $\boldsymbol{\omega}$ uniquely, we need to find one more equation. Recall that in coordinated flight the tail-sitter has no lateral airspeed:
	the condition requires zero lateral airspeed:
	\begin{equation}
		\mathbf v_{a_y}^{\mathcal{B}} = \mathbf{e}_2^T\mathbf{R}^T \mathbf{v}_a \equiv 0
		\label{e_vb_y_zero}
	\end{equation}
	which leads to the derivative on the both sides:
	\begin{equation}
		\begin{aligned}
			- \mathbf{e}_2^T \lfloor\boldsymbol{\omega}\rfloor \mathbf{R}^T \mathbf{v}_a+ \mathbf{e}_2^T \mathbf{R}^T \dot{\mathbf{v}}_a = 0 \\
			\Rightarrow \quad -\mathbf{v}^T_a \mathbf{R} \lfloor\mathbf{e}_2 \rfloor \boldsymbol{\omega} + \mathbf{y}^T_b \dot{\mathbf{v}}_a = 0      
		\end{aligned}
		\label{e_d_v_y}
	\end{equation}
	
	Combing (\ref{e_d_v_y}) and (\ref{e_d_trans_dyn}), we obtain four linear equations in terms of the $\dot{a}_T$ and $\boldsymbol{\omega}$, which can hence be solved as:
	\begin{equation}
		\begin{bmatrix}
			\dot{a}_T \\ \boldsymbol{\omega}
		\end{bmatrix} = \mathbf{N}^{-1}\mathbf{h} = \begin{bmatrix}
			\mathbf{N}_1 \\ \mathbf{N}_2
		\end{bmatrix}^{-1} \begin{bmatrix}
			\mathbf{h}_1 \\ \mathbf{h}_2
		\end{bmatrix}, \ \text{if}\ \text{rank}(\mathbf{N}) = 4
		\label{e_solve_omega}
	\end{equation}
	where $\text{rank}(\mathbf{N}) < 4$ is the second singularity condition that will be discussed in Section \ref{sec_singularities}, and
	\begin{subequations}
		\label{e_h_N}
		\begin{align}
			&\mathbf{h}_1 = \mathbf{y}^T_b \dot{\mathbf{v}}_a \label{e_h1}\\ 
			&\mathbf{h}_2 = \ddot{\mathbf{v}} - \frac{1}{m}\mathbf{R}\frac{\partial \mathbf{f}_a}{\partial \mathbf{v}_a^\mathcal{B}}\mathbf{R}^T\dot{\mathbf{v}}_a \label{e_h2} \\
			&\mathbf{N}_1 = \begin{bmatrix}
				0 & \mathbf{v}^T_a \mathbf{R} \lfloor\mathbf{e}_2 \rfloor
			\end{bmatrix}  \label{e_N1}\\
			&\mathbf{N}_2 = \begin{bmatrix}
				\mathbf{R}\mathbf{e}_1 & \mathbf{R}\! \left(\! -\left \lfloor \left( a_T \mathbf{e}_1 + \frac{\mathbf f_a }{m} \right) \right\rfloor \! + \! \frac{1}{m} \!  \frac{\partial \mathbf{f}_a}{\partial \mathbf{v}_a^\mathcal{B}}\lfloor\mathbf{v}_a^\mathcal{B}\rfloor \! \right) 
			\end{bmatrix} \label{e_N2}
		\end{align}
	\end{subequations}

	Furthermore, the angular acceleration $\dot{\boldsymbol{\omega}}$ can be attained by further taking the derivative of (\ref{e_solve_omega}):
	\begin{equation}
		\begin{bmatrix}
			\ddot{a}_T \\ \dot{\boldsymbol{\omega}}
		\end{bmatrix} = \frac{d}{dt} \left(\mathbf{N}^{-1}\mathbf{h}\right) = -\mathbf{N}^{-1}\dot{\mathbf{N}}\mathbf{N}^{-1}\mathbf{h} + \mathbf{N}^{-1}\dot{\mathbf{h}} 
		\label{e_d_omega}
	\end{equation}
	where the matrix derivative $\dot{\mathbf{N}}$ and $\dot{\mathbf{h}}$ are given in Appendix \ref{app_dNH_dt}. It is noted that the coefficient gradients $\partial^2 \mathbf c_z (\alpha,0) / \partial \alpha^2$ (hence $\partial^2 C_L(\alpha,0) / \partial \alpha^2$ and $\partial^2 C_D(\alpha,0) / \partial \alpha^2$) should be further provided. Then the control moment $\boldsymbol{\tau}$, is solved from (\ref{e_rotational_dyn}) as
	\begin{equation}
		\boldsymbol{\tau} = \mathbf{J}\dot{\boldsymbol{\omega}} - \mathbf{M}_a + \boldsymbol{\omega} \times \mathbf{J} \boldsymbol{\omega}
		\label{e_tau}
	\end{equation}
	where the aerodynamic moments $\mathbf M_a$ is calculated from (\ref{e_aerodyn}) based on $\beta = 0$ and the $\alpha$ solved above. 

\begin{remark}
	Formally, the flatness functions are real-analysis by the classic definition. However, when deriving the flatness function of angle of attack $\alpha$, we cannot find its closed-form solution for a general aerodynamic model due to the extreme nonlinearity. Fortunately, we reduce this problem into a one-dimensional root-finding problem as shown in Fig. \ref{fig_solve_alpha}, that can be solved efficiently by numerical methods in real-time computation. Except $\alpha$, the remaining flatness functions are all given explicitly.
\end{remark}

 \begin{remark}
 In the aerodynamic model (\ref{e_aero_force}) and the differential flatness derivation above, we assumed that the aerodynamic force $\mathbf f_a$ depends only on the vehicle states (i.e., airspeed and attitude) but not the control inputs (i.e., moment $\boldsymbol{\tau}$ and thrust $a_T$). This is generally true for quadrotor tail-sitter VTOL UAVs where no extra flaps are used and the propellers are distant from wing hence the wing aerodynamic force $\mathbf f_a$ not depending on the propeller airflow. For tail-sitter UAVs whose moment $\boldsymbol{\tau}$ is produced by flaps at the trailing edge of wings, such as the twin-rotor tail-sitter UAV in \cite{tal2022global}, the flaps deflection and propeller airflow would change the aerodynamic force $\mathbf f_a$, causing the aerodynamic force $\mathbf f_a$ to depend on the control inputs and preventing the solving of (\ref{e_solve_alpha}). This issue could be overcome practically by a strategy similar to \cite{tal2022global}, which assumes very small changes of control inputs (i.e., flap deflections and propeller thrust) at each step, so that aerodynamic force $\mathbf f_a$ can be evaluated at the last flap angle and propeller thrust, and then used to solve $\boldsymbol{\omega}$ and $\boldsymbol{\tau}$ as detailed above. 
 
 \end{remark}

	\subsection{Singularity conditions}
	\label{sec_singularities}
    
	We discuss the two conditions that singularities occur in the above flatness functions, one is $\Vert \mathbf{v}_a \times \left(\dot{\mathbf{v}} - \mathbf{g}\right)\Vert = 0$ as specified in (\ref{e_yb}) and the other is $\text{rank}(\mathbf{N}) < 4$ as specified in (\ref{e_solve_omega}). We first investigate the possible singularity condition where $\text{rank}(\mathbf{N}) < 4$, by calculating the determinant of $\mathbf{N}$ as follows:
	\begin{theorem}
		\label{theorem_det_N}
		Given the aerodynamic coefficients $\mathbf c(\alpha, \beta)$ of a symmetric airframe configuration satisfying (\ref{e_axial_symmetric}), the determinant of $\mathbf{N}$ defined in (\ref{e_h_N}) is calculated as follows.
		\begin{equation}
			\det(\mathbf{N}) = - \frac{\rho SV^2}{2m} \frac{\partial F(\alpha)}{\partial \alpha} \| \mathbf{v}_a \times (\dot{\mathbf{v}} -\mathbf{g}) \|
			\label{e_det_N}
		\end{equation}
	\end{theorem}
	\begin{proof}
		The proof is given in Appendix \ref{app_det_N}.
	\end{proof}	

	As can be seen in (\ref{e_det_N}), there are two cases that make $\mathbf{N}$ singular, one is $\frac{\partial F(\alpha)}{\alpha} = 0$ and the other is $\| \mathbf{v}_a \times (\dot{\mathbf{v}} -\mathbf{g}) \| = 0$. Because the angle of attack $\alpha$ is solved from ${F}(\alpha) = 0$ in (\ref{e_solve_alpha}), the former condition essentially requires $F(\alpha)$ passing trough zero with a zero slope, a condition that rarely occurs for actual aerodynamic configuration $\mathbf c_z(\alpha, \beta)$ (see Fig.\ref{fig_solve_alpha}). Therefore, the singularity condition $\text{rank}(\mathbf{N}) < 4$ reduces to the first singularity condition $\| \mathbf{v}_a \times (\dot{\mathbf{v}} -\mathbf{g}) \| = 0$, which has to be considered. This singularity condition breaks into the following three sub-conditions:
	\begin{subequations}
	    \begin{align}
	    	&\Vert \dot{\mathbf{v}} - \mathbf{g} \Vert = 0 \\
	    	&\Vert \mathbf{v}_a \Vert = 0 \\
	    	&  \gamma  = 0             
	    \end{align}
    	\label{e_sing_sub_cond}
	\end{subequations}
	We investigate the corresponding flight status for these three condition as follows.
	
	\subsubsection{Singularity sub-conditions 1. $\Vert \dot{\mathbf{v}} - \mathbf{g} \Vert = 0$} \label{sec:singularity_case1}
	
	The sub-condition $\|\dot{\mathbf{v}} - \mathbf{g} \| = 0$ is the case where the vehicle is free falling, which is undesired in usual flights and should be avoided in the trajectory planning. Therefore, this sub-condition would not be encountered in practice.
	
	\begin{figure}[t!] 
		\centering	
		\subfigure[Singularity case $\|\mathbf{v}_a\| = 0$] { \label{fig_sing_va0}      
			\includegraphics[width=0.46\linewidth]{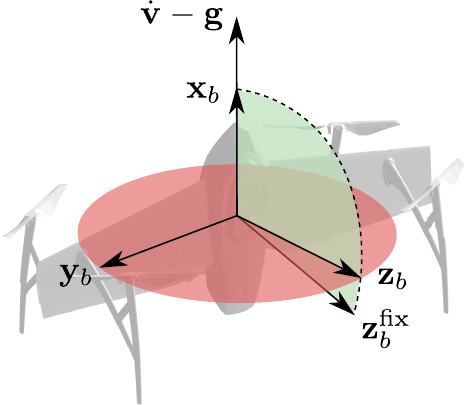}     
		}
		\subfigure[Singularity case $\gamma = 0$] { 
			\label{fig_sing_gamma0}     
			\includegraphics[width=0.44\linewidth]{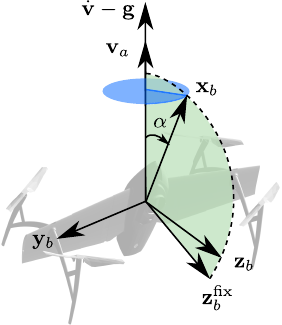}     
		} 
		\caption{Determination of the vehicle body Y axis $\mathbf{y}_b$ (or body Z axis $\mathbf{z}_b$) under two singular conditions (a) $\|\mathbf{v}_a\| = 0$ (e.g., near hovering flights) and (b) $\gamma = 0$ (e.g., vertical takeoff or landing). In both figures, the green plane denotes the plane of $\dot{\mathbf{v}} - \mathbf{g}$ and $\mathbf{z}_b^{\rm fix}$. In (a), the red circular plane perpendicular to $\mathbf{x}_b$ denotes all possible directions of $\mathbf{z}_b$. To minimize the yaw effort, $\mathbf{z}_b$ should the intersecting line of the green and red plane. In (b), the blue disk denotes all possible directions of $\mathbf x_b$. For each direction of $\mathbf x_b$, $\mathbf z_b$ could further rotate along $\mathbf x_b$ freely. To minimize the yaw effort, both $\mathbf x_b$ and $\mathbf z_b$ should be within the green plane. }
		\vspace{-3mm}
	\end{figure}

	\subsubsection{Singularity sub-conditions 2. $\Vert \mathbf{v}_a \Vert = 0$}~
	\label{sec_sing_va0}
	
	The second sub-condition $\Vert \mathbf{v}_a \Vert = 0$ corresponds to zero airspeed, which occurs when the vehicle hovers in windless environments such as indoor places, or flies in the same velocity as the wind in outdoor environments. When the airspeed $\mathbf{v}_a$ is zero, (\ref{e_yb}) becomes singular and hence cannot determine $\mathbf{y}_b$. Actually, even when $\mathbf{v}_a$ is close to zero, (\ref{e_yb}) will be ill-conditioned, where a small change in $\mathbf v_a$ may cause drastic orientation change in $\mathbf y_b$. To avoid this ill condition, we choose a small velocity threshold $v_{\rm min}$ (e.g. $v_{\rm min} = $\SI{0.5}{m/s}). When $\| \mathbf v_a \| < v_{\rm min}$, the aerodynamic force $\mathbf f_a$, which is quadratic to $\| \mathbf v_a \|$, can be safely ignored. Substituting $\mathbf f_a = \mathbf 0$ into (\ref{e_force_balance}) leads to
	
	\begin{equation}
		\mathbf{x}_b = \frac{\dot{\mathbf{v}} - \mathbf{g}}{\|\dot{\mathbf{v}} - \mathbf{g}\|},\quad a_T = \|\dot{\mathbf{v}} - \mathbf{g}\|
		\label{e_xb_aT_sing_va0}
	\end{equation}
	
	For the axis $\mathbf y_b$ (or equivalently, $\mathbf z_b$), it could be any direction perpendicular to $\mathbf x_b$ without affecting the solution in (\ref{e_xb_aT_sing_va0}) (see Fig. \ref{fig_sing_va0}). To minimize the unnecessary efforts for yawing control, we fix the vehicle yaw angle at the value of yaw angle just before $\| \mathbf v_a \| < v_{\rm min}$ took place (e.g., when the vehicle decelerates to hover) or the value of yaw angle at initial time (e.g., when the vehicle just took off from the ground).  Since the yaw angle is represented by the body Z axis, we hope to find a $\mathbf z_b$ that has the smallest angle with $\mathbf{z}_b^{\rm fix}$, the vehicle body Z axis just before $\| \mathbf v_a \| < v_{\rm min}$ took place or at initial time. This essentially causes $\mathbf z_b$ to lie on the plane formed by $\mathbf x_b$ and $\mathbf z_b^{\rm fix}$ (see Fig. \ref{fig_sing_va0}), which, in return, leads $\mathbf y_b$ to be perpendicular to $\mathbf x_b$ (i.e., $\dot{\mathbf{v}} - \mathbf{g}$) and $\mathbf{z}_b^{\rm fix}$:
	\begin{equation}
		\mathbf{y}_b = \frac{\mathbf{z}_b^{\rm fix} \times (\dot{\mathbf{v}} - \mathbf{g})}{\|\mathbf{z}_b^{\rm fix} \times (\dot{\mathbf{v}} - \mathbf{g}) \|}
        \label{e_yb_sing_v0}
	\end{equation}
	With $\mathbf x_b$ and $\mathbf y_b$, the vehicle attitude can be determined by (\ref{e_determine_R}).
	
	Next, to determine the body angular velocity, we notice $\left(\mathbf{z}_b^{\rm fix} \times (\dot{\mathbf{v}} - \mathbf{g})\right)^T \mathbf{z}_b \equiv 0$ always holds. Taking time derivative on both sides and recalling that ${\mathbf z}_b^{\rm fix}$ is a prescribed constant vector, we have
	\begin{equation}
		\begin{aligned}
			\left(\lfloor \mathbf{z}_b^{\rm fix} \rfloor \ddot{\mathbf{v}}\right)^T\mathbf{z}_b &= \| \mathbf{z}_b^{\rm fix} \times (\dot{\mathbf{v}} - \mathbf{g}) \| \mathbf{y}_b^T\mathbf{R} \lfloor \mathbf{e}_3\rfloor \boldsymbol{\omega} \\
			&= \| \mathbf{z}_b^{\rm fix} \times (\dot{\mathbf{v}} - \mathbf{g}) \| \mathbf{e}_1^T \boldsymbol{\omega} 
		\end{aligned}
		\label{e_d_sing_cond_va0}
	\end{equation}
    Moreover, neglecting the aerodynamics, the derivative of translational dynamics in (\ref{e_d_trans_dyn}) can be rewritten as
	\begin{equation}
		\ddot{\mathbf{v}} = \left(\dot{a}_T\mathbf{R}  + a_T\mathbf{R} \lfloor \boldsymbol{\omega} \rfloor \right)\mathbf{e}_1 
		\label{e_d_trans_dyn_va0}
	\end{equation}
	
	Combining (\ref{e_d_sing_cond_va0}) and (\ref{e_d_trans_dyn_va0}), both the $\dot{a}_T$ and $\boldsymbol{\omega}$ can be solved from the a 4-D linear equations in the same form as (\ref{e_solve_omega}) with sub-matrices of $\mathbf{h}$ and $\mathbf{N}$ are rewritten as follows:
	\begin{subequations}
		\begin{align}
			&\mathbf{h}_1 = \left(\lfloor \mathbf{z}_b^{\rm fix} \rfloor \ddot{\mathbf{v}}\right)^T\mathbf{z}_b \label{e_h1_sing_va0}\\
			&\mathbf{h}_2 = \ddot{\mathbf{v}} \\
			&\mathbf{N}_1 = \begin{bmatrix}
				0 & \| \mathbf{z}_b^{\rm fix} \times (\dot{\mathbf{v}} - \mathbf{g}) \| \mathbf{e}_1^T
			\end{bmatrix} \label{e_N1_sing_va0}\\
			&\mathbf{N}_2 = \begin{bmatrix}
				\mathbf{R}\mathbf{e}_1 & -a_T\mathbf{R} \lfloor \mathbf{e}_1 \rfloor
			\end{bmatrix} \label{e_N2_sing_va0}
		\end{align}
		\label{e_h_N_sing_va0}
	\end{subequations}
	\begin{theorem}
		\label{theorem_sing_va0}
		The determinant of $\mathbf{N}$ defined in (\ref{e_h_N_sing_va0}) is calculated as
		\begin{equation}
			\det(\mathbf{N}) = - a_T^2 \| \mathbf{z}_b^{\rm fix} \times (\dot{\mathbf{v}} - \mathbf{g}) \| 
		\end{equation}
	\end{theorem}
	\begin{proof}
		The proof is given in Appendix \ref{app_sing_v0_det_N}.
	\end{proof}
	
	From (\ref{e_xb_aT_sing_va0}), we have $a_T = \|\dot{\mathbf{v}} - \mathbf{g}\|$, which is not zero in practice (see Section \ref{sec:singularity_case1}). Therefore, the only requirement for both (\ref{e_yb_sing_v0}) and $\det(\mathbf{N}) \neq 0$ is $ \| \mathbf{z}_b^{\rm fix} \times (\dot{\mathbf{v}} - \mathbf{g}) \|  \neq 0 $, a condition that is always true because at the moment $\| \mathbf v_a \| \approx v_{\rm min}$, the body X axis $\mathbf x_b^{\rm fix}$ is almost aligned with $\dot{\mathbf v}- \mathbf g$ (the aerodynamic force is negligible and the thrust must provide most of the special acceleration $\dot{\mathbf v}- \mathbf g$), meaning that $\mathbf z_b^{\rm fix}$ cannot be parallel to $\dot{\mathbf v}- \mathbf g$. 
	
	Finally, the angular acceleration and control moment are also solved from (\ref{e_d_omega}) and (\ref{e_tau}), where the derivatives $\dot{\mathbf{h}}$ and $\dot{\mathbf{N}}$ are recalculated in Appendix \ref{app_sing_v0_dh_dN}.

	\subsubsection{Singularity sub-conditions 3. $\gamma = 0$}~
	
	When the airspeed $\mathbf{v}_a$ and the acceleration $\dot{\mathbf{v}} - \mathbf{g}$ is parallel, the singularity sub-condition $\gamma = 0$ occurs. A common possible case is that the vehicle performs vertical takeoff and landing when the wind speed is zero. In this case, the angle of attack $\alpha$ and thrust acceleration $a_T$ can still be solved from (\ref{e_solve_alpha}) and (\ref{e_aT}) respectively, but the body Y axis cannot be determined from (\ref{e_yb}), which is singular. Actually, even when $\gamma$ is close to zero, (\ref{e_yb}) will be ill-conditioned, where a small change in $\mathbf v_a$ or $\dot{\mathbf v} - \mathbf g$ may cause drastic orientation change in $\mathbf y_b$. To avoid this ill condition, we choose a small angle threshold $\gamma_{\rm min}$ (e.g. $\gamma_{\rm min} = 5^\circ$). When $|\gamma | < \gamma_{\rm min}$, we minimize the unnecessary yaw control efforts by restricting the axes $\mathbf x_b$ and $\mathbf z_b$ within the plane formed by $\mathbf v_a$ (or $\dot{\mathbf v} - \mathbf g$) and $\mathbf z_b^{\rm fix}$, the vehicle body Z axis just before $|\gamma| < \gamma_{\rm min}$ occurs or at the initial time. 
	As a result, the body Y axis is perpendicular to $\mathbf v_a$ (or $\dot{\mathbf v} - \mathbf g$) and $\mathbf z_b^{\rm fix}$ and is hence determined from (\ref{e_yb_sing_v0}). With $\mathbf y_b$, the body X axis and the vehicle attitude are determined from (\ref{eq:bodyX_attitude}). 

	 To solve the body angular velocity, we take the time derivative to the constraint $\left(\mathbf{z}_b^{\rm fix} \times (\dot{\mathbf{v}} - \mathbf{g}) \right)^T \mathbf{z}_b \equiv 0 $ which is identical to (\ref{e_d_sing_cond_va0}). Combing this constraints with the derivative of the translational dynamics in (\ref{e_d_trans_dyn}), the body angular velocity is solved in the same form as (\ref{e_solve_omega}) where the sub-matrices are given from (\ref{e_h_N}) for $\mathbf h_2, \mathbf N_2$ and (\ref{e_h_N_sing_va0}) for $\mathbf h_1, \mathbf N_1$:
    \begin{subequations}
        \begin{align}
            &\mathbf{h}_1 = \left(\lfloor \mathbf{z}_b^{\rm fix} \rfloor \ddot{\mathbf{v}}\right)^T\mathbf{z}_b \label{e_h1_sing_gamma0}\\
            &\mathbf{h}_2 = \ddot{\mathbf{v}} - \frac{1}{m}\mathbf{R}\frac{\partial \mathbf{f}_a}{\partial \mathbf{v}_a^\mathcal{B}}\mathbf{R}^T\dot{\mathbf{v}}_a \label{e_h2_sing_gamma0} \\
			&\mathbf{N}_1 = \begin{bmatrix}
				0 & \| \mathbf{z}_b^{\rm fix} \times (\dot{\mathbf{v}} - \mathbf{g}) \| \mathbf{e}_1^T
			\end{bmatrix} \label{e_N1_sing_gamma0}\\
			&\mathbf{N}_2 = \begin{bmatrix}
				\mathbf{R}\mathbf{e}_1 & \mathbf{R}\! \left(\! -\left \lfloor \left( a_T \mathbf{e}_1 + \frac{\mathbf f_a }{m} \right) \right\rfloor \! + \! \frac{1}{m} \!  \frac{\partial \mathbf{f}_a}{\partial \mathbf{v}_a^\mathcal{B}}\lfloor\mathbf{v}_a^\mathcal{B}\rfloor \! \right) 
			\end{bmatrix} \label{e_N2_sing_gamma0}
        \end{align}
        \label{e_h_N_sing_gamma0}
    \end{subequations}
	
	\begin{theorem}
		\label{theorem_sing_gamma0}
		The determinant of $\mathbf{N}$ defined in (\ref{e_h_N_sing_gamma0}) is calculated as
		\begin{equation}
			\det(\mathbf{N}) = \frac{\rho SV^2}{2m} \frac{\partial F(\alpha)}{\partial \alpha} \|\mathbf{z}_b^{\rm fix} \times (\dot{\mathbf{v}} - \mathbf{g})\| \psi_{23}		
		\end{equation}
		where
        \vspace{-2mm}
	\begin{equation}
		\psi_{23} = \|\dot{\mathbf{v}} - \mathbf{g} \|\cos(\gamma - \alpha) -\frac{\rho S V^2}{2m}\frac{\partial \mathbf{c}_y}{\partial \beta}\cos\alpha
        \label{e_psi_23_sing_gamma0}
	\end{equation}
	\end{theorem}
	\begin{proof}
		The proof is given in Appendix \ref{app_sing_gamma0_det_N}.
	\end{proof}

	It is seen in (\ref{e_det_N}) that three possible cases making $\mathbf{N}$ singular, $\frac{\partial F(\alpha)}{\alpha} = 0$, $\|\mathbf{z}_b^{\rm fix} \times (\dot{\mathbf{v}} - \mathbf{g})\| = 0$, or $\psi_{23} = 0$. As the discussion to Theorem \ref{theorem_det_N}, since the angle of attack $\alpha$ is solved from $F(\alpha) = 0$, the former condition requires $F(\alpha)$ to pass through zero with a zero slope, which rarely occurs for actual aerodynamic configuration $\mathbf c_z(\alpha, \beta)$. For the second condition $\|\mathbf{z}_b^{\rm fix} \times (\dot{\mathbf{v}} - \mathbf{g})\| = 0$, since the current singularity case occurred at the vertical ascending or descending flights, the thrust should provide the major special acceleration $\dot{\mathbf v}- \mathbf g$. Since the thrust is aligned with body X axis, the direction $\dot{\mathbf v}- \mathbf g$ should be most similar to $\mathbf x_b^{\rm fix}$, not $\mathbf z_b^{\rm fix}$, which rules out the condition $\|\mathbf{z}_b^{\rm fix} \times (\dot{\mathbf{v}} - \mathbf{g})\| = 0$. For the third condition $\psi_{23} = 0$, it requires a special $\frac{\partial \mathbf c_y (\alpha, \beta)}{\partial \beta}\rvert_{\beta = 0}$ that satisfies both of $F(\alpha) = 0$ and (\ref{e_psi_23_sing_gamma0}), which generally does not hold in actual aerodynamic configurations. Therefore, the matrix $\mathbf{N}$ is non-singular in practice.

	Finally, the angular acceleration and control moment are also solved from (\ref{e_d_omega}) and (\ref{e_tau}), but the derivatives $\dot{\mathbf{h}}$ and $\dot{\mathbf{N}}$ are recalculated in Appendix \ref{app_sing_gamma0_d_h_N}.
	
	\begin{remark}
		Singularity conditions  $||\mathbf v_a|| = 0$  and  $\gamma = 0$, is resolved in a unified manner of assigning $\mathbf z_b$ closest to a fixed direction $\mathbf z_b^{\rm fix}$,  which is equivalent to fixing the yaw angle.  If the singularity conditions are caused by vehicles at low speed vertical flights (e.g., hovering, vertical take-off and landing), such fixing of yaw angle is unnecessary. For example, an extra yaw angle can be specified by assigning $\mathbf z_b^{\rm fix}$, to achieve sensor-pointing and sideways maneuvering.
	\end{remark}

 	\subsection{Differential flatness transform}
 	\label{sec_diff_flat_tran}
 	In this section, we present a complete differential flatness transform that maps a flat-output trajectory to a state-input trajectory, based on the flatness functions in Section \ref{sec_flat_func} with treatments for singularity conditions presented in Section \ref{sec_singularities}. 
 	Since the flatness functions and singularity treatments are all based on $\mathbf v_a$, the airspeed, they could naturally incorporate the wind speed $\mathbf w$ into the inertial speed $\mathbf v$. In practice, we compute the airspeed as $\mathbf v_a = \mathbf v - \bar{\mathbf w}$, where $\bar{\mathbf w}$ is a surrogate wind speed to be compensated. In case of full wind speed compensation, we set $\bar{\mathbf w} = \mathbf w$ or else $\bar{\mathbf w} = \mathbf 0$.
	
 	Combining all elements above, the complete differential flatness transform can be obtained as shown in Algorithm \ref{alg_proc_diff_flat}. With the transform, any flat-output trajectories can be mapped to the system state $\mathbf x_{\rm full} $ and control input $\mathbf u_{\rm full}$ as below:
    \vspace{-2mm}
    \begin{equation}
        \label{e_full_system_flat}
            \mathbf{x}_{\rm full} = \mathcal{X}_{\rm full}(\mathbf{p}^{(0:3)}) , \quad \mathbf{u}_{\rm full} = \mathcal{U}_{\rm full}(\mathbf{p}^{(1:4)}),
    \end{equation}
    where $\mathcal{X}_{\rm full}(\mathbf{p}^{(0:3)})$ denotes the state flatness function of the flat output and its derivatives up to the third order, and $\mathcal{U}_{\rm full}(\mathbf{p}^{(1:4)})$ denotes the input flatness function of the flat-output derivatives up to the fourth order. 

 	\begin{algorithm}[t!]
 		\caption{Differential flatness transform}
 		\label{alg_proc_diff_flat}
 		\SetAlgoLined
 		\textbf{Given: } a trajectory of flat output $\mathbf{p}^{(0:4)}$ avoiding $\|\dot{\mathbf{v}} - \mathbf{g}\| = {0}$, current surrogate wind speed $\bar{\mathbf{w}}$,  previous body Y axis $\mathbf{y}_b^{\rm prev}$, velocity threshold $v_{\rm min}$, angle threshold $\gamma_{\rm min}$, and the body Z axis $\mathbf{z}_b^{\rm fix}$ fixed just before $\| \mathbf v_a \| < v_{\rm min}$ or $|\gamma| < \gamma_{\rm min}$. \\  
 		\textbf{Calculate the airspeed: } \\
 		$\mathbf{v}_a = \mathbf{v} - \bar{\mathbf{w}}$; \\
 		\textbf{Calculate the states and inputs: } \\
 		\eIf{$\Vert \mathbf{v}_a \Vert < v_{\rm min}$}
 		{
 			Assign $\mathbf{y}_b$ perpendicular to $\dot{\mathbf{v}} - \mathbf{g}$ and $\mathbf{z}_b^{\rm fix}$ (\ref{e_yb_sing_v0});\\
            Determine $\mathbf{x}_b$ (\ref{e_xb_aT_sing_va0}), $a_T$ (\ref{e_xb_aT_sing_va0}) and $\mathbf{R}$ (\ref{e_determine_R});\\
            Calculate $\mathbf{h}$ and $\mathbf{N}$ (\ref{e_h_N_sing_va0});
 		}{
                Calculate $\gamma$ (\ref{e_gamma});\\
 			Solve $\alpha$ (\ref{e_solve_alpha}) and $a_T$ (\ref{e_aT}); \\   
 			\eIf{ $|\gamma| < \gamma_{\rm min}$}
 			{
				Assign $\mathbf{y}_b$ perpendicular to $\dot{\mathbf{v}} \!\!-\!\! \mathbf{g}$ and $\mathbf{z}_b^{\rm fix}$ (\ref{e_yb_sing_v0});\\
                Determine $\mathbf{x}_b$ (\ref{e_xb}) and $\mathbf{R}$ (\ref{e_determine_R});\\
                Calculate $\mathbf{h}$ and $\mathbf{N}$ (\ref{e_h_N_sing_gamma0});
 			}
 			{
 				Determine $\mathbf{y}_b$ (\ref{e_yb}), $\mathbf{x}_b$ (\ref{e_xb}) and $\mathbf{R}$  (\ref{e_determine_R}); \\
                Calculate $\mathbf{h}$ and $\mathbf{N}$ (\ref{e_h_N}); \\
 				Set $\mathbf{z}_b^{\rm fix} = \mathbf{z}_b$;
 			}
 		}
	 	Solve $\boldsymbol{\omega}$ (\ref{e_solve_omega}); \\
	 	Solve $\dot{\boldsymbol{\omega}}$ (\ref{e_d_omega}) and  $\boldsymbol{\tau}$ (\ref{e_tau}); \\		
 		Set $\mathbf{y}_b^{\rm prev} = \mathbf{y}_b$;
 	\end{algorithm}

	\section{System overview}
	\label{sec_ctrl_archi}
	In this section, we present the entire framework of trajectory generation and tracking control for aggressive  flights based on the fundamental differential flatness of the tail-sitter vehicle presented previously. 

	\begin{figure*}[t!] 
		\centering
		\includegraphics[width=1\linewidth]{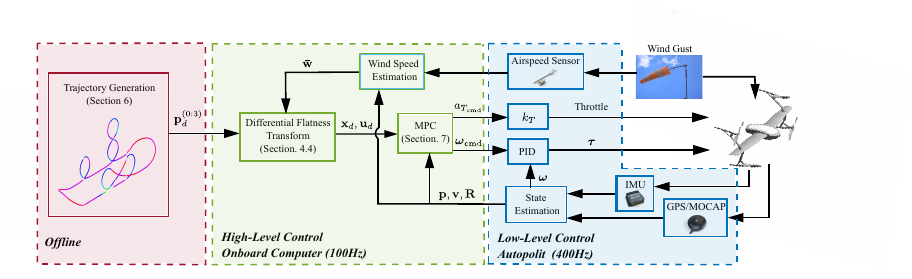} 
		\caption{System overview.} 
		\label{fig_system_overview}
		\vspace{-0mm}
	\end{figure*}
	
	\subsection{System reduction}
	 The full system presented in (\ref{e_vehicle_dyn}) is of dimension 12, comprising the vehicle position, velocity, attitude and angular velocity. Note that the system has a cascaded structure, where the input torque $\boldsymbol{\tau}$ solely affects the angular velocity, and then the angular velocity determines the attitude, hence the velocity and position of the vehicle. Enabled by this cascaded dynamics, we propose to control the angular velocity dynamics (\ref{e_rotational_dyn}) separately (referred to as the ``low-level control"). In the low-level control, the Coriolis term $\boldsymbol{\omega} \times \mathbf{J} \boldsymbol{\omega}$ and aerodynamic moment $\mathbf{M}_a$ can be compensated in a feed forward way, while the remaining dynamics are first order linear systems that can be controlled by linear feedback controller (e.g., PID controller). More systematic and advanced control techniques could also be deployed, such as $\mu$-synthesis \citep{noormohammadi2020system}, $H_{\infty}$ loop shaping \citep{li2020modeling}, Notch filters \cite{xu2019full}, to improve the controller bandwidth and robustness to model uncertainties (e.g., unknown flexible modes) and possible vibrations. Details of our low-level controller is shown in Section \ref{sec_uav_platform}.
	 
	 With a well-designed low-level controller, we assume that the vehicle angular velocity can be instantaneously achieved hence it can be viewed as the control input to the rest vehicle dynamics consisting of (\ref{e_translation_kin}, \ref{e_translation_dyn}, \ref{e_rotation_kin}) (the ``high-level system"). As a result, the state and input of the high-level system are
    \begin{subequations}
        \begin{align}
		\mathbf{x} &= \left(\mathbf{p} \ \ \mathbf{v} \ \ \mathbf{R}\right)  \\
		\mathbf{u} &= \left(a_T \ \  \boldsymbol{\omega}\right) 
	\end{align}
    \end{subequations}
	which are subject to the following system model
	\begin{equation}
		\dot{\mathbf{x}} = f(\mathbf{x}, \mathbf{u}) = \left\{ \begin{array}{l}
			\dot{\mathbf{p}} = \mathbf{v}, \\
			\dot{\mathbf{v}} = \mathbf{g} + a_T\mathbf{R}\mathbf{e}_1 + \frac{1}{m}\mathbf{R}\mathbf{f}_a \\
			\dot{\mathbf{R}} =  \mathbf{R} \lfloor \boldsymbol{\omega} \rfloor 
		\end{array}\right.	
		\label{e_high_level_sys}
	\end{equation}
	Since the state in (\ref{e_high_level_sys}) is a reduced set of the original one in (\ref{e_vehicle_dyn}), the reduced system is still differentially flat. Specifically, the state and input of the high-level system can be written as
	
	\begin{equation}
	    \label{e_reduced_system_flat}
	    \mathbf{x} = \mathcal{X}(\mathbf{p}^{(0:2)}), \quad \mathbf{u} = \mathcal{U}(\mathbf{p}^{(1:3)})
	\end{equation}
	where $\mathcal{X}(\cdot)$ and $\mathcal{U}(\cdot)$ are subsets of $\mathcal{X}_{{\rm full}}(\cdot)$ and $\mathcal{U}_{{\rm full}}(\cdot)$ from (\ref{e_full_system_flat}) and the corresponding state-input trajectory $(\mathbf{x}, \mathbf{u})$ will satisfy the high-level system model (\ref{e_high_level_sys}) subject to the surrogate wind speed $\bar{\mathbf w}$. The high-level system in (\ref{e_high_level_sys}) is of lower dimension and will be used for our trajectory planning and tracking control.
 
	This system reduction presented above has both advantages and disadvantages. One advantage is the reduction of computation complexity in trajectory generation, avoiding the cumbersome derivative of moment $\boldsymbol{\tau}$ with respect to the flat-output and the knowledge of dynamic parameters. Another advantage is decoupling the low-level angular velocity control, which are highly related to the vehicle dynamics (e.g., flexible modes, motor delay, etc.), from the high-level system planning and tracking control. A disadvantage arises that the original dynamical feasibility (i.e., thrust and moment) is approximated as constraints on inputs of the reduced high-level system (i.e., thrust and angular velocity). This approximation becomes less accurate when the vehicle is maneuvering with rapidly varying  angular velocity that necessitates large control moment. Despite the rough approximation, practical quadrotor applications in  drone racing \citep{romero2022time} and aerobatics \citep{kaufmann2020RSS, lu2022manifold} demonstrate that the feasibility can be sufficiently guaranteed by constraining  inputs of the reduced system in most cases.
	

	
	\subsection{System framework}
	With the system reduction above, the overview of our proposed approach is shown in Fig. \ref{fig_system_overview}. A flat-output trajectory up to the third-order smoothness (i.e, $\mathbf{p}_d^{(0:3)}$) is planned offline for the high-level system by a trajectory generation module (Section \ref{sec_traj_gen}). For online trajectory tracking, we propose a two-stage control strategy. The first stage is differential flatness transform (Section \ref{sec_diff_flat_tran}) that maps the flat-out trajectory to the desired state and input trajectory $\mathbf{x}_{d}, \mathbf{u}_{d}$. This transform also incorporates environment wind (if enabled) and fixes the singularity conditions presented in Section \ref{sec_singularities}. The computed state-input trajectories are then tracked in the second stage by an unified global on-manifold MPC, which computes the optimal control inputs $a_{T_{\rm cmd}}$ and $\boldsymbol{\omega}_{\rm cmd}$ (Section \ref{sec_ctrl}). These commands are then sent as reference to the low-level controller. 
	
	\section{Trajectory generation}
	\label{sec_traj_gen}
	
	
	Since the vehicle dynamics is differentially flat in coordinated flight as proved in Section \ref{sec_diff_flat}, all states and inputs can be expressed by flatness functions of the flat output and its derivatives. As a result, the trajectory generation problem reduces to low-dimensional algebra in the flat-output space (i.e., the vehicle position), without any integration of the under-actuated system dynamics in (\ref{e_high_level_sys}). We parameterize the vehicle position as polynomials \citep{bry2015aggressive, mueller2015computationally, ding2019efficient} and minimize the flight time and control efforts computed from the flatness functions (\ref{e_reduced_system_flat}), subject to necessary constraints.

	\subsection{Trajectory optimization}
    \label{sec_planning_coordinated}
	We formulate the trajectory planning as an optimization problem that finds a dynamically-feasible, smooth trajectory $\mathbf{p}(t): \mathbb{R} \in [0, T_f] \mapsto \mathbb{R}^3$ with the minimum flight time $T_f$, control effort $\mathbf{u}$, and passing through a sequence of waypoints $\mathbf{Q} = (\mathbf{q}_0, \  \cdots \  \mathbf{q}_M)$. The purpose of waypoints are two-fold: 1) it could be used to obtain a collision-free trajectory when using with a front-end flight corridor (e.g., \cite{liu2017planning, gao2019flying}); and 2) specifying the location of the waypoints could change the shape of the flight trajectory, so that the desired aerobatic flight trajectories can be obtained. Given the initial state $\mathbf{s}_0$, terminal state $\mathbf{s}_f$ and waypoints $\mathbf{Q}$, the trajectory optimization is formulated as:
	\begin{subequations}
		\begin{align}
			&\displaystyle\min_{\mathbf{p}(t),T_f} \displaystyle \int_{0}^{T_f}  \Vert \mathbf{u}\Vert^2_{\mathbf
				W} dt + \rho T_f  \label{e_opt_cost}\\
			&\mathrm{s.t.} \quad \mathbf{x}(t) = \mathcal{X}\left(\mathbf{p}^{(0:2)}(t)\right), \ \mathbf{u}(t) = \mathcal{U}\left(\mathbf{p}^{(1:3)}(t) \right) \label{e_opt_sys}\\
			& \qquad\ \  \mathbf{p}^{(0:3)}(0) = \mathbf{s}_0 , \ \mathbf{p}^{(0:3)}(T_f) = \mathbf{s}_f \label{e_opt_endpoint}\\
			& \qquad\ \  \mathbf{p}(t_{\mathbf q_i}) = \mathbf{q}_i, \ 0 \leq t_{\mathbf q_0} < \cdots < t_{\mathbf q_M} \leq T_f \label{e_opt_waypoint}\\
			& \qquad\ \   \mathbf{x}(t) \in \mathbb{X},\ \mathbf{u}(t) \in \mathbb{U} \label{e_opt_bound} \\
			& \qquad\ \  \mathcal{S}(\mathbf{x}) \geq \boldsymbol{\epsilon}  \label{e_opt_sing}
		\end{align}
		\label{e_traj_opt}
	\end{subequations}
	where $\mathbf{W} \in \mathbb{R}^{4\times 4}$ is a positive diagonal matrix penalizing the total control effort and $\rho > 0$ is the flight time penalty. $\mathbb{X}$ denotes the kinodynamic constraint that ensures the vehicle to operate within a safe workspace. The state constraints (\ref{e_opt_bound}) in this paper is the velocity condition
	\begin{equation}
		\Vert \mathbf{v}(t) \Vert \leq v_{\rm max}
	\end{equation}
    where $v_{\rm max}$ is the maximum velocity for safe flight. $\mathbb{U} = \{\mathbf{u} \in \mathbb{R}^4 | \ \mathbf{u}_{\rm min} \leq \mathbf{u} \leq \mathbf{u}_{\rm max}\}$ is the boundary of the system inputs (i.e., the thrust acceleration $a_T$ and angular velocity $\boldsymbol{\omega}$). $\mathcal{S}(\mathbf{x})$ denotes the singularity condition. Among the three singularity sub-conditions in Section \ref{sec_singularities}, the conditions $\|\mathbf{v}_a \| < v_{\rm min}$ and $|\gamma| < \gamma_{\rm min}$ have been well treated, hence $\mathcal{S}(\mathbf{x})$ needs only to consider the first sub-condition:		
	\begin{equation}
		\label{e_sing_cond_numerical}
	    \mathcal{S}(\mathbf{x}) = \Vert \dot{\mathbf{v}} - \mathbf{g} \Vert^2 \geq \epsilon^2 
	\end{equation} 		
 	where $\epsilon$ is a small positive value for numerical stability on implementation ($\epsilon = $ \SI{0.1}{m/s^2} in this paper).
	
	The optimization problem in (\ref{e_traj_opt}) optimizes both the flat-output trajectory $\mathbf p(t)$ and the flight time $T_f$, to minimize the total control efforts and time in (\ref{e_opt_cost}). The minimization of control efforts tend to find smooth trajectories that are easier to track and the minimization of total time $T_f$ tends to produce high-speed trajectories. Hence, the optimization (\ref{e_traj_opt}) promises both trajectory smoothness and agility. The system state and control input in (\ref{e_opt_sys}) are characterized as the flatness functions that explicitly exploits the vehicle dynamic and kinematic models. The initial and terminal conditions of the trajectory are specified in (\ref{e_opt_endpoint}). The dynamical feasibility which indicates the actuation capability of the aircraft (or the tracking capability of the low-level control system) is guaranteed by the boundary constraints in (\ref{e_opt_bound}). The collision-free and the shape constraints of the trajectory could be achieved by satisfying the waypoint constraints in (\ref{e_opt_waypoint}). Finally, singularity conditions are incorporated into the constraint in (\ref{e_opt_sing}).
	
	\subsection{Trajectory optimization solving}
	The trajectory optimization (\ref{e_traj_opt}) is a nonlinear constrained optimization problem. We leverage a state-of-the-art flight trajectory planning framework, MINCO \citep{wang2022geometrically}, to parameterize and solve the trajectory. Referring to \cite{wang2022geometrically}, we insert $N$ free control points $\mathbf{d}_i = (\mathbf d_{i_1}, \cdots, \mathbf d_{i_N}) \in \mathbb{R}^{N \times 3}$ between each two consecutive waypoints $\mathbf q_i, \mathbf q_{i+1}$ and create a  waypoint sequence $\mathbf r = (\mathbf q_0, \mathbf d_0, \mathbf q_1, \cdots, \mathbf q_{M-1}, \mathbf d_{M-1}, \mathbf q_M) \in \mathbb{R}^{(M(N+1)+1) \times 3}$. The corresponding passing time for the waypoint sequence is $\mathbf{T} = [t_{\mathbf q_0}, t_{\mathbf d_{0_1}}, \cdots, t_{\mathbf d_{0_N}}, t_{\mathbf q_1}, \cdots, t_{\mathbf q_M}] \in \mathbb{R}^{M(N+1)+1}$. Then we characterize the trajectory by a multi-stage polynomial trajectory, where a $7$th-order polynomial trajectory with $C^4$ continuity is used to connect to two consecutive points $\mathbf r_j, \mathbf r_{j+1} \in \mathbf r$ at their respective passing time $\mathbf{T}_j, \mathbf{T}_{j+1} \in \mathbf{T}$. The entire trajectory is therefore uniquely determined by all points $\mathbf{r}$ and respective passing time $\mathbf{T}$, having the endpoint constraint (\ref{e_opt_endpoint}) and waypoint constraint (\ref{e_opt_waypoint}) naturally satisfied. To deal with remaining boundary constraint (\ref{e_opt_bound}) and the singularity condition (\ref{e_opt_sing}), we relax these constraints to soft penalties in the objective function, hence transforming the constrained nonlinear optimization (\ref{e_traj_opt}) into an unconstrained nonlinear optimization problem. The decision variables of the resultant optimization problem consist of control points $\mathbf D = (\mathbf d_0, \cdots, \mathbf{d}_{M-1}) \in \mathbb{R}^{MN \times 3}$ and  passing time $\mathbf T$, which are solved by a quasi-Newton method \citep{wang2022geometrically}. 
	
	To solve the unconstrained nonlinear optimization with a quasi-Newton method, gradients of the objective and constraints with respect to the decision variables $\mathbf D$ and $\mathbf T$ are needed. The gradients of flat-output (i.e, $\partial \mathbf{p}^{(1:3)}(t) / \partial \mathbf{D}$, $\partial \mathbf{p}^{(1:3)}(t) / \partial \mathbf{T}$) have been derived in detail in \cite{wang2022geometrically}, with which the gradients of the control input $\mathbf{u}$ and singularity condition in (\ref{e_opt_sing}) can be calculated by the chain rule:
	\begin{subequations}
		\begin{align}
			\frac{\partial \mathcal S(\mathbf x)}{\partial \mathbf{D}} &= 2(\dot{\mathbf v} - \mathbf g)^T \frac{\partial \dot{\mathbf v}}{\partial \mathbf D} \\
   \frac{\partial \mathcal S(\mathbf x)}{\partial \mathbf{T}} &= 2(\dot{\mathbf v} - \mathbf g)^T \frac{\partial \dot{\mathbf v}}{\partial \mathbf T}\\
			\frac{\partial \mathbf{u}(t)}{\partial \mathbf{D}} &= \frac{\partial \mathcal{U}\left(\mathbf{p}^{(1:3)}(t)\right)}{\partial \mathbf{p}^{(1:3)}(t)}  \frac{\partial \mathbf{p}^{(1:3)}(t)}{\partial \mathbf{D}} \\ \frac{\partial \mathbf{u}(t)}{\partial \mathbf T} &= \frac{\partial \mathcal{U}\left(\mathbf{p}^{(1:3)}(t)\right)}{\partial \mathbf{p}^{(1:3)}(t)}  \frac{\partial \mathbf{p}^{(1:3)}(t)}{\partial \mathbf T} 
		\end{align}
	\end{subequations}
	where $\partial \mathcal{U}\left(\mathbf{p}^{(1:3)}(t)\right) / \partial \mathbf{p}^{(1:3)}(t)$ is gradients of the flatness functions in Section \ref{sec_diff_flat}. The calculation is provided in Appendix \ref{app_gradient_flat_func}.

	\section{Global control for trajectory tracking}
	\label{sec_ctrl}
	In this section, we develop a global tracking controller that allows a tail-sitter to accurately follow aggressive reference trajectories in real-world environments. Unlike conventional tail-sitter controllers operating in separate flight modes or existing global controllers considering a simplified aerodynamic model, the proposed global controller fully exploits the vehicle aerodynamics, contributing to accurate, agile flights within the entire envelope without encountering control switching or singularity.

	\subsection{The error-state system}
	The goal of the tracking controller is to drive the vehicle state to follow the desired reference state trajectory $\mathbf x_d$, which is computed from the trajectory planned in Section \ref{sec_traj_gen} via the flatness function (\ref{e_reduced_system_flat}). Equivalently, the error between the actual and reference state trajectory should converge to zero. Therefore, we only need to control the error state $\delta \mathbf x$. 

	\subsubsection{Definition of the error state}~
	
	Considering the tail-sitter model in (\ref{e_high_level_sys}), the system state evolves on a compound manifold below
	\begin{align}
		\label{e_state_manifold}
		&\mathcal{M} = \mathbb{R}^3 \times \mathbb{R}^3 \times SO(3), \quad \dim(\mathcal{M}) = 9 \\
		&\mathbf{x} = \begin{pmatrix}
			\mathbf{p} \\  \mathbf{v} \\  \mathbf{R}
		\end{pmatrix}  \in \mathcal{M},\quad  \mathbf{u} = \begin{bmatrix}
			a_T \\ \boldsymbol{\omega}
		\end{bmatrix} \in \mathbb{R}^4	
	\end{align}

	We assume that the trajectory planner generates a full reference trajectory, including the state $\mathbf{x}_d = (
	\mathbf{p}_d \  \mathbf{v}_d \  \mathbf{R}_d
	) \in \mathcal{M}$ and input $\mathbf{u}_d = \begin{bmatrix}
		a_{T_d} & \boldsymbol{\omega}_d^T
	\end{bmatrix}^T \in \mathbb{R}^4$. Note that the state-input trajectory $(\mathbf{x}_d, \mathbf{u}_d)$ satisfies the model (\ref{e_high_level_sys}) subject to the surrogate wind speed $\bar{\mathbf w}$. 
	
	Defining the error between the actual state $\mathbf x$ and the reference one $\mathbf x_d$, both lie on the state manifold $\mathcal{M}$, is not trivial. We adopt the definition in our prior work \citep{lu2022manifold}, which defines the error state in the local homeomorphic space (an open set in Euclidean space) around each point $\mathbf x_d$. This particular error definition on manifold is denoted as $\boxminus$ \citep{hertzberg2013integrating} detailed as below: 
	\begin{subequations}
		\label{e_error_state_def}
		\begin{align}
		\delta \mathbf{x} &\triangleq  \mathbf{x}_d \boxminus \mathbf{x} = \begin{bmatrix}
			\delta \mathbf{p}^T & \delta \mathbf{v}^T & \delta \mathbf{R}^T
		\end{bmatrix}^T \in \mathbb{R}^9 \\
		\delta \mathbf{p} &\triangleq \mathbf{p}_d \boxminus \mathbf{p} = \mathbf{p}_d - \mathbf{p} \in \mathbb{R}^3 \label{e_delta_p}\\
		\delta \mathbf{v} &\triangleq \mathbf{v}_d \boxminus \mathbf{v} = \mathbf{v}_d - \mathbf{v} \in \mathbb{R}^3 \label{e_delta_v}\\ 
		\delta \boldsymbol{\theta} &\triangleq \mathbf{R}_d \boxminus \mathbf{R} = \text{Log}(\mathbf{R}^T\mathbf{R}_d) \in \mathbb{R}^3 \label{e_delta_R}
	\end{align}
	\end{subequations}
	where $\text{Log}(\cdot)$ is the logarithmic map of the manifold $SO(3)$ and also the inverse of the exponential map $\text{Exp}(\cdot)$. The control inputs are in the Euclidean space, so their errors can be defined directly:
	\begin{subequations}
		\label{e_error_input_def}
		\begin{align}
		\label{e_delta_omega}
		\delta\mathbf{u} &\triangleq \mathbf{u}_d - \mathbf{u} = \begin{bmatrix}
			\delta a_T & \delta\boldsymbol{\omega}^T
		\end{bmatrix}^T \in \mathbb{R}^4 \\
		\delta a_T &\triangleq a_{T_d} - a_T \in \mathbb{R} \\
		\delta \boldsymbol{\omega} &\triangleq \boldsymbol{\omega}_d - \boldsymbol{\omega}  \in \mathbb{R}^3
	\end{align}
	\end{subequations} 

	\subsubsection{The error-state system dynamics}~
	\label{sec_error_state}
 
	To control the error state $\delta \mathbf x$ (\ref{e_error_state_def}) to converge to zero, we need to obtain its dynamic model. To do so, we take the derivative of the error state with respect to time. 

	\begin{theorem}
		\label{theorem_error_state}
		Given the error state defined in (\ref{e_error_state_def}), where the actual trajectory $(\mathbf x, \mathbf u)$ satisfies (\ref{e_high_level_sys}) with the actual wind speed $\mathbf w$ and the reference trajectory $(\mathbf x_d, \mathbf u_d)$ satisfies (\ref{e_high_level_sys}) with the surrogate wind speed $\bar{\mathbf w}$, then the dynamics of the error-state system is:
		\begin{subequations}
			\label{e_error_state_sys}
			\begin{align}
				\delta \dot{\mathbf{x}} &=  \begin{bmatrix}
					\delta \dot{\mathbf{p}}^T & \delta \dot{\mathbf{v}}^T & \delta \dot{\boldsymbol{\theta}}^T
				\end{bmatrix}^T  \\
				\delta \dot{\mathbf{p}} &= \delta \mathbf{v} \label{e_d_delta_p}\\
				\delta \dot{\mathbf{v}} \!&=\! \left( a_{T_d}\mathbf{R}_d\mathbf{e}_1 \!\!+\!\! \frac{1}{m}\mathbf{R}_d\mathbf{f}_{a_d} \right) \!\!-\!\! \left( a_T\mathbf{R}\mathbf{e}_1 \!\!+\!\! \frac{1}{m}\mathbf{R}\mathbf{f}_a \right) \label{e_d_delta_v}\\ 
				\delta \dot{\boldsymbol{\theta}} &= \mathbf{A}^T(\delta\boldsymbol{\theta}) \left(-\mathbf{R}_d^T\mathbf{R}\boldsymbol{\omega} + \boldsymbol{\omega}_d \right) \label{e_d_delta_R}
			\end{align}
		\end{subequations}
		where $\mathbf{f}_{a_d}$ and $\mathbf{f}_{a}$ are the aerodynamic forces in terms of the desired and actual state, respectively
        \begin{subequations}
            \begin{align}
                \mathbf{f}_{a_d} & =  \mathbf{f}_a \left(\mathbf{v}_{a_d}^{\mathcal{B}}\right) ,\quad 
		\mathbf{v}_{a_d}^{\mathcal{B}} = \mathbf{R}_d^T \left( \mathbf{v}_d - \bar{\mathbf{w}} \right) \\
                \mathbf{f}_{a} &=  \mathbf{f}_a \left( \mathbf{v}_{a}^{\mathcal{B}} \right) ,\quad 
		\mathbf{v}_{a}^{\mathcal{B}} = \mathbf R \left(\mathbf{v} - {\mathbf{w}} \right)
            \end{align}
        \end{subequations}
        $\mathbf{A}(\cdot)$ denotes the Jacobian of the exponential coordinates of $SO(3)$ \citep{bullo1995proportional}:
            \begin{equation}
                \mathbf{A}(\mathbf{\delta\boldsymbol{\theta}}) \!=\! \mathbf{I}_3 \!+\! \left(\! \frac{1 \!-\! \cos\Vert\delta\boldsymbol{\theta}\Vert}{\Vert\delta\boldsymbol{\theta}\Vert} \! \right) \! \! \frac{\lfloor\delta\boldsymbol{\theta}\rfloor}{\Vert\delta\boldsymbol{\theta}\Vert} \!+\! \left( \! 1 \!-\! \frac{\sin\Vert\delta\boldsymbol{\theta}\Vert}{\Vert\delta\boldsymbol{\theta}\Vert} \! \right) \! \! \frac{\lfloor\delta\boldsymbol{\theta}\rfloor^2}{\Vert\delta\boldsymbol{\theta}\Vert^2}
            \end{equation}
	\end{theorem}
	\begin{proof}
            The proof is given in Appendix \ref{app_error_sys}.
	\end{proof}
	
	\begin{lemma}
		\label{lemma_linear_error_sys}
		The first-order linearization of the error-state dynamics given in (\ref{e_error_state_sys}) is:
		\begin{equation}
			\delta \dot{\mathbf{x}} = \mathbf{F}_\mathbf{x} \delta\mathbf{x} + \mathbf{F}_\mathbf{u} \delta\mathbf{u} + \mathbf{F}_\mathbf{w} \delta\mathbf{w}
			\label{e_linear_error_sys}
		\end{equation}
		where
		\begin{subequations}
                \label{e_aerodynamic_model_linear_matrix}
			\begin{align}
				&\mathbf{F}_\mathbf{x} \!=\!\! \begin{bmatrix}
					\mathbf{0} \!&\! \mathbf{I}_3 &\!\!\!\! \mathbf{0} \\
					\mathbf{0} \!&\! \mathbf{M}_{\mathbf v} &\!\!\!\! \mathbf{M}_\mathbf{R} \\
					\mathbf{0} \!&\! \mathbf{0} &\!\!\!\! - \lfloor \boldsymbol{\omega}_d \rfloor 
				\end{bmatrix} ,  
				\mathbf{F}_\mathbf{u} \!=\!\! \begin{bmatrix}
					\mathbf{0} & \mathbf{0} \\
					\mathbf{M}_{T} & \mathbf{0} \\
					\mathbf{0} & \mathbf{I}_3
				\end{bmatrix},  \mathbf{F}_{\mathbf w} \!=\!\! \begin{bmatrix}
				\!\mathbf{0} \\ \!-\mathbf{M}_\mathbf{v} \\ \mathbf{0}
			\end{bmatrix} \\
				&\mathbf{M}_{T} = \mathbf{R}_d \mathbf{e}_1 \\
				&\mathbf{M}_\mathbf{v} = \frac{1}{m} \mathbf{R}_d \frac{\partial \mathbf{f}_{a_d}}{\partial \mathbf{v}_{a_d}^\mathcal{B}} \mathbf{R}_d^T \\
				&\mathbf{M}_\mathbf{R} = \mathbf{R}_d  \left(- a_{T_d} \lfloor \mathbf{e}_1 \rfloor  -  \lfloor \frac{\mathbf{f}_{a_d}}{m} \rfloor  + \frac{\partial \mathbf{f}_{a_d}}{\partial \mathbf{v}_{a_d}^\mathcal{B}}  \lfloor \frac{\mathbf{v}^{\mathcal{B}}_{a_d}}{m} \rfloor \right) \\
                &\frac{\partial \mathbf{f}_{a_d}}{\partial \mathbf{v}_{a_d}^\mathcal{B}} = \left. \frac{\partial \mathbf{f}_{a}}{\partial \mathbf{v}_{a}^\mathcal{B}} \right|_{ \mathbf v_{a_d}^{\mathcal{B}}} \quad \quad \text{(see Equation (\ref{e_pfa_pvb}))} \\
                & \delta \mathbf{w} = \mathbf{w} - \bar{\mathbf{w}}
			\end{align}
		\end{subequations}
	\end{lemma}
	\begin{proof}
		The proof is given in Appendix \ref{app_linear_error_sys}.
	\end{proof}

        \begin{remark}
            \label{remark_linear_error_sys}
            The error system in (\ref{e_linear_error_sys}) is valid for any desired state-input trajectory $(\mathbf x_d, \mathbf u_d)$. This is because the system matrix $\mathbf F_{\mathbf x}$, input matrix $\mathbf F_{\mathbf u}$ and $\mathbf F_{\mathbf w}$ can always be calculated properly without encountering any singularities at any desired state $\mathbf R_d$ and $\mathbf v_d$. For the calculation of $\frac{\partial \mathbf{f}_{a_d}}{\partial \mathbf{v}_{a_d}^\mathcal{B}}$, as shown in (\ref{e_pfa_pvb}), it involves the calculation of $\frac{\partial \mathbf{c}}{\partial \alpha}$ and $\frac{\partial \mathbf{c}}{\partial \beta}$, which are invalid when $\mathbf v_{a_d} = \mathbf 0$. Fortunately, regardless of the values of $\frac{\partial \mathbf{c}}{\partial \alpha}$ and $\frac{\partial \mathbf{c}}{\partial \beta}$, we always have $\lim_{\mathbf{v}_a \to \mathbf 0} \frac{\partial \mathbf{f}_{a_d}}{\partial \mathbf{v}_{a_d}^\mathcal{B}} = \mathbf{0}$ according to (\ref{e_pfa_pvb}). Consequently, the linearized error-state system in (\ref{e_linear_error_sys}) has no singularities within the entire flight envelope.
        \end{remark}        
	
	\subsection{On-manifold MPC for trajectory tracking}
	With the error-state dynamics (\ref{e_linear_error_sys}), which is a standard linear time varying system, a MPC that minimizes the state $\delta
	\mathbf{x}$ and $\delta \mathbf{u}$ is utilized for trajectory tracking. Setting the unknown disturbance $\delta \mathbf{w}$ in (\ref{e_linear_error_sys}) to zero, the MPC is an optimization problem as follows:
	\begin{equation}
		\begin{aligned}
			\delta \mathbf{u}^* \!= \!&\arg\min_{\delta\mathbf{u}_k} \sum_{k=0}^{N-1} \! \left(\Vert \delta\mathbf{x}_k\Vert^2_{\mathbf{Q}_k} \!\! + \! \Vert\delta\mathbf{u}_k \Vert^2_{\mathbf{R}_k}\right) \! + \! \Vert \delta\mathbf{x}_N\Vert^2_{\mathbf{P}_N} \\
			&\  \mathrm{s.t.} \quad \delta\mathbf{x}_{k+1} \!=\! \left(\mathbf{I}_9 \!+ \! \Delta t \mathbf{F}_{\mathbf{x}_k}\right) \delta\mathbf{x}_k \!+ \! \Delta t\mathbf{F}_{\mathbf{u}_k} \delta\mathbf{u}_k \\
			&\  \qquad \ 
			\delta \mathbf{x}_0 = \delta\mathbf{x}_{\rm init} \\
			&\  \qquad \ 
			\delta \mathbf{u}_{k} \in \delta \mathbb{U}_k, \quad k = 0, \cdots , N-1 
		\end{aligned}  
		\label{e_error_state_mpc}
	\end{equation}
	where $N$ is the predictive horizon, and $\mathbf{Q}_k,  \mathbf{R}_k, \mathbf{P}_N$ are positive-definite diagonal matrices, denoting the penalty of the stage state, stage input and terminal state, respectively. $\mathbb{U}_k = \{\delta \mathbf{u}\in \mathbb{R}^{m} | \mathbf{u}_{\min} - \mathbf u_k^d \leq \delta \mathbf{u} \leq \mathbf{u}_{\max} - \mathbf u_k^d \}$ is the constraints for the input error that is derived from the actual input constraints $\mathbf{u}_{\min} \leq \mathbf{u}_k \leq \mathbf{u}_{\max}$.  The optimization in (\ref{e_error_state_mpc}) is a standard quadratic programming (QP) problem, which can be solved efficiently by existing QP solvers. Finally, the optimal control command at the current step is
    \begin{equation}
        \mathbf{u}_{\rm cmd} = \mathbf{u}_{d_0} + \delta\mathbf{u}_0^*
    \end{equation}

    \begin{remark}
        The MPC in (\ref{e_error_state_mpc}) is minimally parameterized and singularity-free. The minimal parameterization results from the use of error state $\delta \mathbf x$ in the controlled system (\ref{e_linear_error_sys}), which parameterizes the original state $\mathbf x$ on the state manifold $\mathcal{M}$ in its homeomorphic space.  This space, being  a normal Euclidean space, has the same dimension (i.e., 9) as the state manifold $\mathcal{M}$. The resultant MPC formulation ({\ref{e_error_state_mpc}}) does not have any redundant parameters when compared with existing quaternion-based MPC for UAV control \citep{falanga2018pampc, sun2022comparative}. The singularity-free property of the MPC is two-folds. First, the MPC is not singular to the flight trajectory because the error system (\ref{e_linear_error_sys}) is always valid in the entire flight envelope. Second, the MPC is not singular to the parameterization $\delta \mathbf x$. Common minimal parameterization of manifolds, such as Euler angles \citep{kamel2017linear, nguyen2021model}, parameterizes the manifold with respect to a fixed point on the manifold, the resultant parameterization is singular at certain configurations. In contrast, our error state $\delta \mathbf x$ parameterizes the state manifold with respect to each point on the reference trajectory (as opposed to a fixed point). If the feedback MPC controller is stable (as it always needs to be), the error state is stabilized around zero and hence avoids the singularity effectively. The minimally-parameterized, yet singularity-free nature of our MPC, avoids any switching in parameterization or control scheme, and eventually leads to a global trajectory tracking controller. 
    \end{remark}

    \begin{remark}
        The MPC in (\ref{e_error_state_mpc}) is a model-based controller, where the computation of $\mathbf F_{\mathbf x}$ and $\mathbf F_{\mathbf u}$ requires the knowledge of the aerodynamic model $\mathbf f_a$ and its derivative (see (\ref{e_aerodynamic_model_linear_matrix})). This enables the MPC to exploit a high-fidelity aerodynamic model of the vehicle to achieve high-accuracy tracking control while effectively admits practical constraints, such as the input saturation. 
    \end{remark}

\begin{remark}
	The error-state dynamics  formally derived in Section. \ref{sec_error_state} are globally equivalent to the original system. This equivalence allows to treat the tail-sitter as a formal nonlinear system. The system is further linearized along the reference state-input trajectory at each point, leading to a linear time-varying system in (\ref{e_linear_error_sys}).  The consequent MPC design is standard,  and its convergence analysis  can be studied using established techniques in existing literature like \citep{mayne2000constrained}, and hence will not be further discussed in the rest of paper.
\end{remark}

	\section{Real-world experimental results}
	\label{sec_experiment}
	In this section, we validate the key ideas of the approach presented in this paper via real-world experiments on a quadrotor tail-sitter UAV. The algorithms of trajectory generation, flatness transform and global tracking controller are implemented to enable the vehicle to perform aggressive agile flights. Extensive  challenging indoor and outdoor field tests are demonstrated, including agile $SE(3)$ flight through consecutive narrow windows, typical tail-sitter maneuvers (transition, level flight and loiter), and extremely aggressive aerobatics (Wingover, Loop, Vertical Eight, Cuban Eight, and their combo). All experiments are successfully tested at least three times for initial verification, data collection, and  video record.
	
	\subsection{Tail-sitter UAV platform}
	\label{sec_uav_platform}
	\begin{figure}[t!] 
		\centering
		\includegraphics[width=0.9\linewidth]{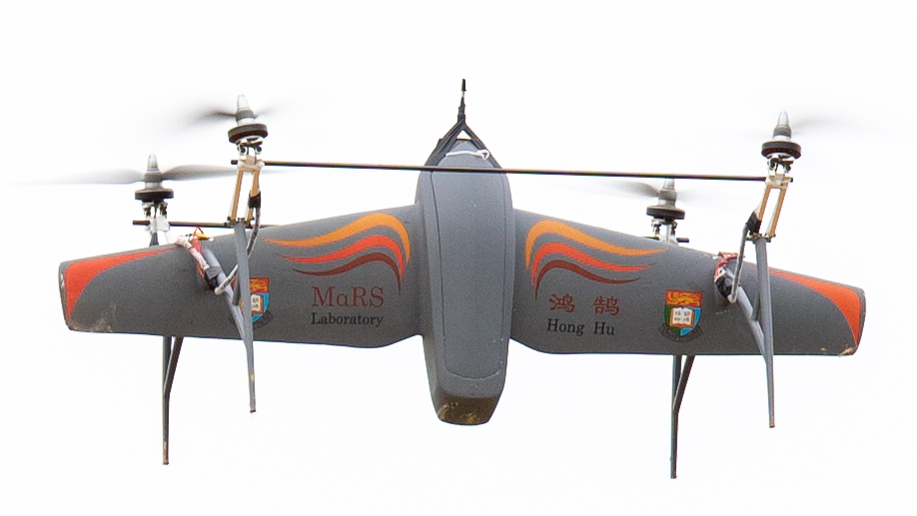} 
		\caption{Our quadrotor tail-sitter UAV prototype: Hong Hu.} 
		\label{fig_honghu_hover}
	\end{figure}

	We validate the presented algorithms on a quadrotor tail-sitter prototype, named ``Hong Hu", based on our previous airframe design  \citep{gu2018coordinate}. As shown in Fig. \ref{fig_honghu_hover}, Hong Hu is manufactured out of carbon fiber, weighs \SI{2.4}{kg}, and has a wingspan of \SI{90}{cm}. The cruise airspeed is \SI{18}{m/s}. It is powered by four T-MOTOR\footnote{https://uav-en.tmotor.com/} MN5006 KV450 motors and APC\footnote{https://www.apcprop.com/} $13\times 10$ propellers, achieving a hovering throttle at $43\%$ of the full throttle. The tail-sitter UAV is equipped with an onboard computer DJI Manifold 2-C\footnote{https://www.dji.com/manifold-2/specs}(\SI{1.8}{GHz} quad-core Intel i7 CPU) and an autopilot PX4 Mini\footnote{https://px4.io/} with a global positioning system (GPS) receiver module. A uni-axial airspeed sensor is mounted on the nose of the airframe. An action camera DJI Action 2\footnote{https://www.dji.com/dji-action-2} is fixed on a carbon rod for first-person-view (FPV) video capturing.
	
	The presented algorithms of trajectory generation and high-level tracking control (i.e., MPC) are implemented on the onboard computer, and communicated via the Robot Operating System (ROS). An open-source QP solver OOQP \citep{gertz2003object} is deployed to solve the MPC problem in (\ref{e_error_state_mpc}) at \SI{100}{Hz}. The predictive horizon is set to 12 in all experiments and the MPC takes 0.85ms in average to compute the optimal commands of thrust acceleration $a_T$ and angular velocity $\boldsymbol{\omega}$, which are then sent to the autopilot PX4 Mini via MAVROS\footnote{http://wiki.ros.org/mavros}. In the autopilot, the thrust acceleration command $a_T$ is mapped to the throttle command by ${\rm thr} = k_T a_T$, where the coefficient $k_T$ is computed as ${\rm thr}_{\rm h}/9.81$ with ${\rm thr}_{\rm h}$ being the throttle at hovering. The angular velocity command $\boldsymbol{\omega}$ is tracked by three PID controllers, each compares the respective angular velocity command with its onboard IMU measurements and calculate a normalized control torque $\boldsymbol{\tau}$ at \SI{400}{Hz}. The three PID controllers, one for each channel, are decoupled, where the coupled Coriolis term $\boldsymbol{\omega} \times \mathbf{J} \boldsymbol{\omega}$ and aerodynamic moment $\mathbf{M}_a$ are all viewed as unknown disturbances and hence ignored in the controller. In the experiments, we found the vehicle exhibited a severe vibration caused by the propeller rotation and attenuate this vibration by a Notch filter added on each PEED controller \citep{xu2019full}. The throttle  and normalized torque  are finally mixed into the four motor pulse-width modulation (PWM) commands using the standard quadrotor configuration. The vehicle state is estimated by an extended Kalman Filter (EKF) also running on the autopilot. External position and heading measurements are obtained by a motion capture system for indoor experiments or the GPS module with magnetometer for outdoor experiments. 
	
	The aerodynamic model is identified by wind tunnel tests in our previous work \citep{lyu2018simulation} and refined by real flight tests due to the new propulsion system and manufacturing. For model refining, we conduct a series of normal and inverted level flight tests in different speeds (and angle of attack), and collect the flight data of motor PWM, vehicle velocity and attitude. To ensure the sideslip angle is zero during the flights, we measure the wind speed using an anemometer, and manually set the vehicle heading along the wind direction prior to each level or inverted flights. We calculate the rotor speed from the motor PWM, the incoming airflow consisting of the measured wind speed and the vehicle's inertial speed, and then obtain the total thrust according to the open-source APC propeller  model\footnote{https://www.apcprop.com/technical-information/performance-data/}. Excluding the propeller thrusts leads to the lift and drag forces exerted on the vehicle and hence the values of $C_L$ and $C_D$ at different angle of attack $\alpha$. We conduct the flight tests from low speed to high speed and iteratively refine the aerodynamic model, to achieve stable flights. For the side force coefficient $C_Y$, we use the model of \citep{lyu2018simulation} without any modification. 
	
	In all outdoor experiments without otherwise specified, the wind speed is estimated and compensated in the differential flatness transform (by setting the surrogate wind $\bar{\mathbf w}$). Referring to \cite{johansen2015estimation}, only the wind speed components in the world frame X-Y plane is estimated by an EKF that propagates a constant wind speed model based on the airspeed sensor measurement and the vehicle inertial velocity and attitude.	To avoid unstable wind speed estimation due to degraded airspeed measurements at low flight speeds, the estimated wind speed is compensated in the differential flatness transform only when the airspeed magnitude $||\mathbf v_a|| > $ \SI{5}{m/s} . When the vehicle speed is below this threshold or in all indoor experiments, no wind speed is compensated in the differential flatness transform (i.e., setting $\bar{\mathbf w} = \mathbf 0$). In all results that follow, the angle of attack and side slip angle are computed based on $\bar{\mathbf w}$ used in the flatness transform, regardless of the actually estimated wind speed.

	\subsection{SE(3) flight through narrow windows}
        \label{sec_se3_flight}
	\begin{figure}[t!] 
		\centering
		\includegraphics[width=1\columnwidth]{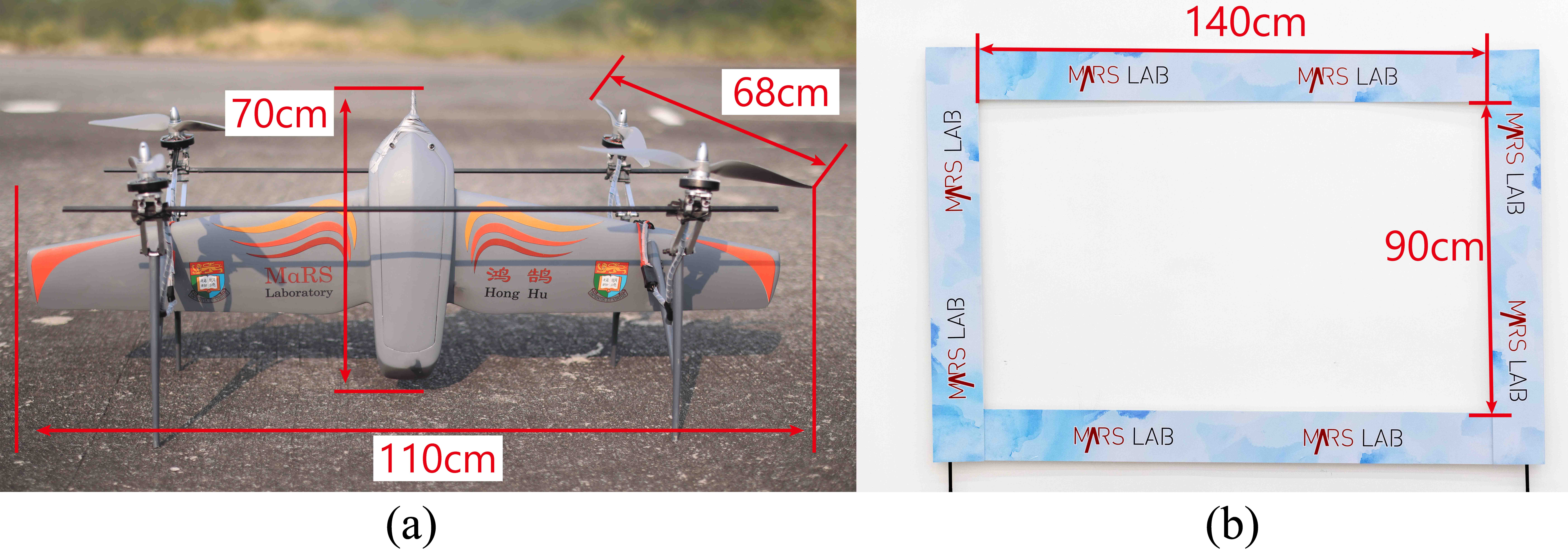} 
		\caption{Sizes of the quadrotor tail-sitter UAV (a) and the narrow window (b). } 
		\label{fig_honghu_frame_size}
	\end{figure}

	\begin{figure}[t!] 
		\centering
		\includegraphics[width=0.75\columnwidth]{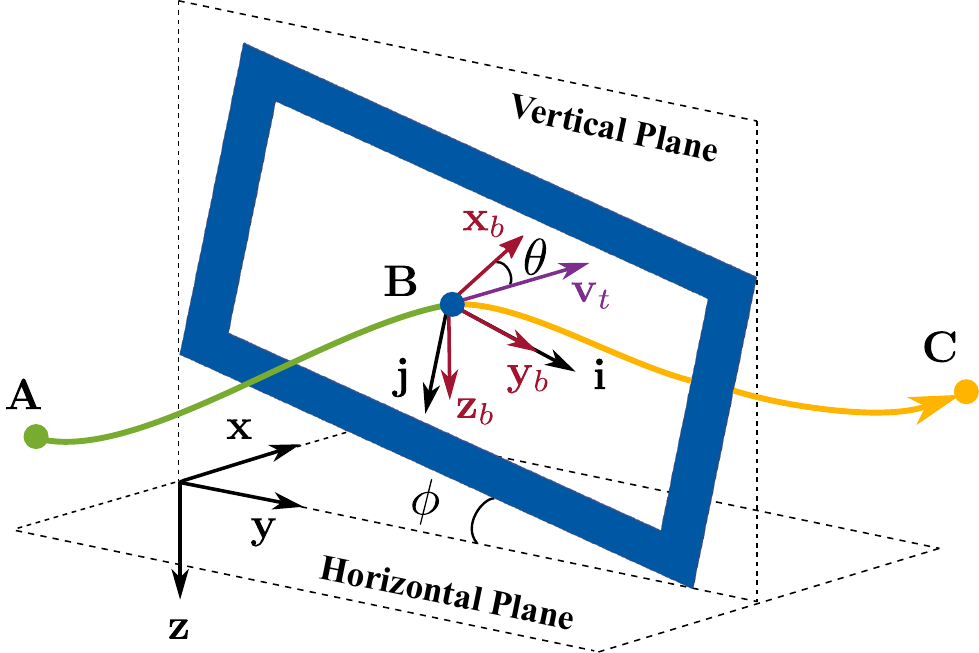} 
		\caption{The traverse trajectory is divided into two pieces: the first one (the green line) connects the start position A and the center of the window B, and the second one (the yellow line) connects B and the target position C. The specified traverse velocity $\mathbf v_t$ is perpendicular to the window plane. The body axis $\mathbf y_b$ is along the long side of the window, and the angle between $\mathbf x_b$ and $\mathbf v_t$ is specified as $\theta = 30^\circ$.} 
		\label{fig_traverse_traj}
	\end{figure}

		\begin{figure*}[ht!] 
		\centering 
		\subfigure[]{   
			\includegraphics[width=1\linewidth]{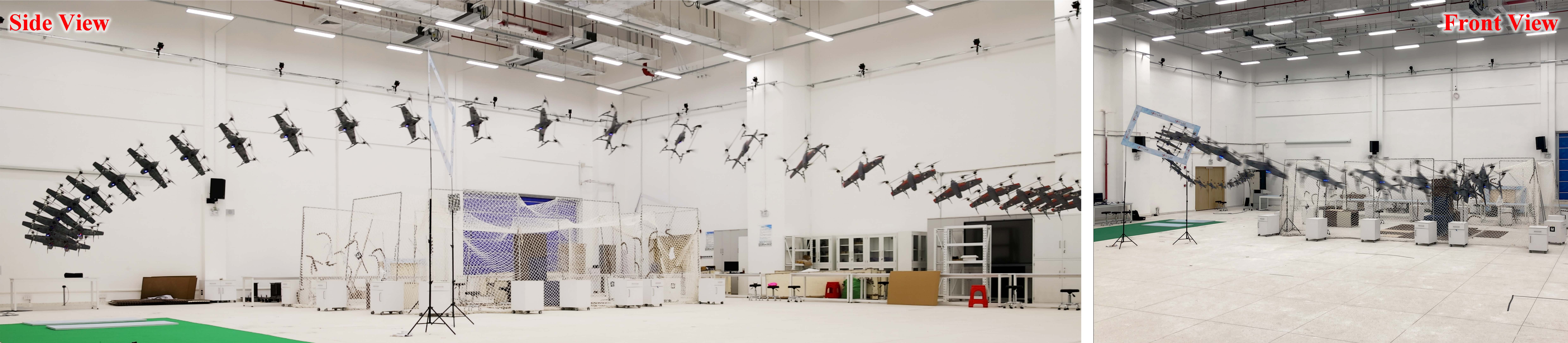} 
			\label{fig_fly_thr_single_window_overlap}    
		}
		\subfigure[]{   
			\includegraphics[width=1\linewidth]{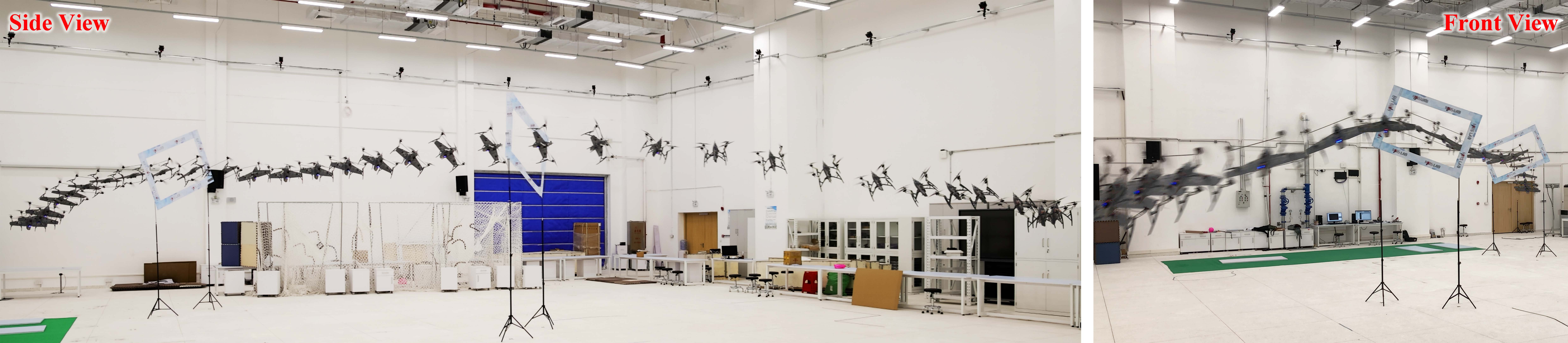}  
			\label{fig_fly_thr_double_window_overlap}   
		}    
		\caption{Snapshot sequences of agile tail-sitter flight through narrow windows. (a) Flying through a single window in \SI{10}{m/s}. (b) Flying through two consecutive windows in \SI{8 }{m/s}.}  
		\label{fig_fly_thr_window_overlap}
	\end{figure*}
	
	\begin{figure*}[h!] 
		\centering
		\includegraphics[width=1\linewidth]{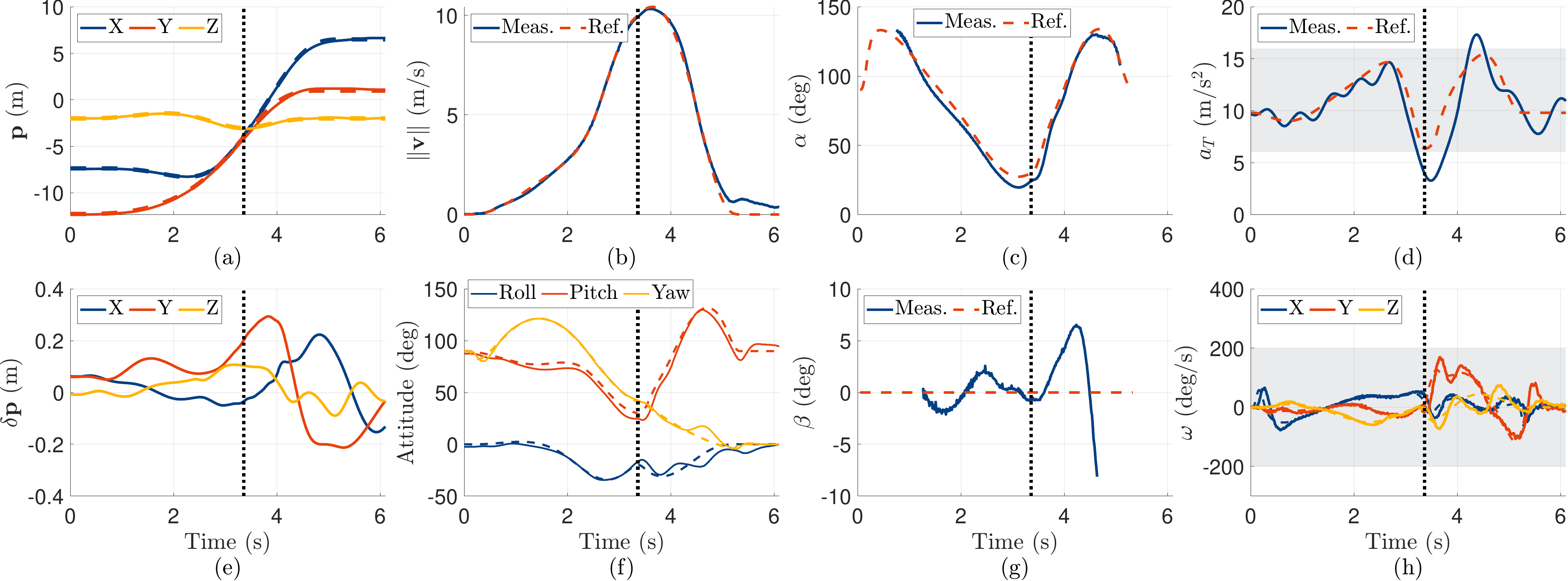} 
		\caption{Flight data of the SE(3) flight through a single window with roll $20^\circ$ and traversing speed \SI{10}{m/s}: (a) position, (b) flight speed, (c) angle of attack, (d) thrust acceleration, (e) position tracking errors, (f) attitude in Euler angles, (g) sideslip angle, (h) angular velocity. In all subplots where applicable, the solid and dashed lines denote the measurement and reference, respectively. For the thrust acceleration, the measurement is obtained from the accelerometer X axis. For the angle of attack and sideslip angle, their measurements are displayed only when the airspeed exceeds \SI{1}{m/s} due to the unstable airspeed measurements at low speeds. The vertical dotted lines denote the moment the vehicle passes the window, and the shaded areas in (d) and (h) denote the feasible  region of the actuation in trajectory optimization (\ref{e_traj_opt}).} 
		\label{fig_thr_single_window_v10}
	\end{figure*}
	
	\begin{figure*}[h!] 
		\centering
		\includegraphics[width=1\linewidth]{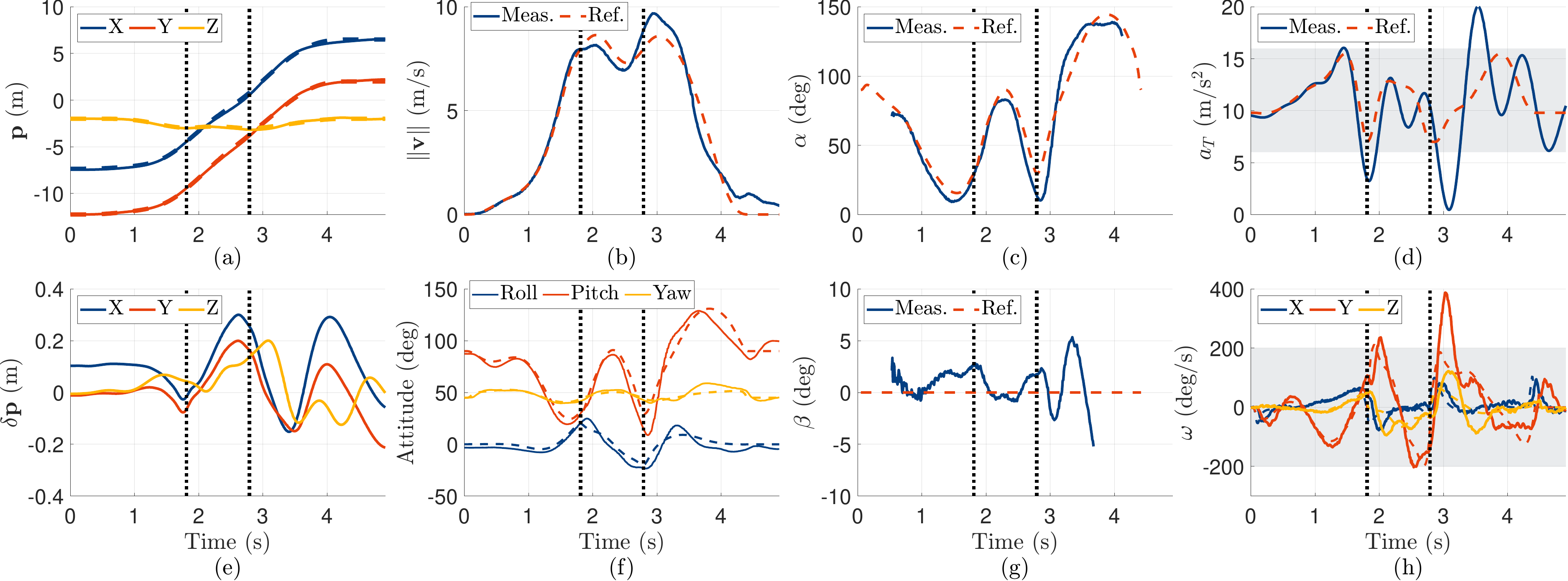} 
		\caption{Flight data of the SE(3) flight through two consecutive windows with roll angles $20^\circ$ and $-20^\circ$ and traversing speeds both at \SI{8}{m/s}: (a) position, (b) flight speed, (c) angle of attack, (d) thrust acceleration, (e) position tracking errors, (f) attitude in Euler angles, (g) sideslip angle, (h) angular velocity. In all subplots where applicable, the solid and dashed lines denote the measurement and reference, respectively. For the thrust acceleration, the measurement is obtained from the accelerometer X axis. For the angle of attack and sideslip angle, their measurements are displayed only when the airspeed exceeds \SI{1}{m/s} due to the unstable airspeed measurements at low speeds. The vertical dotted lines denote the moments when the vehicle passes the windows, and the shaded areas in (d) and (h) denote the feasible  region of the actuation in trajectory optimization (\ref{e_traj_opt}).} 
		\label{fig_thr_double_window_v8}
	\end{figure*}

	Flying through narrow windows is a challenging but potentially worthwhile scenario that a UAV can navigate in obstacle-dense environments, such as searching through thick forest or collapsed buildings after disasters. The main challenge of the problem is that the vehicle can fly through the narrow window only when its body is aligned with the window orientation to fit the limited traversing space as shown in Fig. \ref{fig_honghu_frame_size}. This task requires the UAV to execute a precise, aggressive full body motion on $SE(3)$ (i.e., $SE(3)$ flight). For the sake of flight agility and tracking accuracy, dynamical feasibility of the trajectory should be guaranteed rigorously in planning, such that the tracking error can be reduced when the vehicle executes the maneuver.
	
	To generate a collision-free and dynamically feasible trajectory through a narrow window, we divide the trajectory into two pieces (i.e., before and after passing through the window), and optimize them by (\ref{e_traj_opt}) separately. As shown in Fig. \ref{fig_traverse_traj}, the first trajectory (the green line) connects the UAV start position to a traversing position fixed at the center of the window, and the second trajectory (the yellow line) connects the traversing position to the target position. To determine the boundary conditions for these two trajectories, the speed, acceleration, and jerk at the start and target positions are all set to zeros (i.e., stationary hovering). For the traversing position, the position is the center of the window, velocity $\mathbf v_t$ is normal to the window plane with magnitude manually specified. To determine the acceleration at the traversing position, we specify the body Y axis to be along the window long edge and set the body X axis to form a $\theta = 30^\circ$ angle with the traversing velocity $\mathbf v_t$ (i.e., AoA is $\alpha = 30^\circ$ at the traversing position). Then, we choose the thrust $a_T$ by minimizing the total acceleration $\mathbf a_t$ at the traversing position: 
	\begin{equation}
		\begin{aligned}
			\min_{a_T} \Vert {\mathbf{a}_t} \Vert &= \min_{a_T} \Vert \mathbf{R} ( a_T\mathbf{e}_1 + \mathbf{R}^T\mathbf{g} + \frac{1}{m}\mathbf{f}_a ) \Vert\\
			&= \min_{a_T} \Vert a_T\mathbf{e}_1 + \mathbf{R}^T\mathbf{g} + \frac{1}{m}\mathbf{f}_a \Vert \\
			&= \min_{a_T} \Vert a_T + \mathbf{e}_1^T (\mathbf{R}^T\mathbf{g} + \frac{1}{m}\mathbf{f}_a) \Vert
		\end{aligned}	
	\label{e_uncon_min_acc}	
	\end{equation}
	Taking the thrust constraint into consideration, (\ref{e_uncon_min_acc}) leads to a constrained linear optimization:
	\begin{equation}
		\begin{aligned}
			a_T^* = &\arg\min_{a_T} \Vert a_T + \mathbf{e}_1^T (\mathbf{R}^T\mathbf{g} + \frac{1}{m}\mathbf{f}_a) \Vert \\
			&\mathrm{s.t.} \quad a_{T_{\rm min}} \leq a_T \leq a_{T_{\rm max}}
		\end{aligned}
	\end{equation}
	where $a_{T_{\rm min}}$ and $a_{T_{\rm max}}$ are the boundaries of thrust acceleration. Then, the traversing acceleration can be obtained by substituting the optimal thrust acceleration and the determined attitude into the translational dynamics in (\ref{e_translation_dyn}). Finally,  the traversing jerk is set to zeros for simplicity.

	We validate the algorithms in real-world experiments as shown in Fig. \ref{fig_fly_thr_window_overlap}. 
 The kinodynamic and control input constraints of the planner are $v_{\rm max} = $ \SI{12}{m/s},  $a_{T_{\rm min}} = $ \SI{6}{m/s^2},  $a_{T_{\rm max}} = $ \SI{16}{m/s^2}, and $\omega_{\rm max} = $ \SI{200}{deg/s}. To increase the tracking accuracy for position and attitude, which is crucial for the UAV to pass the window, parameters of the MPC are set as $\mathbf{Q}_k = \text{diag([1800, 1800, 1800, 5, 5, 5, 50, 50, 50])}$, $\mathbf{R}_k = \text{diag([0.3, 0.4 0.4, 0.4])}$, and $\mathbf{P}_N = \mathbf{Q}_k$. All poses of the windows and UAV are measured by a motion capture system. The flying volume is about  $15 \times 15 \times $ \SI{ 4}{ m^3}.
	
	In the first scenario, the tail-sitter performs aggressive $SE(3)$ flights to fly through a single window. Fig. \ref{fig_fly_thr_single_window_overlap} and Fig. \ref{fig_thr_single_window_v10} respectively show the snapshot sequence and experimental data of a successful flight passing through a window with roll angle $\phi = 20^\circ$ and in a traversing speed of $\| \mathbf v_t \| = $ \SI{10}{m/s}. As can be seen in Fig. \ref{fig_thr_single_window_v10}\textcolor{red}{(a-b)}, to fly through the window with the specified speed, the UAV must accelerate from stationary hovering to the traversing speed (i.e., \SI{10}{m/s}) in a time less than 3.4 seconds and a space within $3.6\times 8.1\times $ \SI{1}{m^3}. To achieve this, the UAV performs transition and a banked turn simultaneously (see Fig. \ref{fig_thr_single_window_v10}\textcolor{red}{(f)}). In fact, the planner and controller are not even aware of the transition, but treats the entire flights uniformly. Then the UAV traverses the window with the required pose and velocity at \SI{3.36}{s} (the vertical black dotted line) and finally recover to the hovering status again within a very limited flight space. During the flight, the angle of attack varies up to $113^\circ$ in merely two seconds (see Fig. \ref{fig_thr_single_window_v10}\textcolor{red}{(c)}), indicating a large envelope of angle of attack. Despite this, the overall position error as shown in Fig. \ref{fig_thr_single_window_v10}\textcolor{red}{(e)} is less than \SI{0.3}{m} and the slideslipe angle as shown in Fig. \ref{fig_thr_single_window_v10}\textcolor{red}{(g)} is well stabilized around zero. The seemly large sideslip angle at the beginning and end of the flight is due to the unstable airspeed measurements at very low speeds. Fig. \ref{fig_thr_single_window_v10}\textcolor{red}{(d)} and \textcolor{red}{(h)} show that the trajectory planner effectively bounds the thrust acceleration and angular velocity of the reference trajectory within the nominal actuator constraints (the shaded area).  
	  	
	In the second scenario, the tail-sitter performs more aggressive $SE(3)$ flights to fly through two consecutive windows. Fig. \ref{fig_fly_thr_double_window_overlap} and Fig. \ref{fig_thr_double_window_v8} respectively show the snapshot sequence and experimental data of a successful flight with window roll angles $-20^\circ$ and $20^\circ$ and traversing speeds both at \SI{8}{m/s}. As shown in Fig. \ref{fig_fly_thr_double_window_overlap} and Fig. \ref{fig_thr_double_window_v8}\textcolor{red}{(a), (b) and (f)}, the UAV traverses the first window at \SI{1.81}{s}, then immediately pulls up the pitch angle, which slows down the speed, to gain sufficient lift maintaining the height. After this, the UAV pitches down and accelerates again to fly through the second window safely at \SI{2.79}{s}. The fact that the maneuver in this scenario is more aggressive than the former, is also shown in Fig. \ref{fig_thr_double_window_v8}\textcolor{red}{(d)} and \textcolor{red}{(h)} where the IMU measurements of thrust acceleration and angular velocity reach \SI{20}{m/s^2} and \SI{400}{deg/s}, respectively. The position tracking error in Fig. \ref{fig_thr_double_window_v8}\textcolor{red}{(e)} is consequently larger, but the overall position error remains less than \SI{0.3}{m}. Other phenomenons, such as the large envelope of angle of attack, simultaneous bank turn and transition, and stabilization of the sideslip angles, are all similar to the previous experiment.

	\begin{table}[t!]
		\centering
		\small
		\captionof{table}{Average Pose tracking error of $SE(3)$ flights.}
		\begin{threeparttable} 
			\begin{tabular}{l l l l }
				\toprule [1 pt]
				Roll Angle $\phi$ & $\Vert \mathbf{v}_t \Vert$ (m/s) & $ \delta \mathbf{p}_{\rm RMS} $(cm) & $ \delta \boldsymbol{\theta}_{\rm RMS} (^\circ)$  \\
				\hline
				$0^\circ$  &  8 &  10.8 & 4.7   \\
				$20^\circ$ &  8 &  13.5 & 6.4 \\
				$40^\circ$ &  8 &  9.5 & 4.0 \\
				$20^\circ$ &  3 &  10.7 & 4.5 \\
				$20^\circ$ &  5 &  12.8 & 5.2 \\	
				$20^\circ$ &  10 & 9.5 & 4.0 \\
				
				\hline
				$0^\circ$ \& $20^\circ$ &  8 &  11.6 & 6.2 \\
				$20^\circ$ \& $-20^\circ$ &  8 & 10.0 & 5.5 \\
				$40^\circ$ \& $20^\circ$ &  8 &  12.0 & 6.6 \\
				\toprule [1 pt]
			\end{tabular}
		\end{threeparttable} 
		\label{tab_fly_thr_window}
		\vspace{-4mm}
	\end{table} 

	\begin{figure*}[th!] 
		\centering
		\includegraphics[width=1\linewidth]{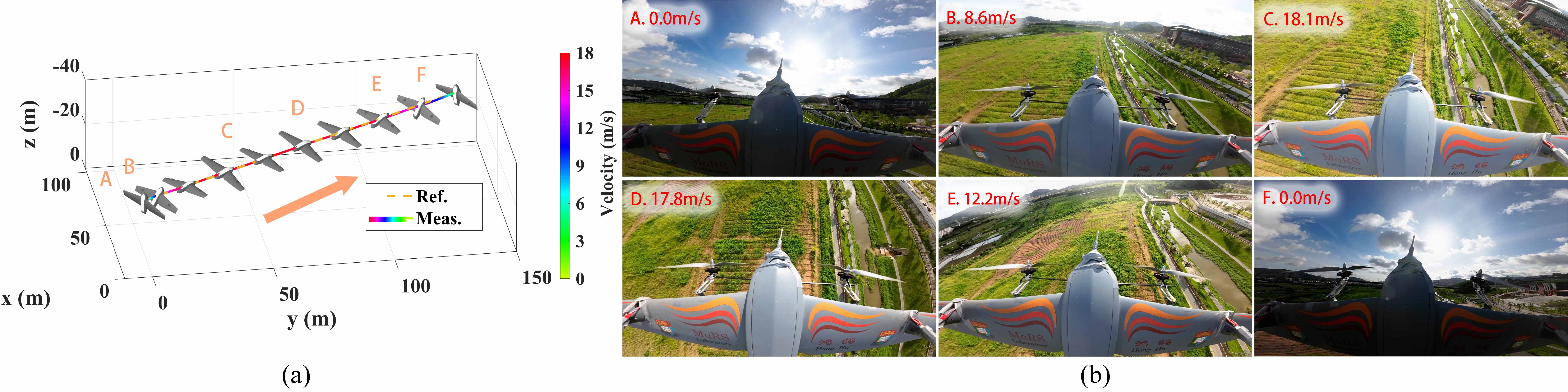} 
		\caption{Forward flight of the straight-line path in \SI{18}{m/s}:  (a) trajectory illustration, (b) images from the FPV camera. Labels A-F denote different flight phases of the vehicle: A. hovering, B. forward transition, C. and D. level flight, E. backward transition, F. hovering.} 
		\label{fig_forward_flight}
		\vspace{-4mm}
	\end{figure*} 
	
	\begin{figure*}[th!] 
		\centering
		\includegraphics[width=1\linewidth]{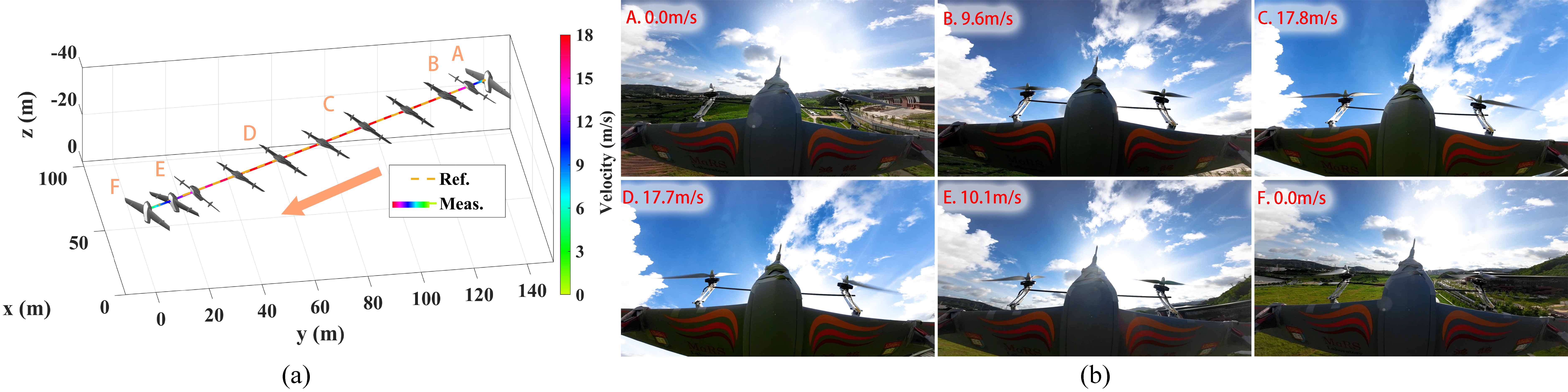} 
		\caption{Inverted flight of the straight-line path in \SI{18}{m/s}: (a) trajectory illustration, (b) images from the FPV camera. Labels A-F denote different flight phases of the vehicle: A. hovering, B. inverted forward transition, C. and D. inverted level flight, E. interted backward transition, F. hovering.} 
		\label{fig_inverted_fligt}
		\vspace{-3mm}
	\end{figure*}

	To provide more convincing results, we conduct two test groups of experiments demonstrating the flights through single and double windows, respectively. The first group consists of six different flight tests with a window roll angle  $\phi \in \{0^\circ, 20^\circ, 40^\circ\}$ and a traversing speed $\| \mathbf v \| \in \{3, 5, 8, 10\}$ m/s.  The second group consists of three $SE(3)$ flights with window angles combinations drawn from $\{-20^\circ, 0^\circ, 20^\circ, 40^\circ\}$ and a traversing speed of \SI{8}{m/s}. All nine experiments are successfully conducted with results summarized in Tab. \ref{tab_fly_thr_window}.  The first group results demonstrate a sufficiently high control accuracy to avoid collision (i.e., the maximum average position and attitude error are 13.5 cm and $6.4^\circ$, respectively). It also shows that the proposed trajectory generation in coordinated flight is applicable to low-speed $SE(3)$ flights (down to \SI{3}{m/s}). For the second group results, the pose tracking error slightly increases at the second window due to the dramatic attitude and velocity variations as mentioned before, but still small enough for the UAV to pass through the window. To sum up, the varioius agile flights through narrow windows demonstrate that the proposed trajectory generation and control framework is capable to execute accurate $SE(3)$ flights, which shows a promising application to  aggressive autonomous flight with obstacle avoidance in cluttered environments. Readers are encouraged to watch the accompanying videos for better visualization of the experiments.

	\subsection{Typical maneuvers in field environments}
        \label{sec_typical_maneuvers}
	In this task, we examine the effectiveness and performance of the proposed algorithms for typical maneuvers in field environments. We test a straight-line maneuver (including hovering, transition and level flight) and loiter flights with speed ranging from \SI{5}{m/s} to \SI{20}{m/s}, and make comparisons to conventional tail-sitter controllers (with details supplied later). We reserve the same parameters of the planner and controller, except decreasing the MPC position penalty (i.e., the first three diagonal elements of $\mathbf{Q}_k$) to $[1200, 1200, 1200]$ to  increase the robustness to uncertainties like unmeasured wind disturbance, and noisy  GPS measurement in outdoor environments.

	\subsubsection{Straight-line flight}~
	\label{sec_straight_line}

	Transition and level flights are two crucial maneuvers for tail-sitter UAVs and are commonly tested for tail-sitter controllers.  We demonstrate the proposed framework on these maneuvers via a forward flight trajectory, which involves three maneuvers: forward transition, level flight, and backward transition (see Fig. \ref{fig_forward_flight}\textcolor{red}{(a)}, and an inverted flight trajectory, which involves another three maneuvers: inverted forward transition, inverted level flight, and inverted backward transition (see Fig. \ref{fig_inverted_fligt}\textcolor{red}{(a)}.  We present the tracking performance on these trajectories with different level-flight speed ranging from \SI{5}{m/s} to \SI{20}{m/s}, and make comparisons against existing works in terms of  transition accuracy.  
	
	We design the forward and inverted flight trajectories along the same straight-line path, where the vehicle first flies forward along the path to a target position and then flies in an inverted pose along the same path back to the origin, as shown in Fig. \ref{fig_transition_18}. Both forward and inverted flight trajectories have the same level-flight phase, which is manually specified as a constant-velocity trajectory (speed ranges from \SI{5}{m/s} to \SI{20}{m/s}) lasting for 4-5 seconds. The trajectories from the initial hovering position to the constant-velocity trajectory and that from the constant-velocity trajectory to the target hovering position are designed by the proposed trajectory optimization method in (\ref{e_traj_opt}), for both forward and inverted trajectories. 
	
		\begin{figure*}[t!] 
		\centering
		\includegraphics[width=1\linewidth]{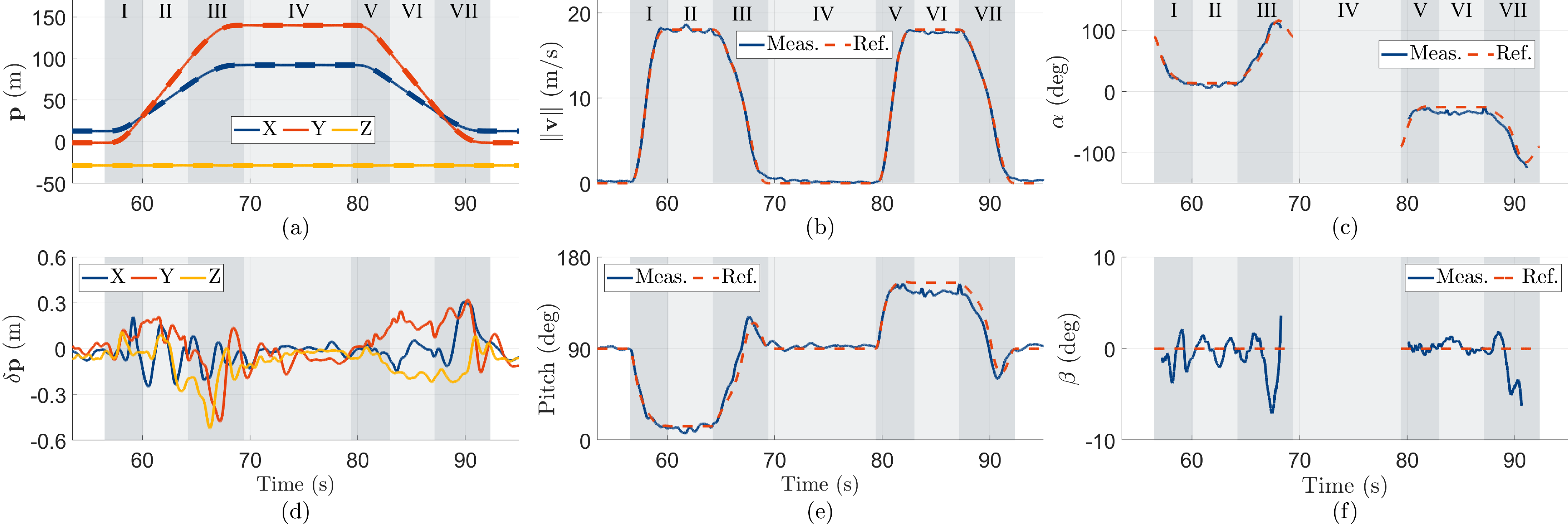} 
		\caption{Flight data of the straight-line flight test consisting of both forward and inverted flights: (a) position, (b) flight speed, (c) angle of attack, (d) position tracking errors, (e) pitch angle, (f) sideslip angle. Flight stages from I to VII divided by shaded areas indicate the I. forward transition, II. level flight, III. backward transition, IV. hovering, V. inverted forward transition, VI. inverted level flight and VII. inverted backward transition. For the angle of attack and sideslip angle, their measurements are displayed only when the airspeed exceeds \SI{2}{m/s} due to the unstable airspeed measurements at low speeds. In all subplots where applicable, the solid and dashed lines respectively denote the measurement and reference.} 
		\label{fig_transition_18}
	\end{figure*}

	\begin{figure*}[h!] 
		\centering
		\includegraphics[width=1\linewidth]{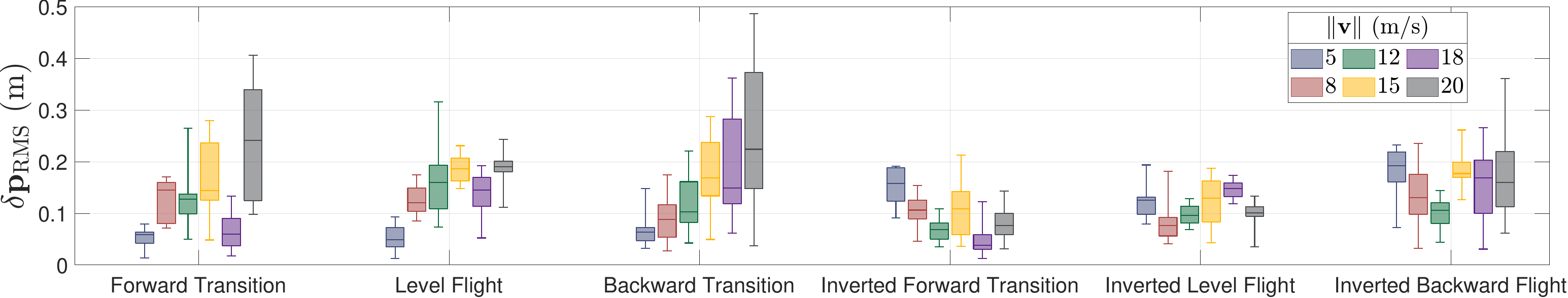} 
		\caption{The position tracking error in six different speeds when the tail-sitter flies in different phases of a straight-line path.} 
		\label{fig_tran_boxplot}
	\end{figure*}

	Fig. \ref{fig_forward_flight} and Fig. \ref{fig_inverted_fligt} show the 3-D trajectory and FPV images of the test with level-flight speed of \SI{18}{m/s}, and the corresponding flight data are detailed in Fig. \ref{fig_transition_18}. As shown in Fig. \ref{fig_transition_18}, the tail-sitter first performs a forward transition (phase I) from hovering to level flight with speed \SI{18}{m/s} (phase II) while the pitch angle decreases from $90^\circ$ to $13^\circ$. After flying \SI{112}{m} over 7.7 seconds, the vehicle performs a backward transition (phase III) to hovering (phase IV). Subsequently, the vehicle performs an inverted forward transition (phase V), where the pitch angle increases from $90^{\circ}$  to $146^\circ$, reaching the inverted level flight with speed of \SI{18}{m/s} (phase VI). Finally the vehicle performs an inverted backward transition (phase VII) to return to the initial hovering position. It is seen that the vehicle feedback trajectory of position, velocity and pitch angle tracks the reference state trajectory precisely throughout the flight. Fig. \ref{fig_transition_18}\textcolor{red}{(d)} shows the position tracking error. As can be seen, the overall tracking error is  \SI{0.13}{m} in average and  \SI{0.52}{m} at most, which is incredibly small considering that the flight speed is up to \SI{18}{m/s}, and the angle of attack varies over $230^{\circ}$ (see Fig. \ref{fig_transition_18}\textcolor{red}{(c)}). 
	
	To provide a more convincing result and demonstrate the effectiveness of the proposed framework in full-envelope flight, we conduct a group of straight-line flights with six different level-flight speeds of $\Vert \mathbf{v} \Vert \in \{5, 8, 12, 15, 18, 20\}$ m/s. The  position tracking error in each flight phase of each flight is statistically analyzed in Fig. \ref{fig_tran_boxplot}. As can be seen, the errors at all times in all 36 groups of data across different flight speeds or phases are less than \SI{0.5}{m}, showing that the proposed framework enables a tail-sitter to fly within the whole envelope in high accuracy. Notably, the tracking error during inverted flight is as low as \SI{0.2}{m}. The increased tracking accuracy in the inverted flight are due to better fitting of the aerodynamic coefficients in negative AoA regions. It is also noted that existing methods based on separated trajectory planners and controllers \cite{frank2007hover, oosedo2017optimal, lyu2017hierarchical, xu2019full} did not demonstrate such inverted flights, because the required AoA is out of the designed envelope. 
	
	 \begin{figure*}[th!] 
		\centering 
		\subfigure{   
			\includegraphics[width=1\linewidth]{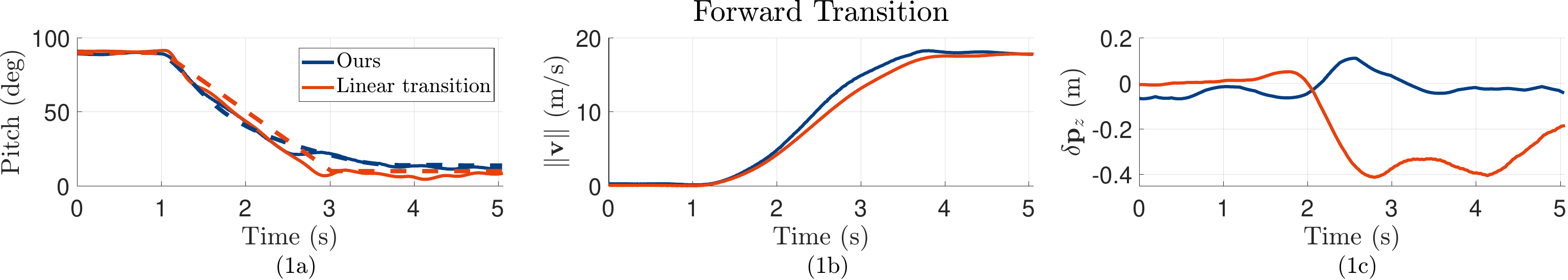} 
			\label{fig_forward_trans_compare}    
		}
		\subfigure{   
			\includegraphics[width=1\linewidth]{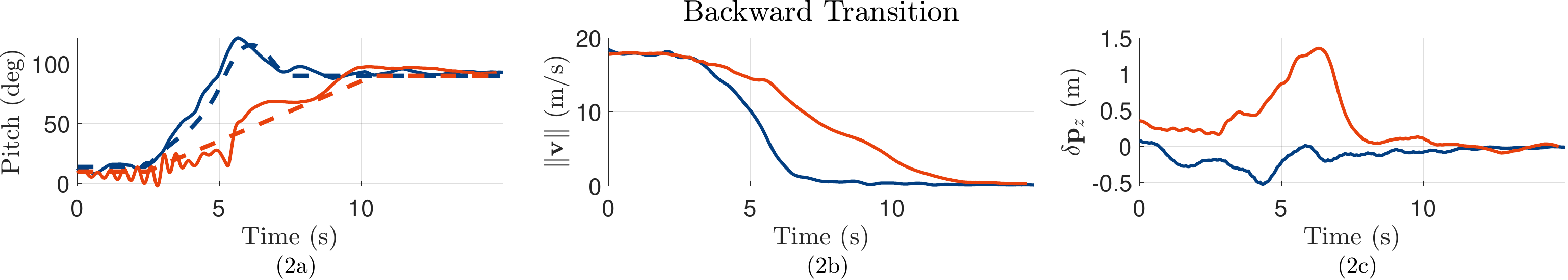}  
			\label{fig_backward_trans_compare}   
		}    
		\caption{Comparison on transition control performance between our presented approach and the linear transition method \citep{lyu2017hierarchical}. (1a)-(1c) are respectively the pitch angle, flight speed and altitude tracking error for the forward transition, and (2a)-(2c) are those for backward transition. The solid and dashed lines in (1a) and (2a) respectively denote the measurement and reference.}  
		\label{fig_linear_tran_comparison}
		\vspace{-3mm}
	\end{figure*}

	Moreover, we make a comparison on the transition accuracy with a traditional linear transition controller \citep{lyu2017hierarchical}, which is the same strategy used by the autopilot PX4. To ensure a fair comparison, we implement both  our MPC and the traditional controller with the same low-level angular velocity controller, on the same vehicle. In addition, the linear transition controller are tuned to the best extent. The cascaded attitude and altitude PID controllers of the linear transition controller are turned by Ziegler-Nichols method, while the linear Pitch reference is determined by a transition duration and angle span. The angle span indicates the Pitch change between hovering and level flight at cruise speed, and is obtained as $13^\circ - 90^\circ$ in the former straight-line flight experiment (see Fig 15(e)). The transition duration is initially set to the same transition time of our method but it failed the transition flight due to the too short transition time. Then, we gradually increase the transition duration  until successful flight is achieved.  We iteratively refine the above attitude and altitude controller, achieving comparable performance demonstrated in existing research \citep{oosedo2017optimal, lyu2017hierarchical, xu2019full}.
 
	 Because the linear transition controllers usually focus on pitch and altitude control, and have no position control in other directions, we focus on the comparison of longitudinal state variables only. Fig. \ref{fig_linear_tran_comparison}\textcolor{red}{(1a)-(1c)} shows the comparison in forward transition. Our method controls the pitch angle to decrease from $90^\circ$ to $13^\circ$ smoothly and speeds up from hovering to \SI{18}{m/s} in merely 3 seconds with altitude error peaking at \SI{0.11}{m}, while the linear method has good performance in pitch control and accelerating but the altitude drops \SI{0.41}{m}. Similarly, in backward transition shown in Fig. \ref{fig_linear_tran_comparison}\textcolor{red}{(2a)-(2c)}, our method tracks the reference pitch angle smoothly and has maximum altitude error of \SI{0.52}{m} only, while the linear method tracks the linear pitch trajectory with significant pitch fluctuations and has large altitude deviation of  \SI{1.35}{m}. It can be also noticed that our method pulls up the pitch angle over $120^\circ$ and then returns to $90^\circ$ for fast deceleration as shown in Fig. \ref{fig_linear_tran_comparison}\textcolor{red}{(2b)}, the resultant backward transition is 4 seconds shorter than the linear method. In this comparison, our model-based framework shows its advantages in tracking accuracy and flight aggressiveness, which outperforms the model-free linear transition control.

	\subsubsection{Loiter flight}~
	\begin{figure*}[h!] 
		\centering
		\includegraphics[width=1\linewidth]{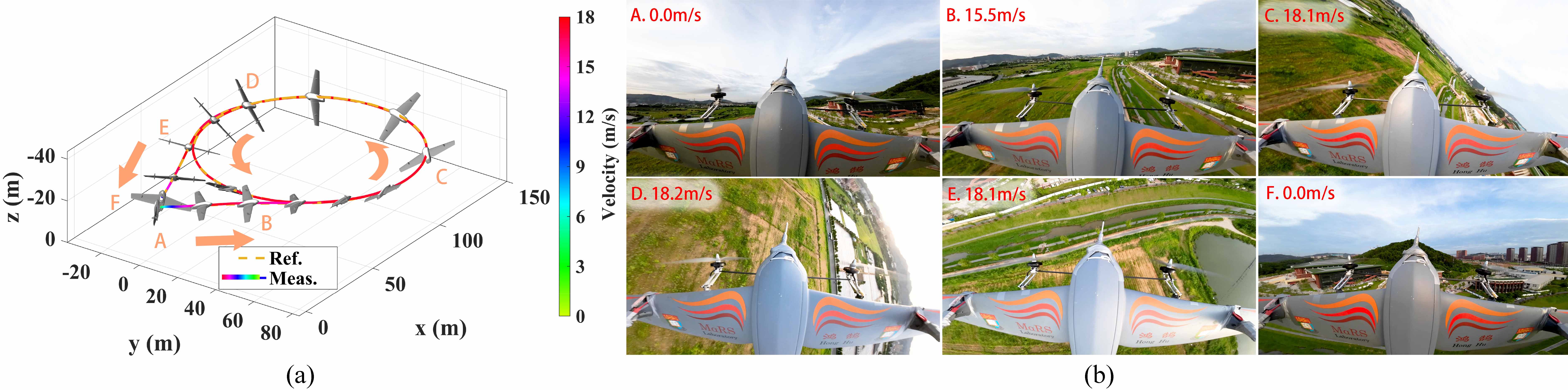} 
		\caption{Loiter flight with flight radius of 50m and velocity of \SI{18}{m/s}: (a) trajectory illustration, (b) images from the FPV camera. Labels A-F denote different states of the vehicle: A. hovering, B. banked forward transition, C and D. loiter flight, E. banked backward transition, F. hovering.} 
		\label{fig_horizontal_circle}
	\end{figure*}
	
	\begin{figure*}[h!] 
		\centering
		\includegraphics[width=1\linewidth]{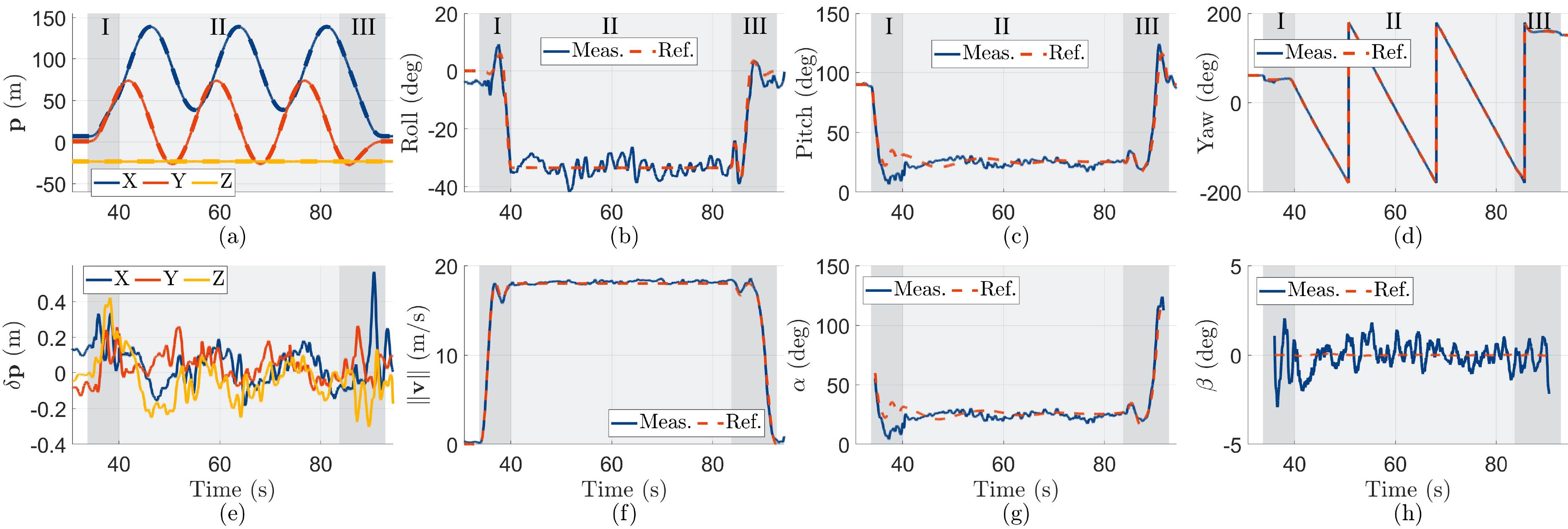} 
		\caption{Flight data of the loiter flight in \SI{18}{m/s}: (a) position, (b-d) attitude in Euler angles, (e) position tracking errors, (f) flight speed, (g) angle of attack, (h) sideslip angle. Flight stages from I to III divided by shaded areas indicate the banked forward transition, loiter, banked backward transition, respectively. For the angle of attack and sideslip angle, their measurements are displayed only when the airspeed exceeds \SI{2}{m/s} due to the unstable airspeed measurements at low speeds. In all subplots where applicable, the solid and dashed lines respectively denote the measurement and reference.} 
		\label{fig_circle_18}
	\end{figure*}

	\begin{figure}[h!] 
		\centering
		\includegraphics[width=1\linewidth]{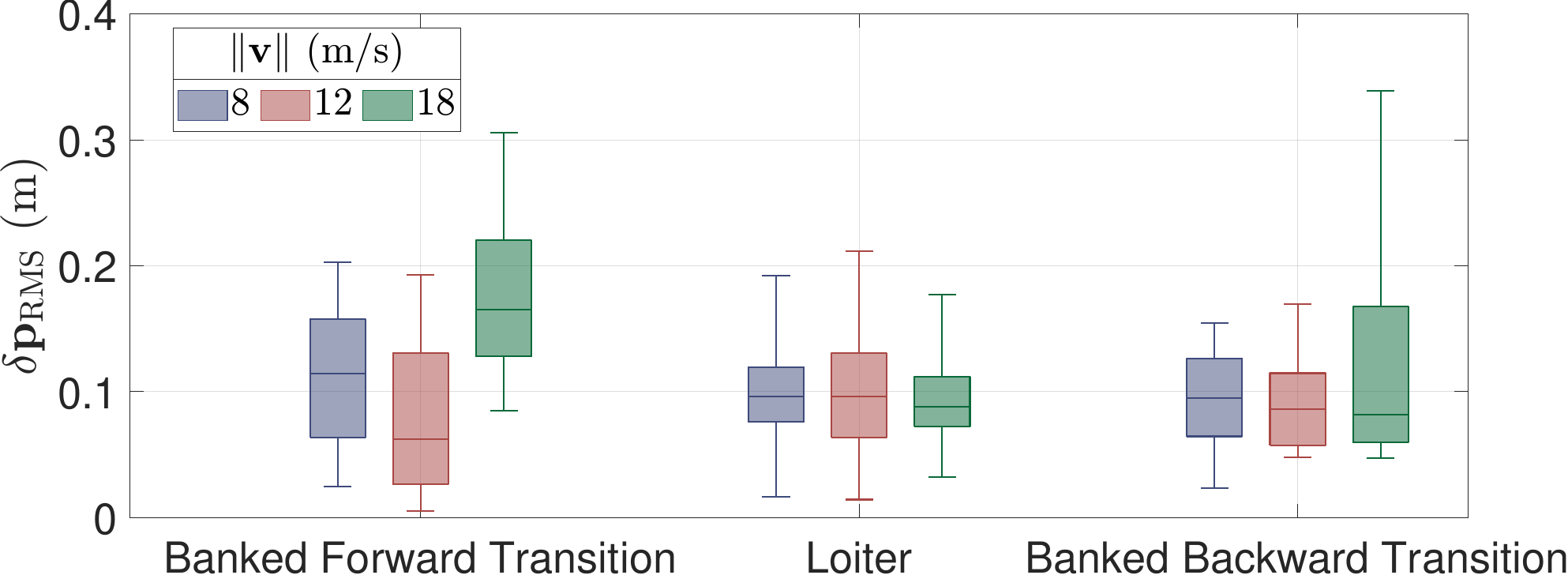} 
		\caption{Position tracking error in three different phases when the tail-sitter flies the loiter trajectory in different speeds.} 
		\label{fig_circle_boxplot}
	\end{figure}

	Loiter flight is another typical trajectory that validates the cruise performance of tail-sitters. As shown in Fig. \ref{fig_horizontal_circle}, the trajectory consists of three phases: banked forward transition from hovering to loiter, loiter flight in constant speed, and banked backward transition from loiter to hovering. 
    The loiter trajectory is designed in three steps. The constant-speed circular trajectory is first determined manually. Then the banked forward transition trajectory is optimized by (\ref{e_traj_opt}) with initial condition as the hovering state and terminal condition as the first point on the circular trajectory. Similarly, the banked backward transition trajectory is optimized by (\ref{e_traj_opt}) to perform a loiter-to-hovering maneuver.  
	
	\begin{figure*}[h!] 
		\centering
		\includegraphics[width=1\linewidth]{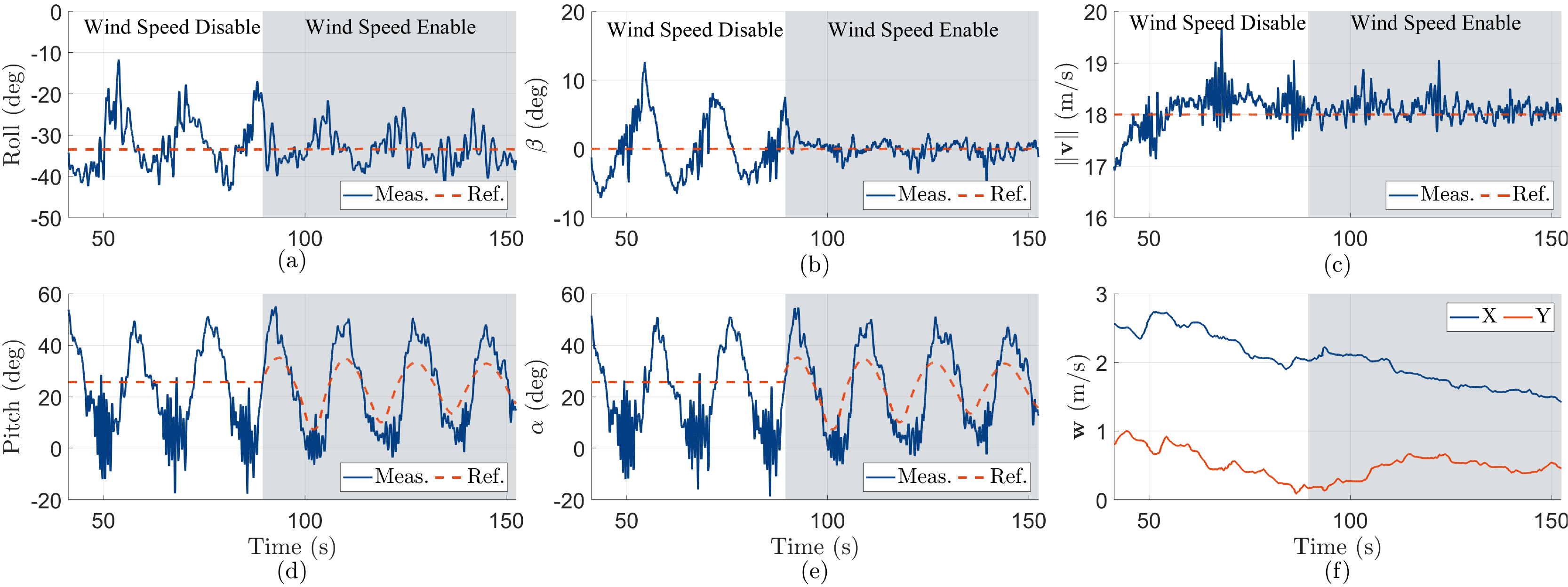} 
		\caption{Loiter flight test in \SI{18}{m/s} with and without wind speed compensation in the controller: (a) roll angle, (b) sideslip angle, (c) flight speed, (d) pitch angle, (e) angle of attack, (f) estimated wind speed. In all subplots, the white area denotes the flight when setting $\bar{\mathbf{w}} = \mathbf 0$ in the flatness function (Algorithm \ref{alg_proc_diff_flat}), while the shaded area denotes the duration when the online-estimated wind speed is used as the $\bar{\mathbf{w}}$ in the flatness transform. Note that both the measured and reference angle of attack and sideslip angle are computed based on $\bar{\mathbf w}$, regardless of the actually estimated wind speed. } 
		\label{fig_wind speed_comparison}
	\end{figure*}
	
	\begin{figure}[h!] 
		\centering
		\includegraphics[width=0.9\linewidth]{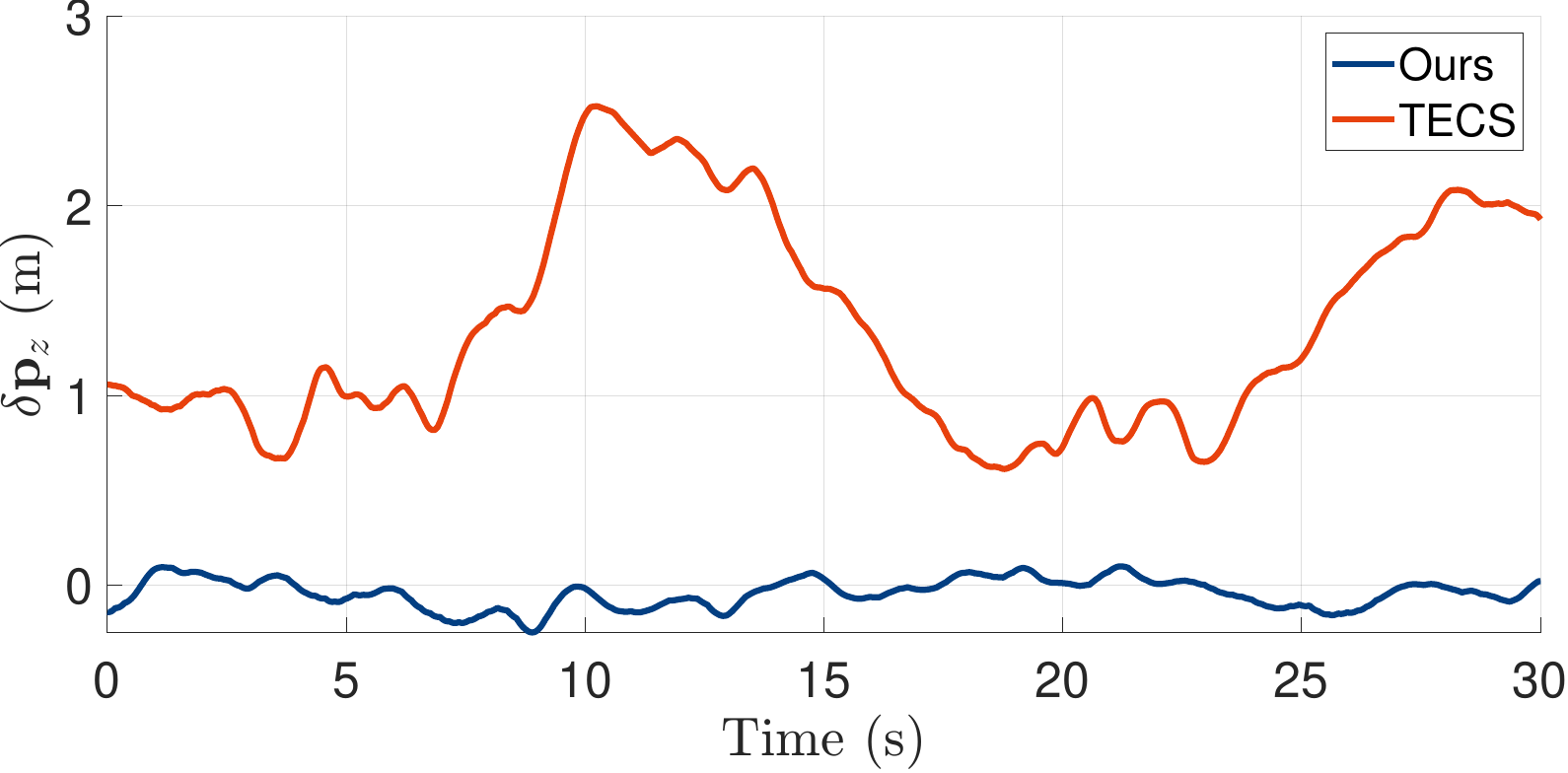} 
		\caption{Comparison on the altitude control performance between our presented approach and the total energy control system (TECS) in a \SI{18}{m/s} loiter flight.} 
		\label{fig_tecs_comparison}
	\end{figure}
		
	Fig. \ref{fig_horizontal_circle} and Fig. \ref{fig_circle_18} respectively show the trajectory and flight data in the loiter test with a flight radius of 50m and speed of \SI{18}{m/s}. As shown in Fig. \ref{fig_horizontal_circle} and Fig. \ref{fig_circle_18}\textcolor{red}{(b-c)}, after a while of stationary hovering, the tail-sitter first performs a coupled roll and pitch rotation to smoothly transition into the circular trajectory, and similarly transitions out of the circular trajectory with coupled roll and pitch rotations. Compared to traditional control methods \citep{verling2016full, lyu2017hierarchical} where a loiter trajectory is separated into straight-line transition followed by a bank turn in level flight, our maneuver is more elegant and time-saving due to less extra flight distance.  It is seen in Fig. \ref{fig_circle_18} that during the entire flight, the feedback of position, velocity and attitude tracks the reference closely. More specifically, Fig. \ref{fig_circle_18}\textcolor{red}{(e)} illustrates the position tracking error, which is less than \SI{0.26}{m} during the 45-second constant-speed loiter and slightly increases to  \SI{0.42}{m} and \SI{0.56}{m} in the two transition phases, respectively. Moreover, we conduct this test with different loiter speed $\Vert \mathbf{v} \Vert \in \{8, 12, 18\}$ m/s. The tracking error statics of each phase of the three tests are summarized in Fig. \ref{fig_circle_boxplot}. Banked transitions in the largest speed \SI{18}{m/s} have the largest worst-case tracking errors (i.e., \SI{0.31}{m} for the forward transition and \SI{0.33}{m} for the backward transition), while all loiter flights have similarly small errors less than \SI{0.21}{m}. The above experimental results demonstrate that the proposed trajectory generation and tracking control framework promises high-accuracy flights in real outdoor environments.

	In order to demonstrate the effectiveness and significance of wind speed compensation in the controller, we conduct a loiter fight in \SI{18}{m/s} with wind speed in the flatness transform enabled and disabled online. As shown in Fig. \ref{fig_wind speed_comparison}\textcolor{red}{(f)}, the wind speed is estimated during the entire flight test, but the control framework only compensates the wind speed after 89s, indicated by the shaded background. When the wind speed is not compensated, the reference pitch angle (and angle of attack) maintains at a constant value $25^\circ$ due to the constant loitering speed (see Fig. \ref{fig_wind speed_comparison}\textcolor{red}{(d, e)}). In contrast, the actual vehicle pitch angle climbs to about $50^\circ$ to increase the lift due to the smaller airspeed when following the wind, and drops to around $10^\circ$ to decrease the lift due to the larger airspeed when against the wind. Moreover, due to the uncompensated wind speed, the vehicle actually does not perform coordinated flight, causing a side force that is then compensated by the vehicle roll angles (see Fig. \ref{fig_wind speed_comparison}\textcolor{red}{(a)}). Furthermore, the uncompensated wind speed contributes to an extra disturbance as shown in (\ref{e_linear_error_sys}), which causes the control error of the measured sideslip angle (which is computed without considering the estimated wind velocity and should be equal to the reference sideslip angle) to fluctuate between $12.5^{\circ}$ and $-7^{\circ}$ (see Fig. \ref{fig_wind speed_comparison}\textcolor{red}{(b)}). On the other hand, when the estimated wind speed is used in the differential flatness transform for the calculation of the state-input trajectory and the subsequent trajectory tracking controller, the reference pitch angle is recalculated to fluctuate according to the wind speed, similarly the reference yaw angle is also adjusted to keep the sideslip angle at zero (i.e., ensuring the coordinated flight condition). As a result, the control errors in pitch, slideslip angle, and flight speed are significantly reduced.
	
	Finally, a comparison between our method and the total energy control system (TECS) is conducted on the loiter flight of \SI{18}{m/s}. As a mature technique for fixed-wing aircraft flight control, TECS also has been widely {used} in tail-sitter level flights. Due to the approximately linear aerodynamic force in low AoA, TECS employs a proportional and integral (PI) control scheme to regulate the airspeed and altitude by controlling the error of the total energy (i.e., the sum of potential and kinetic energy) to zero \citep{lambregts1983vertical}. We use the TECS implemented in the PX4 autopilot and tune its parameters to the best extent. Similar to the previous transition control comparison, both the proposed MPC and TECS utilize the same low-level controller for tracking the angular velocity command. The inner attitude loop, middle energy balance loop, and outer total energy loop of the TECS, which compute the commands for angular velocity, pitch angle, and thrust respectively, are tuned in sequence using the Ziegler-Nichols method. The resulting control performance achieved in the experiment is on par with those demonstrated in related works \citep{lyu2017hierarchical, gu2017development}. As shown in Fig. \ref{fig_tecs_comparison}, the vehicle altitude drops around \SI{1.5}{m} in average and \SI{2.5}{m} in maximum when using TECS for the loiter flight. In comparison, there is no obvious steady-state error for our approach and the maximum altitude error is less than \SI{0.25}{m}. The results are reasonable since TECS does not make use of any aerodynamic models of the vehicle, while our approach fully exploits these information.

	\subsection{Aerobatics}
	
	\begin{figure*}[h!] 
		\centering
		\includegraphics[width=1\textwidth]{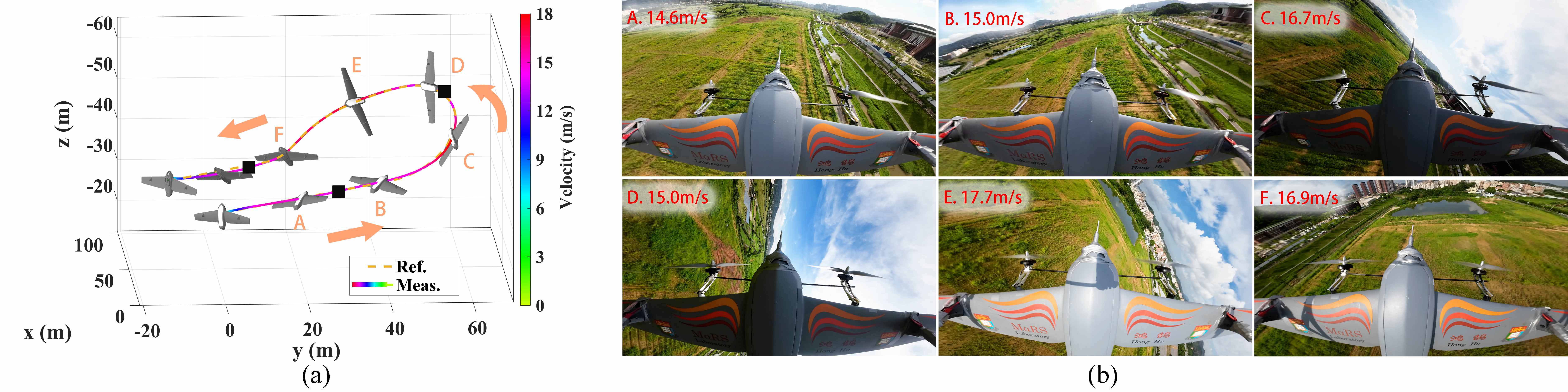} 
		\caption{Wingover: (a) illustration of the trajectory that is divided in two segments by three boundary points (black squares). Each segment is optimized by (\ref{e_traj_opt}), (b) images from the FPV camera. Labels A-F denote different flight phases of the vehicle.} 
		\label{fig_wingover}
		\vspace{5mm}
	\end{figure*}
	
	\begin{figure*}[h!] 
		\centering
		\includegraphics[width=1\linewidth]{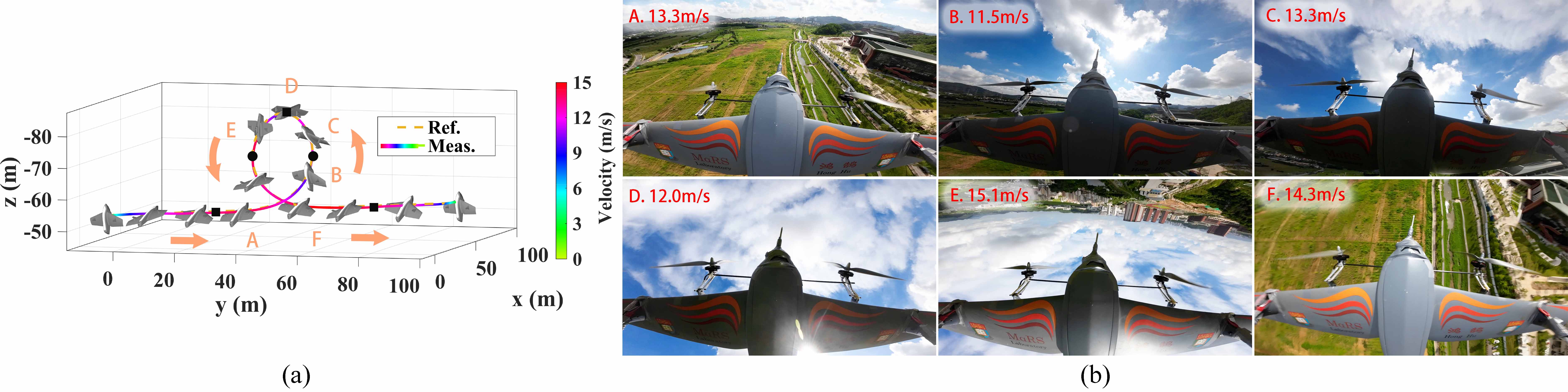} 
		\caption{Loop: (a) illustration of the trajectory, boundary points (black squares), and intermediate waypoints (black dots) that are constrained in (\ref{e_traj_opt}), (b) images from the FPV camera. Labels A-F denote different flight phases of the vehicle.} 
		\label{fig_loop}
		\vspace{5mm}
	\end{figure*}
	
	\begin{figure*}[h!] 
		\centering
		\includegraphics[width=1\linewidth]{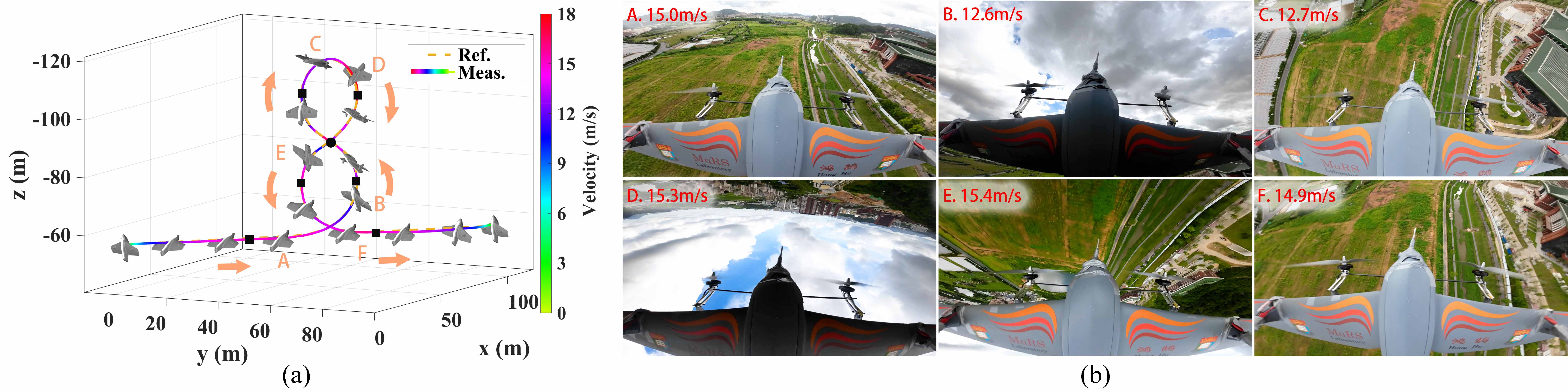} 
		\caption{Vertical Eight: (a) illustration of the trajectory, boundary points (black squares), and intermediate waypoints (black dots), (b) images from the FPV camera. Labels A-F denote different flight phases of the vehicle.} 
		\label{fig_vertical_eight}
		\vspace{5mm}
	\end{figure*}
	
	\begin{figure*}[h!] 
		\centering
		\includegraphics[width=1\linewidth]{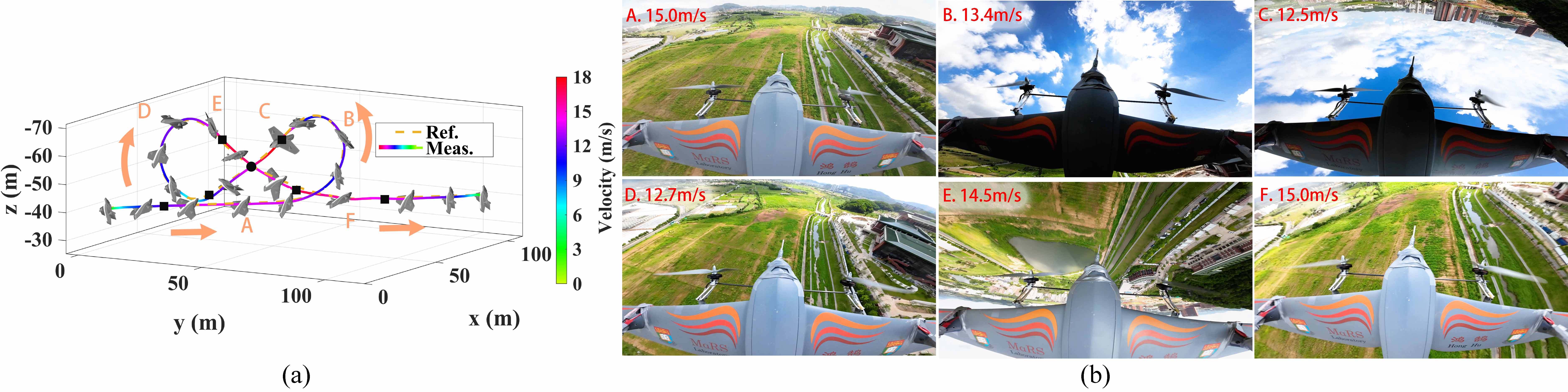} 
		\caption{Cuban Eight: (a) illustration of the trajectory, boundary points (black squares), and intermediate waypoints (black dots), (b) images from the FPV camera. Labels A-F denote different flight phases of the vehicle.} 
		\label{fig_cuban_eight}
		\vspace{5mm}
	\end{figure*}
	
	\begin{figure*}[h!] 
		\centering
		\includegraphics[width=0.8\linewidth]{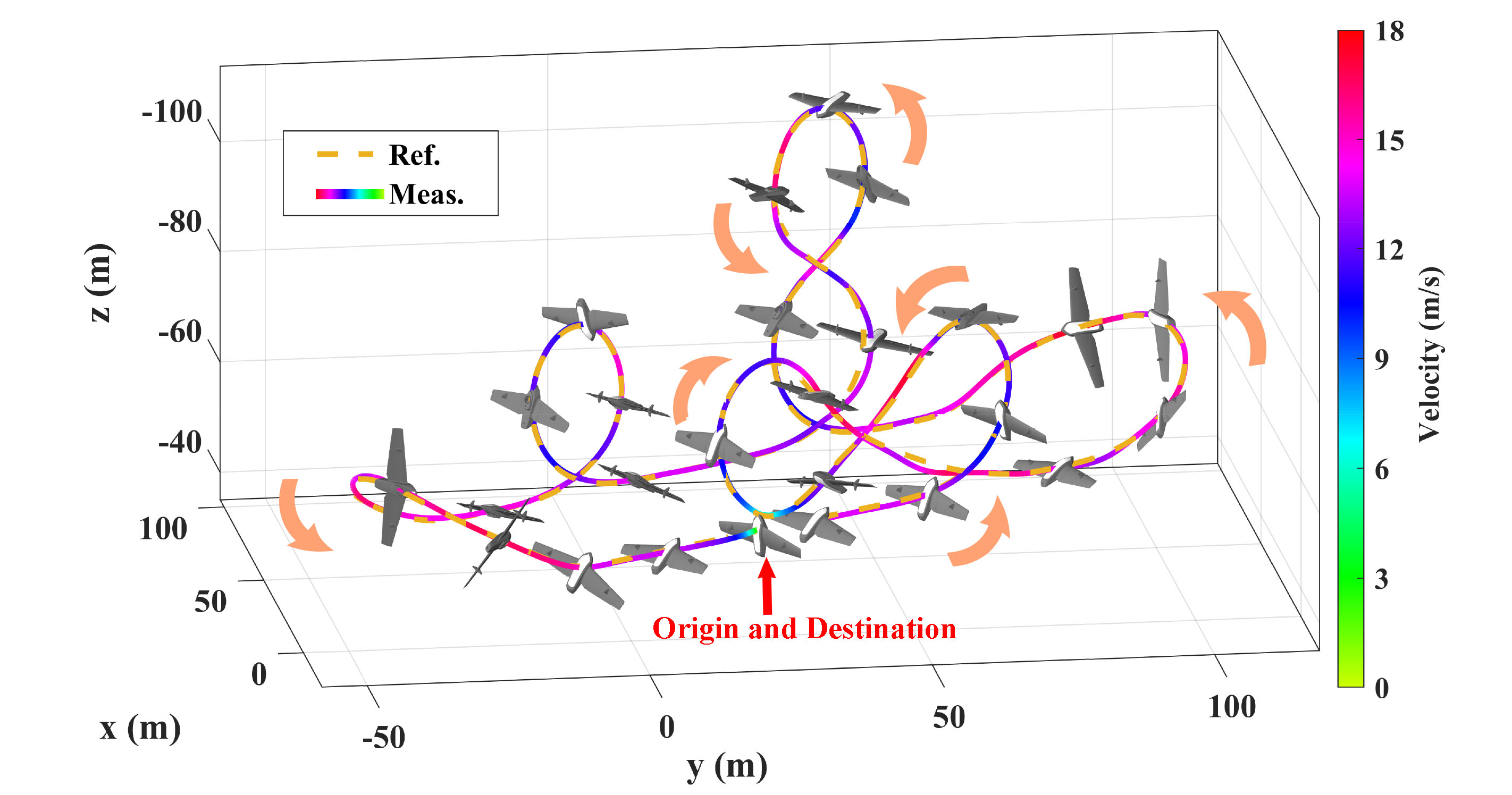} 
		\caption{Combo flight trajectory illustration.} 
		\label{fig_combo}
	\end{figure*}

	In this task, we push the tail-sitter to its physical limits to perform extremely aggressive aerobatics, which further demonstrates the effectiveness and robustness of our proposed methods. Our approach is the first to enable an autonomous tail-sitter to perform a series of aerobatic maneuvers with such agility in real outdoor environments. These maneuvers are highly challenging even for expert human pilots and are listed by increasing difficulty as follows:
	\begin{itemize}
		\item [1)] Wingover: the vehicle makes a $180^\circ$ turn in heading by executing a fast climb and turn, during which the wing swings over the top of the turn (i.e., the roll angle reaches $90^\circ$), as shown in Fig. \ref{fig_wingover}. 
		
		\item [2)] Loop: the vehicle enters a vertical circle and makes a $360^\circ$ flip in pitch angle, as shown in Fig. \ref{fig_loop}. 
		
		\item [3)] Vertical Eight: the vehicle performs a vertical figure-``8" trajectory with pitch angle pulled up and down over $180^\circ$, as shown in Fig. \ref{fig_vertical_eight}. 
		
		\item [4)] Cuban Eight: similarly to the Vertical Eight, the vehicle performs a ``$\infty$"-shape trajectory with pitch angle pulled up and down over $180^\circ$, as shown in Fig. \ref{fig_cuban_eight}. 
		
		\item [5)] Combo: the vehicle starts with Cuban Eight, followed by  Wingover, Vertical Eight, Loop, and ends with another Wingover to fly back to the origin, as shown in Fig. \ref{fig_combo}. The entire maneuver is executed consecutively without any breaks.		
	\end{itemize}
 
 	 As shown in Fig. \ref{fig_wingover}-\ref{fig_cuban_eight}, to specify the shape of the trajectory and the vehicle pose at certain position on the trajectory, we separate the entire trajectory by multiple pieces by boundary points (i.e., the black squares). At the boundary points, the full vehicle states (i.e., position, velocity, and attitude) are specified and transformed to trajectory boundary conditions  $\mathbf p^{(0:3)}$. With these boundary conditions, trajectories within two consecutive boundary points are optimized by our trajectory optimization framework (\ref{e_traj_opt}). To further specify the shape of each trajectory segment, we specify some waypoints (i.e., the black dots) that the trajectory must pass through, which is naturally supported by the optimization framework in (\ref{e_traj_opt}). 
   All the trajectories begins with a forward transition (i.e., the origin to the first black square) and ends with a backward transition to hovering (i.e., the last black square to the destination).  Taking the Wingover in Fig. \ref{fig_wingover}\textcolor{red}{(a)} as example, the trajectory consists of four segments: forward transition, climbing up with $90^\circ$ rotation in both roll and yaw, diving down with reverse heading, and backward transition. The design of the Loop trajectory in Fig. \ref{fig_loop}\textcolor{red}{(a)} is similar, except that the top boundary point is designed to drive the vehicle upside down (i.e., $-180^\circ$ in  pitch angle) and further inserting two waypoints to guarantee the shape of Loop. The Vertical Eight and Cuban Eight trajectories are generated by connecting two Loop trajectories. The boundary points of the connecting trajectories are obtained from the original Loop trajectories and a waypoint in the middle is used to serve the intersection point of the two connecting trajectories.   

    For all aerobatics above, we use the same parameters of the planner and controller as in the indoor $SE(3)$ flights and outdoor typical flights (i.e., Section \ref{sec_se3_flight} and \ref{sec_typical_maneuvers}, respectively), except further decreasing the MPC position penalty (i.e., the first three diagonal elements of $\mathbf{Q}_k$) to $[900, 900, 900]$ to  increase the system robustness in consideration of the highly aggressive maneuvers being executed.
    	
    \begin{figure*}[h!] 
    	\centering
    	\includegraphics[width=1\linewidth]{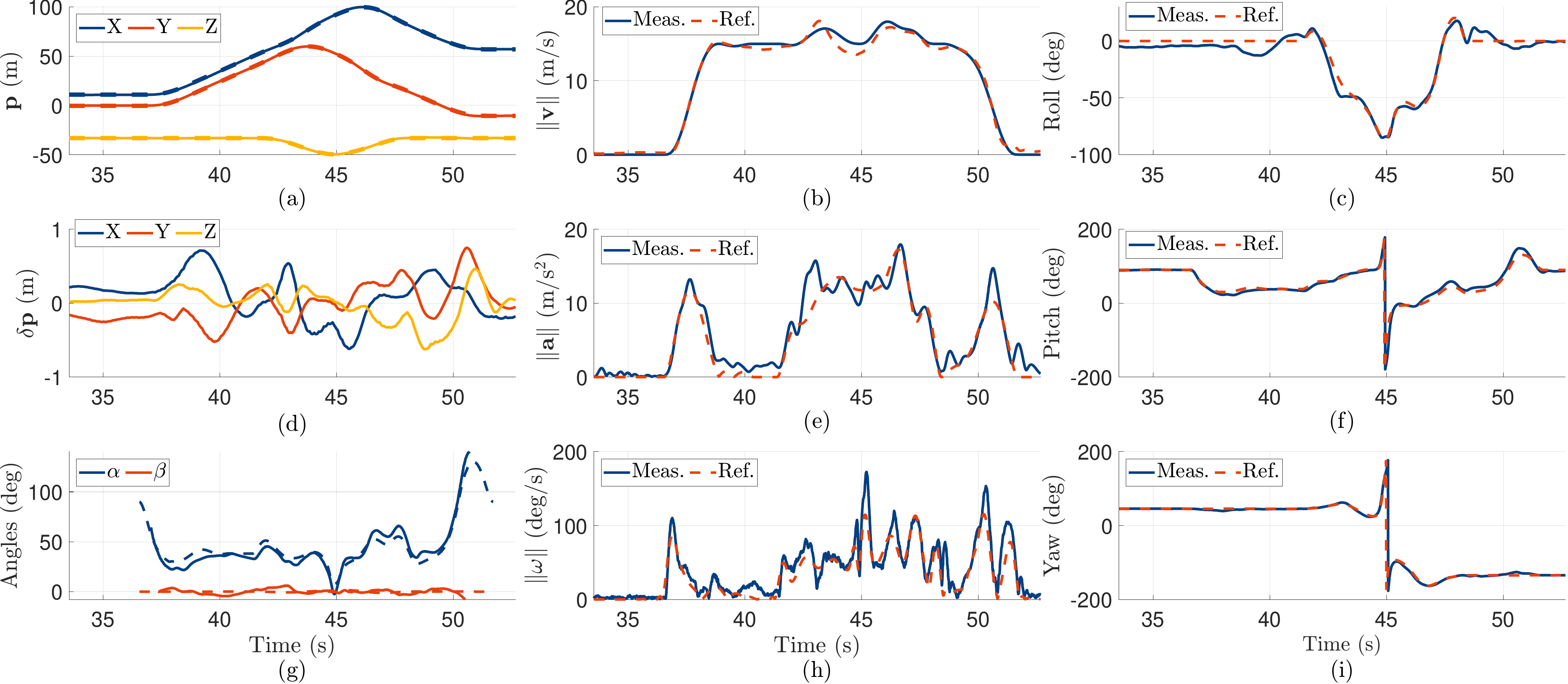} 
    	\caption{Flight results of the aerobatic maneuver Wingover: (a) position, (b) flight speed, (c),(f) and (i) attitude in Euler angles, (d) position tracking errors, (e) acceleration. (g) angle of attack and sideslip angle, (h) angular velocity. For the angle of attack and sideslip angle, their measurements are displayed only when the airspeed exceeds \SI{2}{m/s} due to the unstable airspeed measurements at low speeds. In all subplots, the solid and dashed lines respectively denote the measurement and reference. Note that the ZXY Euler angle representation incur singularity when roll angle reaches $90^\circ$ at \SI{45}{s}.} 
    	\label{fig_wingover_plot}
    \end{figure*}
    
    \begin{figure*}[h!] 
    	\centering
    	\includegraphics[width=1\linewidth]{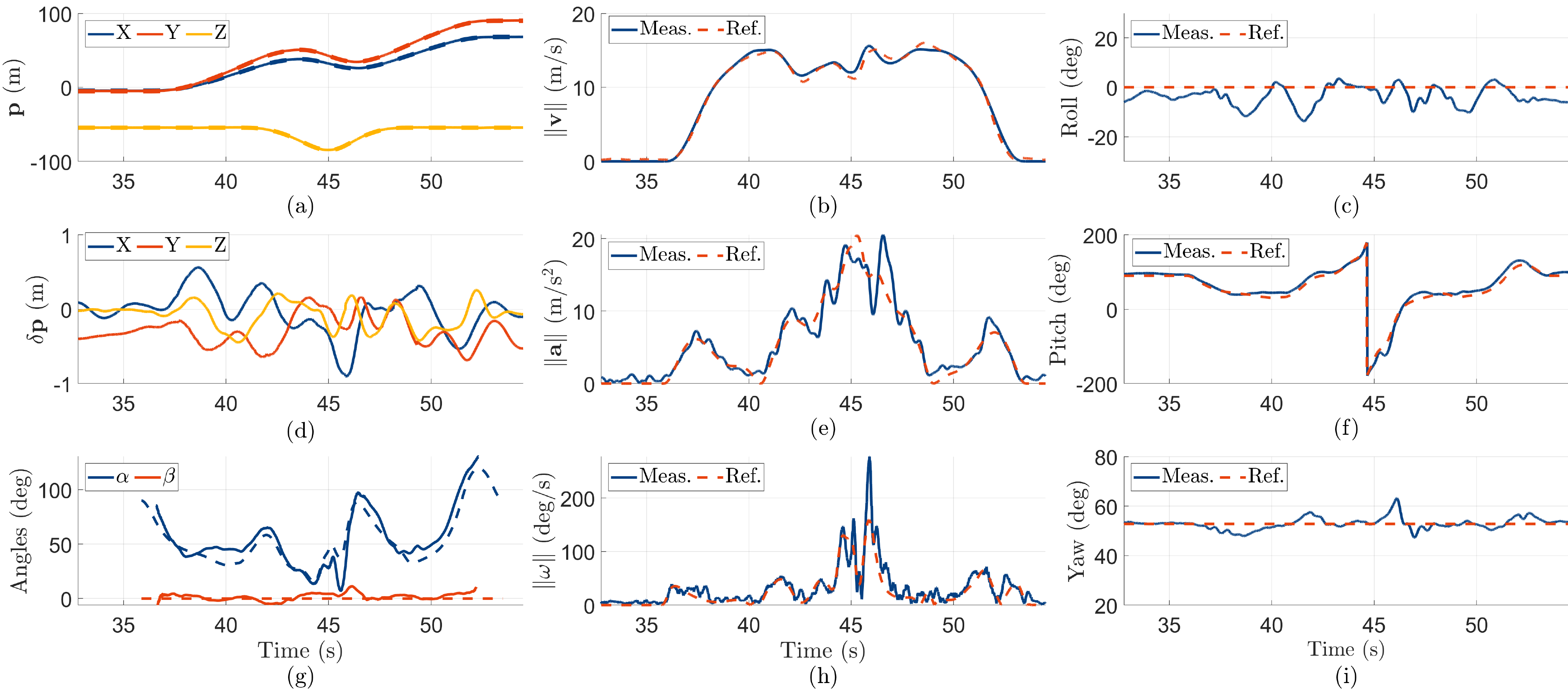} 
    	\caption{Flight results of the aerobatic maneuver Loop: (a) position, (b) flight speed, (c),(f) and (i) attitude Euler angles, (d) position tracking errors, (e) acceleration. (g) angle of attack and sideslip angle, (h) angular velocity. For the angle of attack and sideslip angle, their measurements are displayed only when the airspeed exceeds \SI{2}{m/s} due to the unstable airspeed measurements at low speeds. In all subplot, the solid and dashed lines respectively denote the measurement and reference.} 
    	\label{fig_loop_plot}
    \end{figure*}
    
    \begin{figure*}[h!] 
    	\centering
    	\includegraphics[width=1\linewidth]{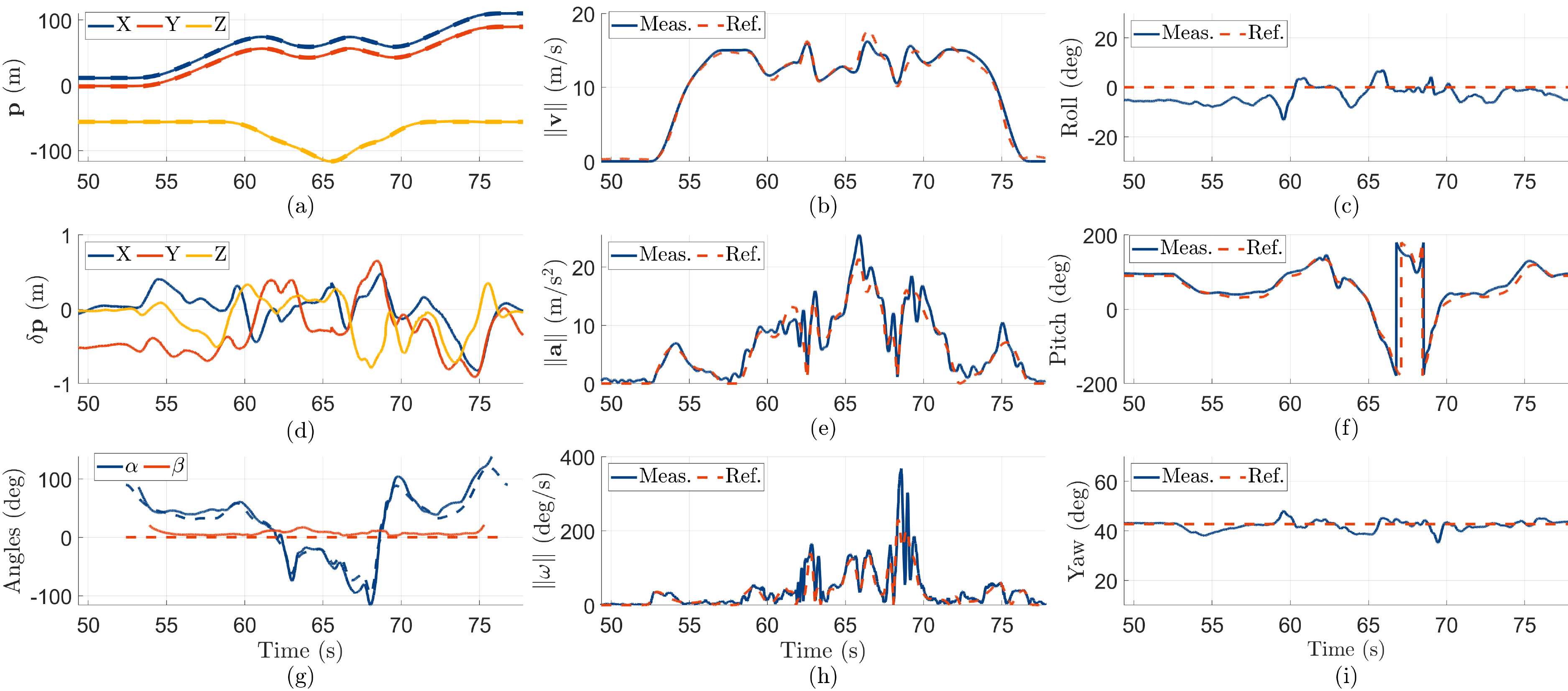} 
    	\caption{Flight results of the aerobatic maneuver Vertical Eight: (a) position, (b) flight speed, (c),(f) and (i) attitude Euler angles, (d) position tracking errors, (e) acceleration. (g) angle of attack and sideslip angle, (h) angular velocity. For the angle of attack and sideslip angle, their measurements are displayed only when the airspeed exceeds \SI{2}{m/s} due to the unstable airspeed measurements at low speeds. In all subplot, the solid and dashed lines respectively denote the measurement and reference.} 
    	\label{fig_vertical_eight_plot}
    \end{figure*}
    
    \begin{figure*}[h!] 
    	\centering
    	\includegraphics[width=1\linewidth]{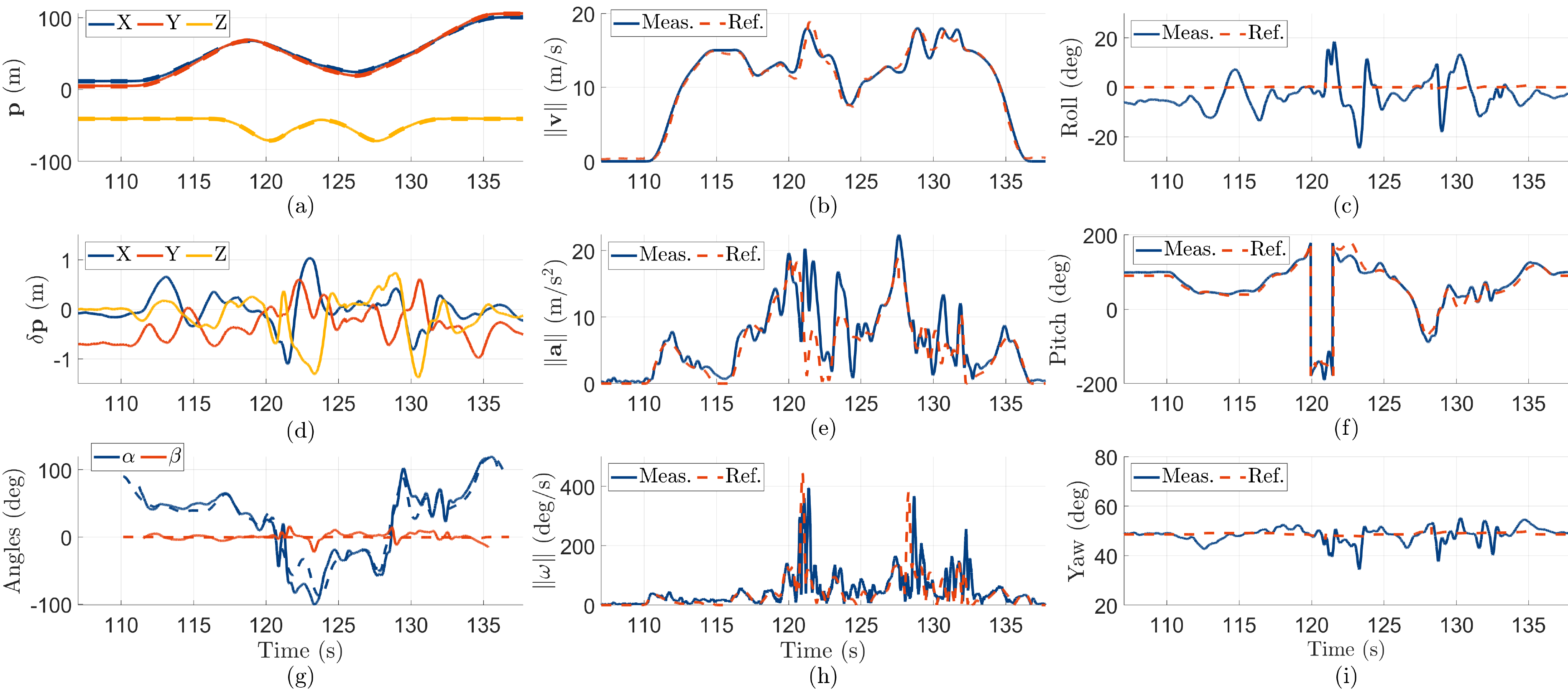} 
    	\caption{Flight results of the aerobatic maneuver Cuban Eight: (a) position, (b) flight speed, (c),(f) and (i) attitude Euler angles, (d) position tracking errors, (e) acceleration. (g) angle of attack and sideslip angle, (h) angular velocity. For the angle of attack and sideslip angle, their measurements are displayed only when the airspeed exceeds \SI{2}{m/s} due to the unstable airspeed measurements at low speeds. In all subplot, the solid and dashed lines respectively denote the measurement and reference.} 
    	\label{fig_cuban_eight_plot}
    \end{figure*}
    
    \begin{figure*}[h!] 
    	\centering
    	\includegraphics[width=1\linewidth]{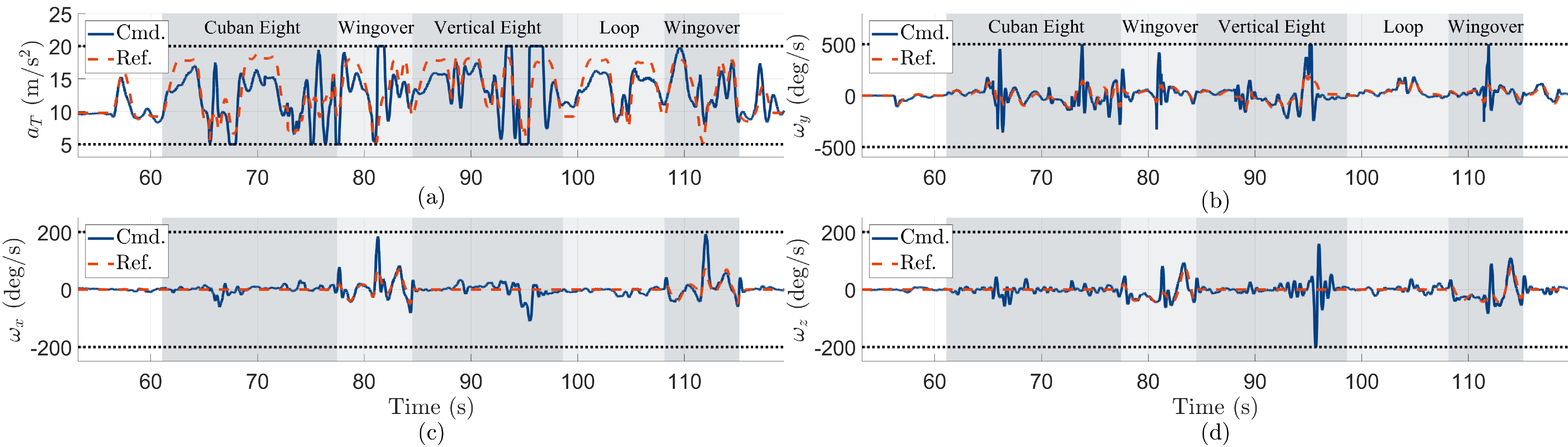} 
    	\caption{Control efforts of the aerobatic maneuver Combo. The shaded areas indicate the comprising maneuvers of Cuban Eight, Wingover, Vertical Eight, Loop and Wingover. The dotted lines denote input constraints in the MPC optimization (\ref{e_error_state_mpc}).} 
    	\label{fig_combo_ctrl}
    \end{figure*}
    
    \begin{figure*}[h!] 
    	\centering
    	\includegraphics[width=1\linewidth]{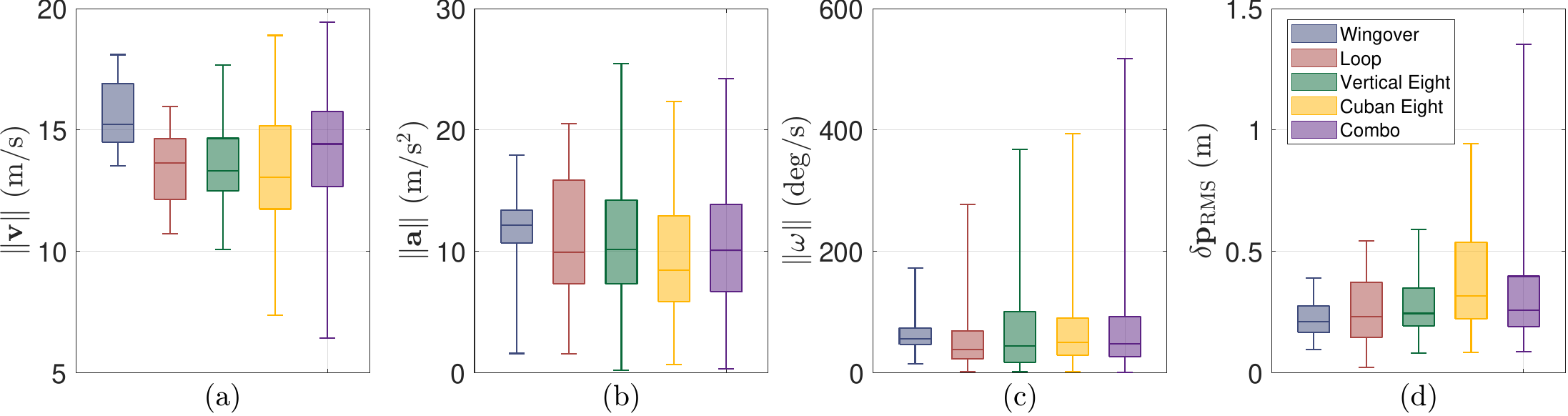} 
    	\caption{The aggressiveness and control performance of five aerobatic maneuvers in the Combo flight. (a-d) show the norm of velocity, acceleration, angular velocity and position control error, respectively.} 
    	\label{fig_aerobatics_box}
    \end{figure*}
    
    \begin{figure}[h!] 
    	\centering
    	\includegraphics[width=1.0\linewidth]{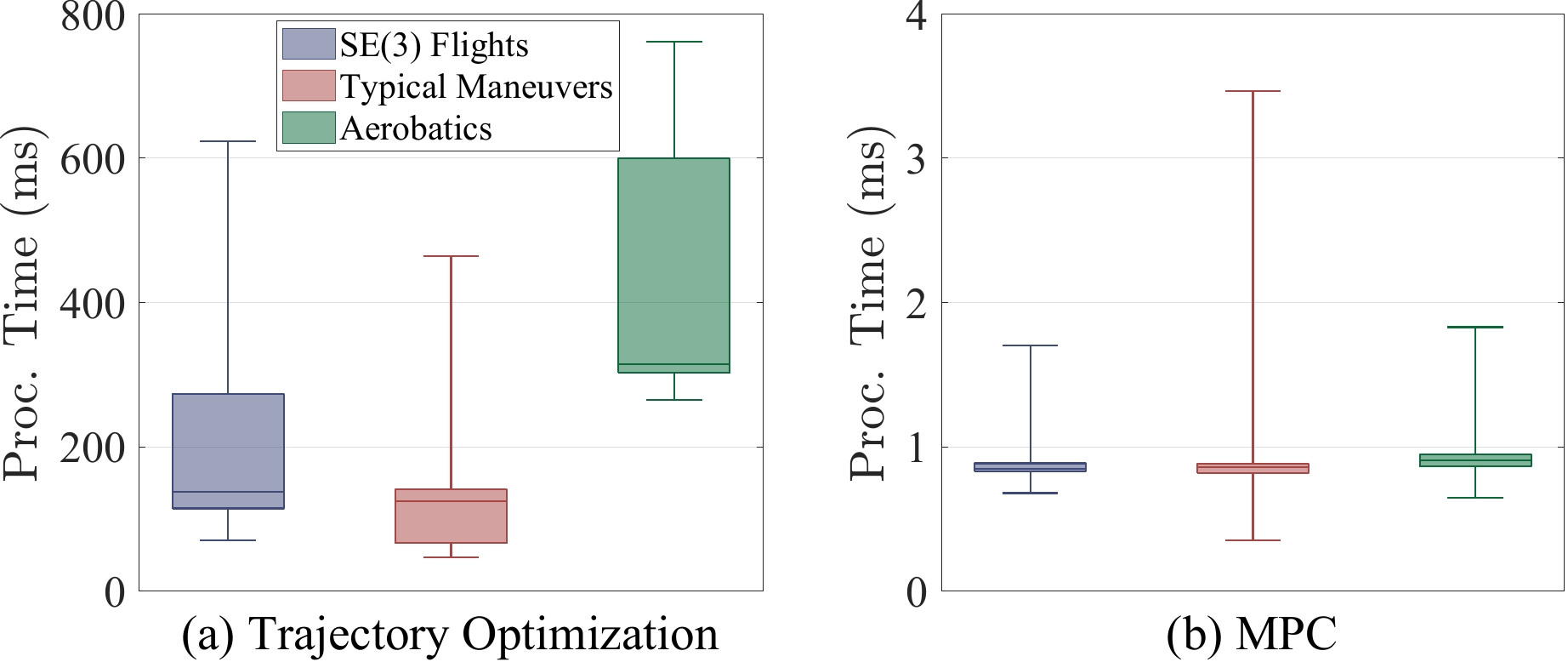} 
    	\caption{Time consumption of (a) solving the proposed trajectory optimization to generating a trajectory segment,  (b) solving the proposed MPC at each control step.}
    	\label{fig_mpc_proc_time}
    \end{figure}
	
	Fig. \ref{fig_wingover_plot} details the flight data of the Wingover. The vehicle first transits from hovering to level flight with speed of \SI{15}{m/s}, then performs the Wingover maneuver in 14-\SI{18}{m/s} and finally ends with a backward transition to hovering. The vehicle climbs \SI{16.5}{m} at the top and achieves the specified $90^\circ$ roll and yaw angle at \SI{45}{s}. Note that the ZXY Euler angle incurs singularity in the visualization, but our global on-manifold MPC has no such singularity as shown in the FPV image in Fig. \ref{fig_wingover}\textcolor{red}{(b)D}. Throughout the flight, the vehicle tracks all of the state trajectories closely: the position error is less than \SI{0.75}{m} in all time and the sideslip angle is well stabilized around zero. This tracking accuracy is not trivial for outdoor UAV aerobatics with such large large span of angle of attack (up to $130^\circ$), acceleration (up to \SI{18}{m/s^2}), and angular velocity (up to \SI{175}{deg/s}). 
		
	The flight results of the Loop is shown in Fig. \ref{fig_loop_plot}. The vehicle transits to \SI{15}{m/s} and successfully finishes a Loop with radius of around \SI{15}{m} in \SI{10}{s}. It is seen that the pitch angle rises to $180^\circ$ at \SI{44.6}{s}, when the vehicle is totally upside-down at the top of the Loop as designed,  which is also shown in Fig. \ref{fig_loop}\textcolor{red}{(b)D}. The vehicle also tracks  all of the state trajectories closely in the coordinated flight condition (i.e., the sideslip angle is shown around zero). The position error also remains below \SI{1}{m} in all directions, even though the maximum acceleration and angular velocity increase to \SI{25.5}{m/s^2} and  \SI{370}{deg/s}, respectively.
	
	As shown in Fig. \ref{fig_vertical_eight_plot}, the vehicle finishes a more aggressive aerobatic maneuver of Vertical Eight also in high tracking accuracy. From \SI{58}{s} in level flight, the vehicle begins to pull up the pitch angle to $145^\circ$ and quickly lowers it to zero at \SI{65}{s}, meanwhile the vehicle simultaneously gains \SI{60}{m} altitude by following a ``S"-shape trajectory (i.e., position A-B-C in Fig. \ref{fig_vertical_eight}\textcolor{red}{(a)}). After that, the vehicle flies another ``S"-shape trajectory to decrease to the original altitude when the pitch angle continues to decrease to $-260^\circ$ (i.e, nearly free falling as shown in Fig. \ref{fig_vertical_eight}\textcolor{red}{(b)E}) and quickly increases to $35^\circ$ to perform a \SI{15}{m/s} level flight again. It is seen that the angle of attack ranges from $-115^\circ$ to $120^\circ$, the largest span among all the demonstrated aerobatics. Moreover, the acceleration and angular velocity respectively peak at \SI{25.5}{m/s^2} and \SI{400}{deg/s}. Despite such large span of angle of attack and high acceleration and angular velocity, the overall position error still remains less than \SI{1}{m}.
	
	Similarly, the vehicle executes the Cuban Eight maneuver in high tracking performance despite the extremely high aggressiveness. The vehicle tracks the $``\infty"$-shape with a width of \SI{65}{m}, a height of \SI{30}{m}, and a time duration of \SI{15}{s}. The pitch angle increases from $36^\circ$ in level flight to $225^\circ$ at position C in Fig \ref{fig_cuban_eight}\textcolor{red}{(a)}, then it decreases to $-88^\circ$ at position E and recovers to $36^\circ$ at position F for level flight. The resulting span of angle of attack is about $220^\circ$. The acceleration and angular velocity peaks at \SI{22}{m/s^2} and \SI{400}{deg/s}, respectively. The maximum position error slightly rises to \SI{1.38}{m} due to the large control actuation but the overall tracking performance for the other state trajectories are still as good as other aerobatic maneuvers.
	 
	Moreover, we demonstrate a Combo trajectory by connecting the above aerobatic maneuvers in sequence, as shown in Fig. \ref{fig_combo}. The vehicle performs the Combo maneuver which requires extremely large control actuation over the entire \SI{62}{s} flight. In Fig. \ref{fig_combo_ctrl}, it is seen that the thrust acceleration and angular velocity commands computed by the MPC frequently touch their limits, but the controller still manages to stabilize the vehicle under such control saturation.
	
	The trajectory aggressiveness and tracking accuracy of the five aerobatic maneuvers in the Combo flight are statistically analyzed in Fig. \ref{fig_aerobatics_box}. The maximum velocity, acceleration and angular velocity reach \SI{19.4}{m/s}, \SI{25.5}{m/s} and \SI{520}{deg/s}, respectively.  Still, the proposed global controller shows a remarkable tracking performance that the average position tracking error is \SI{0.33}{m} and largest position error is only \SI{1.35}{m}. Readers are highly recommended to watch the accompanying videos for better visualization of the experiments.

    \subsection{Time consumption}
    The statistical time consumption of the trajectory planner and MPC in all of the above flight tests including $SE(3)$ flights, typical maneuvers, and aerobatic maneuvers is summarized in Fig. \ref{fig_mpc_proc_time}. For the trajectory generation in (\ref{e_traj_opt}) which runs offline, the computation time is about 45-\SI{750}{ms} to generate one trajectory segment with a length of about 2.5-\SI{50}{m}. The average and maximum time consumption to solve the MPC problem in (\ref{e_error_state_mpc}) in total is \SI{0.84}{ms} and \SI{3.46}{ms}, respectively, showing a high computational efficiency ensuring online implementation at \SI{100}{Hz}.

	\section{Conclusion}
	\label{sec_diss_and_con}
	In this section, we discuss the limitation and extension of the proposed framework, and then draw the conclusion.
	
	\subsection{Limitation}
	Our proposed framework is a model-based approach. Higher tracking accuracy requires a more precise dynamic model, especially the aerodynamic model. However, identifying a high-fidelity aerodynamic model generally requires high-cost and time-consuming wind tunnel tests. The cost and time escalate for tail-sitter UAVs where the envelope of angle of attack is large. In this paper, we leveraged the wind tunnel test data in \citep{lyu2018simulation}. For general tail-sitter UAVs, such aerodynamic model could be identified from onboard sensor data collected in real flights, which would be a promising future research to pursue. 
    
    Another limitation lies in the robustness and computation efficiency of the trajectory planner. In this paper, we adopted the MINCO trajectory optimization framework \citep{wang2022geometrically}, which parameterizes the trajectory by a multi-stage polynomial and penalizes the constraints in the objective function as soft constraints. Softly penalizing the constraints in objective functions could reduce the optimization time by eliminating the hard constraints. However, due to the extremely nonlinear objective function, the solver could easily converge to local minimum violating the constraints. This phenomenon occasionally occurred in the planning of the outdoor aerobatic trajectories when the waypoints locations are poorly specified. Moreover, the optimization time is still quite long, 40-\SI{750}{ms}, preventing it from real-time implementation on current tail-sitter onboard computing devices.

	\subsection{Extension}
	Firstly, the proposed trajectory optimization could potentially be solved more efficiently by leveraging state-of-the-art nonlinear optimization techniques (e.g., \cite{schulman2014motion, gill2005snopt}), the availability of higher-performance onboard computing devices, and the parallelization of the optimization based on Graphic Processing Units (GPUs). With an efficient solution, the proposed trajectory generation could serve as a reliable back-end planner for on-line trajectory planning. Equipped with onboard sensors such as cameras and lidars, and the corresponding front-end corridor generation techniques (e.g., \cite{liu2017planning, gao2019flying}), the tail-sitter could perform autonomous obstacle avoidance in cluttered environments.
	
	Secondly, the tracking accuracy can be further improved by augmenting a low-level controller to the thrust acceleration $a_T$. In the present implementation, we directly mapped the thrust acceleration command $a_T$ to the collective throttle of the four motors. However, the actual propeller thrust is also affected by various other factors, such as the propeller inflow \citep{brandt2011propeller, gill2017propeller} and motor internal dynamics. These factors have caused significant errors between the actual and commanded thrust acceleration as shown in our experiment results. This issue could be mitigated by tracking the thrust acceleration command $a_T$ with a low-level controller based on accelerometer measurements. 
	
	Thirdly, other than the model predictive controller in the present implementation, the flatness function provides a possibility to design a more light-weight cascaded PID controller that runs on low-cost micro processors. The cascaded control architecture could be similar to that of a multicopter: an outer-loop position controller first computes the desired acceleration, then our differential flatness function maps the desired acceleration to the desired attitude and thrust, finally the attitude is tracked by an inner-loop attitude controller. Such a cascaded control structure is also used in existing works \citep{ritz2017global, cheng2022transition}, but based on an over-simplified aerodynamic model. 
	
	Finally, the proposed framework can be extended to other configurations of tail-sitter UAVs, such as the single-propeller configuration \citep{frank2007hover, wang2017design, de2018design} and the shoulder-mounted twin-engine configuration \citep{bapst2015design, ritz2017global, sun2018design}. All of he trajectory generation, flatness transform and global control for the high-level system can be directly applied to the other configurations, while the low-level  controller could be re-designed according to the specific vehicle dynamic parameters and actuator performances. 

	\subsection{Conclusion}
	In this paper, we proposed a trajectory generation and global tracking controller for aggressive agile tail-sitter flights. The foundation of the framework is the differential flatness property that is proved in coordinated flights. The singularity conditions occurred in the flatness function were fully investigated and resolved in the framework. Based on these theoretical results, we developed a trajectory optimization framework for trajectory generation and a model predictive controller for trajectory tracking. The entire approach is tested on a quadrotor tail-sitter prototype in extensive real-world flights. Notably, we demonstrated agile $SE(3)$ flights in indoor environments and aerobatic maneuvers in windy outdoor environments, which were rarely shown in any prior literature works. Extensive flight tests on typical maneuvers of transition, level flight and loiter, have also shown a superior tracking accuracy compared to existing methods. 
	
	\section*{Acknowledgment}
	This work was supported in part by Hong Kong RGC ECS under grant 27202219 and in part by DJI donation. The authors appreciate Dr. Wei Xu and Dr. Haowei Gu for the help on the initial set up of the prototype, and Huirong Cheng for discussion. The author especially thank Dr. Ximin Lyu of Sun Yat-sen University for providing the space and equipment for conducting the indoor experiments. 
	
	\bibliographystyle{SageH}
	\bibliography{reference}
	
	\begin{appendices}
		\section{Proof of theorem \ref{theorem_rank_pfpu}}
		\label{app_rank_pfpu}
		Given the vehicle dynamics in (\ref{e_vehicle_dyn}), the rank of system dynamics derivative w.r.t input can be given by eliminating unrelated items:
        \begin{equation}
		\begin{aligned}
			\text{\rm rank} \left(\frac{\partial \mathbf f(\mathbf x_{\rm full}, \mathbf u_{\rm full})}{\partial \mathbf u_{\rm full}} \right) &= \text{\rm rank} \left( \frac{\partial \left(\dot{\mathbf v}, \dot{\boldsymbol{\omega}}\right)}{\partial (a_T, \boldsymbol{\tau})}\right)\\  &= \text{\rm rank} \left(\begin{bmatrix}
				\frac{\partial \dot{\mathbf v}}{\partial a_T} & \frac{\partial \dot{\mathbf v}}{\partial \boldsymbol{\tau}} \\
				\frac{\partial \dot{\boldsymbol{\omega}}}{\partial a_T} & \frac{\partial \dot{\boldsymbol{\omega}}}{\partial \boldsymbol{\tau}}
			\end{bmatrix} \right)
		\label{e_rank_pfpu_}
		\end{aligned}
        \end{equation}
		where the following elements can be computed directly from the system dynamics in (\ref{e_vehicle_dyn}):
		\begin{align}
			\frac{\partial \dot{\mathbf v}}{\partial a_T} = \mathbf x_b, \quad \frac{\partial \dot{\mathbf v}}{\partial \boldsymbol{\tau}} = \mathbf 0, \quad \frac{\partial \dot{\boldsymbol{\omega}}}{\partial a_T} = \mathbf 0 
			\label{e_pfpu}
		\end{align}	
		and $\frac{\partial \dot{\boldsymbol{\omega}}}{\partial \boldsymbol{\tau}}$ has coupling effect due to the coordinated flight condition that the vehicle has no lateral airspeed:
		\begin{equation}
			\mathbf v_{a_y}^{\mathcal{B}} = \mathbf{e}_2^T\mathbf{R}^T \mathbf{v}_a \equiv 0
		\end{equation}
		which leads to the derivative on the both sides:
		\begin{subequations}
			\begin{align}
				&\qquad - \mathbf{e}_2^T \lfloor\boldsymbol{\omega}\rfloor \mathbf{R}^T \mathbf{v}_a+ \mathbf{e}_2^T \mathbf{R}^T \dot{\mathbf{v}}_a = 0 \\
				&\Rightarrow \quad \mathbf{y}^T_b \dot{\mathbf{v}}_a = \mathbf{v}^T_a \mathbf{R} \lfloor\mathbf{e}_2 \rfloor \boldsymbol{\omega}  \\
				&\Rightarrow \quad \mathbf{y}^T_b \dot{\mathbf{v}}_a = \mathbf v_{a_x}^{\mathcal B} \boldsymbol{\omega}_z - \mathbf v_{a_z}^{\mathcal B} \boldsymbol{\omega}_x
			\end{align}
		\end{subequations}
		It is seen that the body angular velocity elements $\boldsymbol{\omega}_x$ and $\boldsymbol{\omega}_z$ are coupled. Without loss of generality, we consider $\boldsymbol{\omega}_z$ as a function of $\boldsymbol{\omega}_x$. Then we have
		\begin{equation}
			\frac{\partial \dot{\boldsymbol{\omega}}}{\partial \boldsymbol{\tau}} = \begin{bmatrix}
				\frac{\partial \dot{\boldsymbol{\omega}}_x}{\partial \boldsymbol{\tau}} \\
				\frac{\partial \dot{\boldsymbol{\omega}}_y}{\partial \boldsymbol{\tau}} \\
				\frac{\partial \dot{\boldsymbol{\omega}}_z}{\partial \boldsymbol{\tau}} 
			\end{bmatrix} = \begin{bmatrix}
			\frac{\partial \dot{\boldsymbol{\omega}}_x}{\partial \boldsymbol{\tau}} \\
			\frac{\partial \dot{\boldsymbol{\omega}}_y}{\partial \boldsymbol{\tau}} \\
			\frac{\partial \dot{\boldsymbol{\omega}}_z}{\partial \dot{\boldsymbol{\omega}}_x} \frac{\partial \dot{\boldsymbol{\omega}}_x}{\partial \boldsymbol{\tau}}  
			\end{bmatrix}  = \underbrace{\begin{bmatrix}
			1 & 0 & 0 \\
                0 & 1 & 0 \\
                \frac{\partial \dot{\boldsymbol{\omega}}_z}{\partial \dot{\boldsymbol{\omega}}_x} & 0 & 0 
		\end{bmatrix}}_{\mathbf A} \mathbf J^{-1}
                \label{e_pdwptau}
		\end{equation}
	
		With (\ref{e_rank_pfpu_}), (\ref{e_pfpu}) and (\ref{e_pdwptau}), the derivative is finally computed and transformed based on the fact of non-zero vector $\mathbf x_b$ and the full-rank inertia matrix:
        \begin{equation}
		\begin{aligned}
			\begin{bmatrix}
				\frac{\partial \dot{\mathbf v}}{\partial a_T} & \frac{\partial \dot{\mathbf v}}{\partial \boldsymbol{\tau}} \\
				\frac{\partial \dot{\boldsymbol{\omega}}}{\partial a_T} & \frac{\partial \dot{\boldsymbol{\omega}}}{\partial \boldsymbol{\tau}}
			\end{bmatrix} &= \begin{bmatrix}
			    \mathbf x_b & \mathbf 0 \\
                    \mathbf 0 & \mathbf A
			\end{bmatrix}  \begin{bmatrix}
			    1 & \mathbf 0 \\
                    \mathbf 0 & \mathbf J^{-1}
			\end{bmatrix} \sim \begin{bmatrix}
			    \mathbf x_b & \mathbf 0 \\
                    \mathbf 0 & \mathbf A
			\end{bmatrix}
		\end{aligned} 
        \end{equation}
	Therefore, we have its rank
        \begin{equation}
             \text{\rm rank} \left(\frac{\partial \mathbf f(\mathbf x_{\rm full}, \mathbf u_{\rm full})}{\partial \mathbf u_{\rm full}}\right) = \text{\rm rank} \left( \begin{bmatrix}
			    \mathbf x_b & \mathbf 0 \\
                    \mathbf 0 & \mathbf A
			\end{bmatrix} \right) = 3
        \end{equation}
	
		\section{Proof of theorem \ref{theorem_pfa_pvb}}
		\label{app_pfa_pvb}
		Reminding the aerodynamic force in (\ref{e_aero_force}) and the coordinated flight condition that there is no lateral airspeed in (\ref{e_vb_y_zero}) (i.e., $\mathbf{e}_2^T\mathbf{v}_a^\mathcal{B} = 0$) , we have
		\begin{align}
			\frac{\partial \mathbf{f}_a}{\partial \mathbf{v}_a^\mathcal{B}} &= \frac{\rho S}{2} \left(\mathbf{c} \frac{\partial V^2}{\partial \mathbf{v}_a^\mathcal{B}} + V^2 \frac{\partial \mathbf{c}}{\partial \alpha} \frac{\partial \alpha}{\partial \mathbf{v}_a^\mathcal{B}} + V^2 \frac{\partial \mathbf{c}}{\partial \beta} \frac{\partial \beta}{\partial \mathbf{v}_a^\mathcal{B}} \right)
			\label{e_pfa_pvb_temp}
		\end{align}
		where
		\begin{subequations}
			\begin{align}
				\frac{\partial V^2}{\partial \mathbf{v}_a^\mathcal{B}} &= \frac{\partial \Vert \mathbf{v}_a^\mathcal{B} \Vert^2}{\partial \mathbf{v}_a^\mathcal{B}} = 2\mathbf{v}_a^{\mathcal{B}^T} \\
				\frac{\partial \alpha}{\partial \mathbf{v}_a^\mathcal{B}} &= \frac{\partial }{\partial \mathbf{v}_a^\mathcal{B}} \tan^{-1}\frac{\mathbf{e}_3^T\mathbf{v}_a^\mathcal{B}}{\mathbf{e}_1^T\mathbf{v}_a^\mathcal{B}} \nonumber\\
				&= \frac{1}{1 + \left(\frac{\mathbf{e}_3^T\mathbf{v}_a^\mathcal{B}}{\mathbf{e}_1^T\mathbf{v}_a^\mathcal{B}}\right)^2} \frac{\mathbf{e}_1^T\mathbf{v}_a^\mathcal{B}\mathbf{e}_3^T - \mathbf{e}_3^T\mathbf{v}_a^\mathcal{B}\mathbf{e}_1^T}{\left(\mathbf{e}_1^T\mathbf{v}_a^\mathcal{B}\right)^2} \nonumber \\
				&= \frac{\begin{bmatrix}
						-\mathbf{v}_{a_z}^\mathcal{B} & 0 & \mathbf{v}_{a_x}^\mathcal{B}
				\end{bmatrix}}{\mathbf{v}_{a_x}^{\mathcal{B}^2} + \mathbf{v}_{a_z}^{\mathcal{B}^2}} \nonumber\\
				&= \frac{\mathbf{v}_a^{\mathcal{B}^T} \lfloor \mathbf{e}_2 \rfloor }{V^2} \\
				\frac{\partial \beta}{\partial \mathbf{v}^{\mathcal{B}}_a} &= \frac{\partial }{\partial \mathbf{v}^{\mathcal{B}}_a} \left(\sin^{-1}\frac{\mathbf{e}_2^T\mathbf{v}^{\mathcal{B}}_a}{\Vert \mathbf{v}^{\mathcal{B}}_a\Vert}\right)  \nonumber\\
				&= \frac{1}{\sqrt{1 - \left(\frac{\mathbf{e}_2^T\mathbf{v}^{\mathcal{B}}_a}{\Vert \mathbf{v}^{\mathcal{B}}_a\Vert}\right)^2}} \frac{\Vert \mathbf{v}^{\mathcal{B}}_a\Vert\mathbf{e}_2^T - \mathbf{e}_2^T\mathbf{v}_a^{\mathcal{B}} \frac{\mathbf{v}_a^{\mathcal{B}^T}}{\Vert \mathbf{v}^{\mathcal{B}}_a\Vert}}{\Vert \mathbf{v}^{\mathcal{B}}_a\Vert^2}  \nonumber\\
				&= \frac{\mathbf{e}_2^T  }{V}
			\end{align}
			\label{e_p_V_alpha_beta_pvb}	
		\end{subequations}		
		and the aerodynamic coefficient gradients $\frac{\partial \mathbf{c}}{\partial \alpha}$ and $\frac{\partial \mathbf{c}}{\partial \beta}$ of an axially symmetric airframe satisfies (\ref{e_axial_symmetric}). Substituting (\ref{e_p_V_alpha_beta_pvb}) into (\ref{e_pfa_pvb_temp}), we have
		\begin{equation}
			\frac{\partial \mathbf{f}_a}{\partial \mathbf{v}_a^\mathcal{B}} = \frac{\rho S}{2}\left(2\mathbf{c} \mathbf{v}_a^{\mathcal{B}^T} + \frac{\partial \mathbf{c}}{\partial \alpha} \mathbf{v}_a^{\mathcal{B}^T} \lfloor \mathbf{e}_2 \rfloor  + V \frac{\partial \mathbf{c}}{\partial \beta}\mathbf{e}_2^T \right)
			\label{e_pfa_pvb_2}
		\end{equation}
	
		\section{Calculation of matrices $\dot{\mathbf{N}}$ and $\dot{\mathbf{h}}$}
		\label{app_dNH_dt}
		As the matrices $\mathbf{N}$ and $\mathbf{h}$ are broken into block matrices in (\ref{e_solve_omega}), their time derivatives can be taken in block matrices as follows: 
		\begin{equation}
			\dot{\mathbf{h}} = \begin{bmatrix}
			\dot{\mathbf{h}}_1 \\ \dot{\mathbf{h}}_2
			\end{bmatrix} \quad, \dot{\mathbf{N}} = \begin{bmatrix}
			\dot{\mathbf{N}}_1 \\ \dot{\mathbf{N}}_2
			\end{bmatrix}
			\label{e_dh_dN_block_mat}
		\end{equation}
		where
        \begin{subequations}            
			\begin{align}
				\dot{\mathbf{h}}_1 &= \frac{d}{dt} \left(\mathbf{y}^T_b \dot{\mathbf{v}}_a\right) \nonumber\\
				&= \frac{d}{dt} \left(\mathbf{e}_2^T\mathbf{R}^T \dot{\mathbf{v}}_a\right) \nonumber \\
				&= \mathbf{e}_2^T \left(- \lfloor \boldsymbol{\omega} \rfloor \mathbf{R}^T\dot{\mathbf{v}}_a + \mathbf{R}^T\ddot{\mathbf{v}}_a\right) \\
				\dot{\mathbf{h}}_2 &= \frac{d}{dt} \left(\ddot{\mathbf{v}} - \frac{1}{m}\mathbf{R}\frac{\partial \mathbf{f}_a}{\partial \mathbf{v}_a^\mathcal{B}}\mathbf{R}^T\dot{\mathbf{v}}_a\right) \nonumber\\
				&= \dddot{\mathbf{v}} - \frac{1}{m}\mathbf{R}\left( \lfloor \boldsymbol{\omega} \rfloor \frac{\partial \mathbf{f}_a}{\partial \mathbf{v}_a^\mathcal{B}}\mathbf{R}^T\dot{\mathbf{v}}_a + \frac{d}{dt} \left(\frac{\partial \mathbf{f}_a}{\partial \mathbf{v}_a^\mathcal{B}} \right) \mathbf{R}^T\dot{\mathbf{v}}_a\right. \nonumber\\
				& \qquad \left. - \frac{\partial \mathbf{f}_a}{\partial \mathbf{v}_a^\mathcal{B}} \lfloor \boldsymbol{\omega} \rfloor \mathbf{R}^T\dot{\mathbf{v}}_a + \frac{\partial \mathbf{f}_a}{\partial \mathbf{v}_a^\mathcal{B}}\mathbf{R}^T\ddot{\mathbf{v}}_a\right) \\
				\dot{\mathbf{N}}_1 &= \begin{bmatrix}
					0 & \dot{\mathbf{N}}_{12}
				\end{bmatrix} \\
				\dot{\mathbf{N}}_{12} &= \frac{d}{dt}\left(\mathbf{v}^T_a \mathbf{R}  \lfloor \mathbf{e}_2 \rfloor \right) = \left(\dot{\mathbf{v}}_a^T\mathbf{R} + \mathbf{v}_a^T\mathbf{R} \lfloor \boldsymbol{\omega} \rfloor  \right)  \lfloor \mathbf{e}_2 \rfloor  \\
				\dot{\mathbf{N}}_2 &= \begin{bmatrix}
					\dot{\mathbf{N}}_{21} & \dot{\mathbf{N}}_{22}
				\end{bmatrix} \\ 
				\dot{\mathbf{N}}_{21} &= \frac{d}{dt} \left(\mathbf{R}\mathbf{e}_1\right) = \mathbf{R} \lfloor \boldsymbol{\omega} \rfloor \mathbf{e}_1 \\
				\dot{\mathbf{N}}_{22} &= \frac{d}{dt} \left(\mathbf{R}\left(-\left \lfloor \left( a_T \mathbf{e}_1 + \frac{\mathbf f_a }{m} \right) \right\rfloor  + \frac{1}{m} \frac{\partial \mathbf{f}_a}{\partial \mathbf{v}_a^\mathcal{B}} \lfloor \mathbf{v}_a^\mathcal{B} \rfloor \right) \right) \nonumber\\
				&= \mathbf{R} \lfloor \boldsymbol{\omega} \rfloor \left(-\left \lfloor \left( a_T \mathbf{e}_1 + \frac{\mathbf f_a }{m} \right) \right\rfloor  + \frac{1}{m} \frac{\partial \mathbf{f}_a}{\partial \mathbf{v}_a^\mathcal{B}} \lfloor \mathbf{v}_a^\mathcal{B} \rfloor \right)  \nonumber \\
	            &\quad + \mathbf{R} \left(-\left \lfloor \left( \dot{a}_T \mathbf{e}_1 + \frac{1}{m}\left(\frac{\partial \mathbf f_a}{\partial V}\dot{V} + \frac{\partial \mathbf f_a}{\partial \alpha}\dot{\alpha}\right) \right) \right\rfloor   \right. \nonumber\\
				&\quad \left.  + \frac{1}{m} \left(\left(\frac{d}{dt} \left(\frac{\partial \mathbf{f}_a}{\partial \mathbf{v}_a^\mathcal{B}} \right)\right) \lfloor \mathbf{v}_a^\mathcal{B} \rfloor  + \frac{\partial \mathbf{f}_a}{\partial \mathbf{v}_a^\mathcal{B}}  \lfloor \dot{\mathbf{v}}_a^\mathcal{B} \rfloor  \right)  \right) \\
				\frac{\partial \mathbf{f}_a}{\partial V} &=  \rho V S \mathbf{c} \\
				\frac{\partial \mathbf{f}_a}{\partial \alpha} &=  \frac{1}{2} \rho V^2 S\frac{\partial \mathbf{c}}{\partial \alpha} \\
				\dot{\mathbf{v}}_a^\mathcal{B} &= \frac{d}{dt}\left(\mathbf{R}^T\mathbf{v}_a\right) = - \lfloor \boldsymbol{\omega} \rfloor \mathbf{R}^T\mathbf{v} + \mathbf{R}^T\dot{\mathbf{v}}_a \\
				\dot{\alpha} &= \frac{1}{1 + \left(\frac{\mathbf{v}_{a_z}^{\mathcal{B}}}{\mathbf{v}_{a_x}^\mathcal{B}}\right)^2} \frac{\dot{\mathbf{v}}_{a_z}^\mathcal{B} \mathbf{v}_{a_x}^\mathcal{B} - \mathbf{v}_{a_z}^\mathcal{B} \dot{\mathbf{v}}_{a_x}^\mathcal{B}}{\mathbf{v}_{a_x}^{\mathcal{B}^2}} \nonumber\\
                &= \frac{\dot{\mathbf{v}}_{a_z}^\mathcal{B} \mathbf{v}_{a_x}^\mathcal{B} - \mathbf{v}_{a_z}^\mathcal{B} \dot{\mathbf{v}}_{a_x}^\mathcal{B}}{V^2} \\
				\dot{\mathbf{v}}_a &= \dot{\mathbf{v}} - \dot{\mathbf{w}} \\
				\ddot{\mathbf{v}}_a &= \ddot{\mathbf{v}} - \ddot{\mathbf{w}} \\
				\dot{V} &= \mathbf{v}_a^T\dot{\mathbf v}_a / V 
			\end{align}
			\label{e_d_h_N}
        \end{subequations}
		and with (\ref{e_pfa_pvb_2}), we have
		\begin{align}
			\frac{d}{dt}\left(\frac{\partial \mathbf{f}_a}{\partial \mathbf{v}_a^\mathcal{B}}\right) &\!=\! \frac{\rho S}{2} \! \left(2\left(\frac{\partial \mathbf{c}}{\partial \alpha}\dot{\alpha} \mathbf{v}_a^{\mathcal{B}^T} \!+\! \mathbf{c} \dot{\mathbf{v}}_a^{\mathcal{B}^T}\right) \!+\! \left(\frac{\partial^2 \mathbf{c}}{\partial \alpha^2} \dot{\alpha} \mathbf{v}_a^{\mathcal{B}^T}  \right.  \right. \nonumber\\
			& \!\!\!\!\!\!\!\left.\left. \!+\! \frac{\partial \mathbf{c}}{\partial \alpha} \dot{\mathbf{v}}_a^{\mathcal{B}^T}\right)\!  \lfloor \mathbf{e}_2 \rfloor \!+\!\! \left(\dot{V} \frac{\partial \mathbf{c}}{\partial \beta} \!+\! V \frac{\partial^2 \mathbf c}{\partial \beta \partial \alpha} \dot \alpha\right)\!\mathbf{e}_2^T \right)
		\end{align}
				
		\section{Proof of Theorem \ref{theorem_det_N} (determinant of $\mathbf{N}$)} \label{app_det_N}
		We first denote
		\begin{equation}
			\boldsymbol{\Psi} = -\left \lfloor \left( a_T \mathbf{e}_1 + \frac{\mathbf f_a }{m} \right) \right\rfloor  + \frac{1}{m} \frac{\partial \mathbf{f}_a}{\partial \mathbf{v}_a^\mathcal{B}} \lfloor \mathbf{v}_a^\mathcal{B} \rfloor
			\label{e_psi_origin}
		\end{equation}
		
		With (\ref{e_decomp_force}), we have
		\begin{equation}
			a_T\mathbf{e}_1 + \frac{\mathbf{f}_a}{m} =  \|\dot{\mathbf{v}} - \mathbf{g} \| \begin{bmatrix}
				\cos(\gamma - \alpha) & 0 & -\sin(\gamma - \alpha)
			\end{bmatrix}^T
			\label{e_aTe1_fm}
		\end{equation}
		
		With (\ref{e_pfa_pvb}) and (\ref{e_axial_symmetric}), we have
		\begin{align}
			&\frac{\partial \mathbf{f}_a}{\partial \mathbf{v}_a^\mathcal{B}} \lfloor \mathbf{v}_a^\mathcal{B} \rfloor = \frac{\rho S}{2}\left(-\frac{\partial \mathbf{c}}{\partial \alpha}\mathbf{e}_2^T \lfloor \mathbf{v}_a^\mathcal{B} \rfloor^2 + V \frac{\partial \mathbf{c}}{\partial \beta} \mathbf e_2^T \lfloor \mathbf{v}_a^\mathcal{B} \rfloor \right) \nonumber\\
			&\qquad\qquad = \frac{\rho S V^2}{2} \begin{bmatrix}
				0 & \frac{\partial \mathbf{c}_x}{\partial \alpha}  & 0 \\
				\frac{\partial \mathbf{c}_y}{\partial \beta}\sin\alpha & 0 & -\frac{\partial \mathbf{c}_y}{\partial \beta}\cos\alpha \\
				0 & \frac{\partial \mathbf{c}_z}{\partial \alpha} & 0
			\end{bmatrix} 
			\label{e_pfa_pvb_vbskew}
		\end{align}
		
		Therefore, combining (\ref{e_aTe1_fm}) and (\ref{e_pfa_pvb_vbskew}), (\ref{e_psi_origin}) can be rewritten as
		\begin{equation}
			\boldsymbol{\Psi} = \begin{bmatrix}
				0 & \psi_{12}  & 0 \\
				\psi_{21} & 0 & \psi_{23} \\
				0 & \psi_{32} & 0
			\end{bmatrix}
			\label{e_Psi}
		\end{equation}
		where
		\begin{subequations}
			\begin{align}
				\psi_{12} &= -\|\dot{\mathbf{v}} - \mathbf{g} \|\sin(\gamma - \alpha) + \frac{\rho S V^2}{2m}\frac{\partial \mathbf{c}_x}{\partial \alpha} \\
				\psi_{21} &= \|\dot{\mathbf{v}} - \mathbf{g} \|\sin(\gamma - \alpha) + \frac{\rho S V^2}{2m}\frac{\partial \mathbf{c}_y}{\partial \beta}\sin\alpha \\
				\psi_{23} &= \|\dot{\mathbf{v}} - \mathbf{g} \|\cos(\gamma - \alpha) -\frac{\rho S V^2}{2m}\frac{\partial \mathbf{c}_y}{\partial \beta}\cos\alpha \label{e_psi_23}\\
				\psi_{32} &= -\|\dot{\mathbf{v}} - \mathbf{g} \|\cos(\gamma - \alpha) + \frac{\rho S V^2}{2m}\frac{\partial \mathbf{c}_z}{\partial \alpha}
			\end{align}
			\label{e_psi}
		\end{subequations}
		
		Now we calculate the determinant of $\mathbf{N}$. With (\ref{e_N1}) and (\ref{e_N2}), $\mathbf{N}$ can be factorized as
		\begin{equation}
			\mathbf{N} = \begin{bmatrix}
				1 & \mathbf{0} \\ \mathbf{0} & \mathbf{R}
			\end{bmatrix} \underbrace{\begin{bmatrix}
					0 & \mathbf{v}_a^{\mathcal{B}^T}\lfloor\mathbf{e}_2 \rfloor \\ \mathbf{e}_1 & \boldsymbol{\Psi}
			\end{bmatrix}}_{\bar{\mathbf{N}}}
		\end{equation}
		which implies $\det(\mathbf{N}) = \det(\widebar{\mathbf{N}})$. Performing elementary row and column operations on $\widebar{\mathbf{N}}$ produces
		\begin{align}
			\widebar{\mathbf{N}} \!=\! \begin{bmatrix}
				0 & -\mathbf{v}_{a_z}^\mathcal{B} & 0 & \mathbf{v}_{a_x}^\mathcal{B} \\
				1 & 0 & \psi_{12} & 0 \\
				0 & \psi_{21} & 0 & \psi_{23} \\
				0 & 0 & \psi_{32} & 0
			\end{bmatrix} \!\sim\! \begin{bmatrix}
				1 & 0 & 0 & 0 \\
				0 & \psi_{32} & 0 & 0 \\
				0 & 0 & -\mathbf{v}_{a_z}^\mathcal{B} & \mathbf{v}_{a_x}^\mathcal{B} \\
				0 & 0 & \psi_{21} & \psi_{23}
			\end{bmatrix}
			\label{e_sim_bar_N}
		\end{align}
		By substituting (\ref{e_psi}) into (\ref{e_sim_bar_N}), the determinant of ${\mathbf{N}}$ hence can be calculated as follows:
		\begin{align}	
			\det(\mathbf{{\mathbf{N}}}) &= \det({\widebar{\mathbf{N}}}) = -\psi_{32} \left( \mathbf{v}_{a_z}^\mathcal{B} \psi_{23} + \mathbf{v}_{a_x}^\mathcal{B} \psi_{21} \right) \nonumber\\
			&=  -\psi_{32} \| \mathbf{v}_a \| \left( \psi_{23}\sin\alpha +  \psi_{21}\cos\alpha \right) \nonumber \\
			&= -\psi_{32} \|\mathbf{v}_a\| \left( \| 
			\dot{\mathbf{v}} -\mathbf{g} \| \cos(\gamma -\alpha)\sin\alpha \right. \nonumber\\ 
			& \qquad \left. +  \| \dot{\mathbf{v}} -\mathbf{g} \| \sin(\gamma - \alpha) \cos\alpha \right) \nonumber\\
			&= -\psi_{32} \|\mathbf{v}_a\| \| \dot{\mathbf{v}} -\mathbf{g} \| \sin\gamma \nonumber \\
			&= -\psi_{32} \| \mathbf{v}_a \times (\dot{\mathbf{v}} -\mathbf{g}) \|
		\end{align}		
		It is noted that the derivative of (\ref{e_solve_alpha}) w.r.t. $\alpha$ is given as
		\begin{align}
			\frac{\partial F(\alpha)}{\partial \alpha} &= -\frac{2m \| \dot{\mathbf{v}} - \mathbf{g} \|}{\rho S V^2}  \cos(\gamma - \alpha) + \frac{\mathbf{c}_z(\alpha, 0)}{\partial \alpha} \nonumber\\
			&= \frac{2m}{\rho SV^2} \psi_{32}
			\label{e_equivalent_psi_32}
		\end{align}
		Therefore, the determinant of $\mathbf{N}$ is finally arrived at
		\begin{equation}
			\det(\mathbf{N}) = - \frac{\rho SV^2}{2m} \frac{\partial F(\alpha)}{\partial \alpha} \| \mathbf{v}_a \times (\dot{\mathbf{v}} -\mathbf{g}) \|
		\end{equation}
	
		\section{Singularity $\|\mathbf{v}_a\| = 0$}
		\label{app_sing_va0}	
		
		\subsection{Proof of Theorem \ref{theorem_sing_va0}: determinant of $\mathbf{N}$}
		\label{app_sing_v0_det_N}
		 With (\ref{e_N1_sing_va0}) and (\ref{e_N2_sing_va0}), $\mathbf{N}$ can be factorized as
		\begin{equation}
			\mathbf{N} = \begin{bmatrix}
				1 & \mathbf{0} \\ \mathbf{0} & \mathbf{R}
			\end{bmatrix} \underbrace{\begin{bmatrix}
					0 & \| \mathbf{z}_b^{\rm fix} \times (\dot{\mathbf{v}} - \mathbf{g}) \| \mathbf{e}_1^T \\ \mathbf{e}_1 & -a_T \lfloor\mathbf{e}_1\rfloor
			\end{bmatrix}}_{\bar{\mathbf{N}}}
		\end{equation}
		which implies $\det(\mathbf{N}) = \det(\widebar{\mathbf{N}})$. Performing elementary row and column operations on $\widebar{\mathbf{N}}$ produces
		\begin{equation}
			\widebar{\mathbf{N}} \sim {\rm diag}\left(\begin{bmatrix}
				1 & \| \mathbf{z}_b^{\rm fix} \times (\dot{\mathbf{v}} - \mathbf{g}) \|  & a_T & -a_T
			\end{bmatrix}\right)
		\end{equation}
		Hence, the determinant of $\mathbf{N}$ can be calculated as:
		\begin{equation}
			\det(\mathbf{N}) = - a_T^2 \| \mathbf{z}_b^{\rm fix} \times (\dot{\mathbf{v}} - \mathbf{g}) \| 
		\end{equation}
		
		\subsection{Calculation of $\dot{\mathbf{h}}$ and $\dot{\mathbf{N}}$}
		\label{app_sing_v0_dh_dN}
		As $\mathbf{h}$ and $\mathbf{N}$ break into block matrices, their derivatives $\dot{\mathbf{h}}$ and $\dot{\mathbf{N}}$ can be presented as like (\ref{e_dh_dN_block_mat}), where each block is calculated as follows:  
		\begin{subequations}
			\begin{align}
				&\dot{\mathbf{h}}_1 = \left(\lfloor \mathbf{z}_b^{\rm fix} \rfloor \dddot{\mathbf{v}}\right)^T\mathbf{z}_b + \left(\lfloor \mathbf{z}_b^{\rm fix} \rfloor \ddot{\mathbf{v}}\right)^T \mathbf{R} \lfloor \boldsymbol{\omega}\rfloor\mathbf{e}_3 \label{e_d_h1_sing}\\
				&\dot{\mathbf{h}}_2 = \dddot{\mathbf{v}} \\
				&\dot{\mathbf{N}}_1 \!=\! \begin{bmatrix}
					0 \!\!&\!  \frac{-(\dot{\mathbf{v}} \!-\! \mathbf{g})^T \lfloor \mathbf{z}_b^{\rm fix} \rfloor^2 \ddot{\mathbf{v}} \mathbf{e}_1^T}{\|\lfloor \mathbf{z}_b^{\rm fix} \rfloor (\dot{\mathbf{v}} \!-\! \mathbf{g})\|}
				\end{bmatrix} \label{e_d_N1_sing}\\
				&\dot{\mathbf{N}}_2 \!=\! \begin{bmatrix}				\mathbf{R}\lfloor\boldsymbol{\omega}\rfloor\mathbf{e}_1 \!\!&\! -\left(\frac{(\dot{\mathbf{v}} \!-\! \mathbf{g})^T\ddot{\mathbf{v}}}{a_T} \mathbf{R} \!+\! a_T\mathbf{R}\lfloor\boldsymbol{\omega}\rfloor\right) \lfloor \mathbf{e}_1 \rfloor		
				\end{bmatrix}
			\end{align}
		\end{subequations}

		\section{Singularity $\gamma = 0$}
		\label{app_sing_gamma0}	
		
		\subsection{Proof of Theorem \ref{theorem_sing_gamma0}: determinant of $\mathbf{N}$}
		\label{app_sing_gamma0_det_N}
		
		Because $\mathbf{y}_b$ is perpendicular to $\mathbf{v}_a$, so it still holds the lateral airspeed condition $\mathbf{y}_b^T\mathbf{v}_a = 0$. We can leverage the results in (\ref{e_Psi}), (\ref{e_psi}) and (\ref{e_equivalent_psi_32}) in Appendix \ref{app_det_N} to factorize $\det(\mathbf{N})$ as follows:
		\begin{equation}
			\mathbf{N} = \begin{bmatrix}
				1 & \mathbf{0} \\ \mathbf{0} & \mathbf{R}
			\end{bmatrix} \underbrace{\begin{bmatrix}
					0 & \|\mathbf{z}_b^{\rm fix} \times (\dot{\mathbf{v}} - \mathbf{g})\|\mathbf{e}_1^T \\ \mathbf{e}_1 & \boldsymbol{\Psi}
			\end{bmatrix}}_{\bar{\mathbf{N}}}
			\label{e_factorize_N_sing_gamma0}
		\end{equation}
		which implies $\det(\mathbf{N}) = \det(\widebar{\mathbf{N}})$. Performing elementary row and column operations on $\widebar{\mathbf{N}}$ produces
		\begin{equation}
			\widebar{\mathbf{N}} \sim {\rm diag}\left(\begin{bmatrix}
				1 & \|\mathbf{z}_b^{\rm fix} \times (\dot{\mathbf{v}} - \mathbf{g})\| & \psi_{23} & \psi_{32}
			\end{bmatrix}\right)
		\end{equation}
		where $\psi_{23} = \left(\|\dot{\mathbf{v}} - \mathbf{g}\| - \frac{\rho S V^2}{2m}\frac{\partial \mathbf{c}_y}{\partial \beta}\right)\cos\alpha$ is from (\ref{e_psi_23}) by setting $\gamma = 0$, and $\psi_{32} = \frac{\rho SV^2}{2m} \frac{\partial F(\alpha)}{\partial \alpha} $ is from (\ref{e_equivalent_psi_32}). Finally, the determinant of $\mathbf{N}$ is
		\begin{equation}
			\begin{aligned}
				\det(\mathbf{N}) 
				&= \frac{\rho SV^2}{2m} \frac{\partial F(\alpha)}{\partial \alpha} \|\mathbf{z}_b^{\rm fix} \times (\dot{\mathbf{v}} - \mathbf{g})\| \psi_{23}
			\end{aligned}		
		\end{equation}
	
		\subsection{Calculation of $\dot{\mathbf{h}}$ and $\dot{\mathbf{N}}$}
		\label{app_sing_gamma0_d_h_N}
		Similarly, the derivatives $\dot{\mathbf{h}}$ and $\dot{\mathbf{N}}$ can be presented as like (\ref{e_dh_dN_block_mat}), in which block matrices $\dot{\mathbf{h}}_1$ and $\dot{\mathbf{N}}_1$ are given by (\ref{e_d_h1_sing}) and (\ref{e_d_N1_sing}) in Appendix \ref{app_sing_v0_dh_dN}, while $\dot{\mathbf{h}}_2$ and $\dot{\mathbf{N}}_2$ are given by (\ref{e_d_h_N}).

        \section{Gradients of the flatness functions}
		\label{app_gradient_flat_func}
		We denote $\mathcal{P} = \begin{bmatrix}
			\mathbf v^T & \dot{\mathbf v}^T & \ddot{\mathbf v}^T
		\end{bmatrix}^T, \mathbf b = \begin{bmatrix}
			\dot a_T & \boldsymbol{\omega}^T
		\end{bmatrix}^T $ for simplicity. We also split the matrix $\mathbf h$ and $\mathbf N_2$ in (\ref{e_h_N}), (\ref{e_h_N_sing_va0}), (\ref{e_h_N_sing_gamma0}) into $\mathbf h = \begin{bmatrix} \mathbf h_1 & \mathbf h_2^T \end{bmatrix}^T$ and  $\mathbf N_2 = \begin{bmatrix} \mathbf N_{21} & \mathbf N_{22} \end{bmatrix}$, respectively.
	
		\subsection{When in coordinated flight}	
		\label{app_grad_coordinated_flight}
		The flatness functions are presented in Section \ref{sec_flat_func}, and the corresponding gradients are given as follows:
		\begin{subequations}
			\begin{align}
				\frac{\partial a_T}{\partial \mathcal{P}} &= \frac{\partial \Vert\dot{\mathbf{v}} - \mathbf{g}\Vert \cos\left(\gamma \!-\! \alpha\right) - \mathbf{f}_{a_x} / m}{\partial \mathcal{P}} \nonumber\\
				&= \frac{\partial \Vert\dot{\mathbf{v}} - \mathbf{g}\Vert}{\partial \mathcal{P}}\cos(\gamma - \alpha) - \Vert\dot{\mathbf{v}} - \mathbf{g}\Vert\sin(\gamma - \alpha) \nonumber\\
				&\quad  \left(\frac{\partial \gamma}{\partial \mathcal{P}} - \frac{\partial \alpha}{\partial \mathcal{P}} \right) - \frac{\mathbf{e}_1^T}{m}\frac{\partial \mathbf f_a}{\partial \mathcal P} \\
				\frac{\partial \boldsymbol{\omega}}{\partial \mathcal{P}} &= \begin{bmatrix}
					\mathbf{0}_{3\times 1} & \mathbf{I}_3
				\end{bmatrix} \mathbf{N}^{-1} \left(\frac{\partial \mathbf{h}}{\partial \mathcal{P}} - \frac{\partial \mathbf{N}}{\partial \mathcal{P}} \mathbf b\right)
			\end{align}
		\end{subequations}
		where
		\begin{subequations}
			\begin{align}
				&\frac{\partial \Vert\dot{\mathbf{v}} - \mathbf{g}\Vert}{\partial \mathcal{P}} = \frac{(\dot{\mathbf{v}} - \mathbf{g})^T}{\Vert\dot{\mathbf{v}} - \mathbf{g}\Vert} \frac{\partial \dot{\mathbf{v}}}{\partial \mathcal{P}} \\
				&\frac{\partial \mathbf f_a}{\partial \mathcal P} = \frac{\partial \mathbf f_a}{\partial \mathbf v_a^{\mathcal B}} \frac{\partial \mathbf v_a^{\mathcal B}}{\partial \mathcal P}\\
				&\frac{\partial \gamma}{\partial \mathcal{P}} = r \frac{\partial }{\partial \mathcal{P}} {\rm cos^{-1}}\left( \frac{(\dot{\mathbf{v}} - \mathbf{g})^T\mathbf{v}_a}{\|\dot{\mathbf{v}} - \mathbf{g}\| \|\mathbf{v}_a\|}\right)  \nonumber\\
				&\hspace{5.5mm}= \frac{r}{|\sin \gamma|}\left(\left( \frac{\cos\gamma(\dot{\mathbf{v}} -\mathbf{g})^T}{\|\dot{\mathbf{v}} -\mathbf{g}\|^2} - \frac{\mathbf{v}_a^T}{\| \dot{\mathbf{v}} \!-\! \mathbf{g}\| \| \mathbf{v}_a \|} \right) \frac{\partial \dot{\mathbf{v}}}{\partial \mathcal P} \right. \nonumber\\
				&\hspace{10mm} \left. + \left(\frac{\cos\gamma \mathbf{v}_a^T}{\|\mathbf v_a\|^2} - \frac{(\dot{\mathbf{v}} - \mathbf g)^T}{\| \dot{\mathbf{v}} \!-\! \mathbf{g}\| \| \mathbf{v}_a \|}\right) \frac{\partial \mathbf v}{\partial \mathcal P}\right) \\
				&\frac{\partial \alpha}{\partial \mathcal P} = \frac{ \frac{\partial h}{\partial \mathcal P} \sin(\gamma -\alpha) + h\cos(\gamma -\alpha)\frac{\partial \gamma}{\partial \mathcal P}}{h\cos(\gamma - \alpha) - \frac{\partial \mathbf{c}_z}{\partial \alpha}} \\
				&\frac{\partial h}{\partial \mathcal P} = \frac{2m}{\rho S} \left(\frac{(\dot{\mathbf{v}} \!-\! \mathbf{g})^T}{\|\dot{\mathbf v} - \mathbf g\|\|\mathbf{v}_a\|^2} \frac{\partial \dot{\mathbf v}}{\partial \mathcal P} \!-\! \frac{2\|\dot{\mathbf{v}} \!-\! \mathbf{g}\| \mathbf{v}_a^T}{\|\mathbf v_a\|^4}\frac{\partial \mathbf v}{\partial \mathcal P}\right)	\\
				&\frac{\partial \mathbf h}{\partial \mathcal P} = \begin{bmatrix}
					\frac{\partial \mathbf h_1}{\partial \mathcal P}^T & \frac{\partial \mathbf h_2}{\partial \mathcal P}^T
				\end{bmatrix}^T \\
				&\frac{\partial \mathbf h_1}{\partial \mathcal P} = \dot{\mathbf v}_a^T \frac{\partial \mathbf y_b}{\partial \mathcal P} + \mathbf y_b^T \frac{\partial \dot{\mathbf v}}{\partial \mathcal P} \\
				&\frac{\partial \mathbf h_2}{\partial \mathcal P} \!\!=\!\! \left.\frac{\partial \ddot{\mathbf v}}{\partial \mathcal P} \!-\! \frac{1}{m} \left(\!\frac{\partial \mathbf R \boldsymbol \xi}{\partial \mathcal P}\right\rvert_{\boldsymbol \xi = \frac{\partial \mathbf{f}_a}{\partial \mathbf{v}_a^\mathcal{B}}\mathbf{R}^T\dot{\mathbf{v}}_a} \!\!\!\!+\! \mathbf{R}\left.\frac{\!\partial\! \left(\frac{\partial \mathbf{f}_a}{\partial \mathbf{v}_a^\mathcal{B}\!}\right)\!\boldsymbol \xi}{\partial \mathcal P}\right\rvert_{\boldsymbol \xi = \mathbf{R}^T\dot{\mathbf{v}}_a} \right.\nonumber\\
				&\hspace{10mm} \left. + \mathbf R\frac{\partial \mathbf{f}_a}{\partial \mathbf{v}_a^\mathcal{B}}\left.\frac{\partial \mathbf R^T     \boldsymbol \xi}{\partial \mathcal P}\right\rvert_{\boldsymbol \xi = \dot{\mathbf{v}}_a} + \mathbf{R}\frac{\partial \mathbf{f}_a}{\partial \mathbf{v}_a^\mathcal{B}}\mathbf{R}^T\frac{\partial \dot{\mathbf{v}}}{\partial \mathcal P}\right)\\
				&\frac{\partial \mathbf R \boldsymbol \xi}{\partial \mathcal P} = \boldsymbol \xi_1 \frac{\partial \mathbf x_b}{\partial \mathcal P} + \boldsymbol \xi_2 \frac{\partial \mathbf y_b}{\partial \mathcal P} + \boldsymbol \xi_3 \frac{\partial \mathbf z_b}{\partial \mathcal P}\\
				&\frac{\partial \mathbf R^T \boldsymbol \xi}{\partial \mathcal P} = \begin{bmatrix}
					\left(\boldsymbol \xi^T \frac{\partial \mathbf x_b}{\partial \mathcal P}\right)^T & \left(\boldsymbol \xi^T \frac{\partial \mathbf y_b}{\partial \mathcal P}\right)^T & \left(\boldsymbol \xi^T \frac{\partial \mathbf z_b}{\partial \mathcal P}\right)^T
				\end{bmatrix}^T\\
				&\frac{\partial \mathbf y_b}{\partial \mathcal P} = r
				\frac{\|\lfloor\mathbf{v}_a\rfloor  \left(\dot{\mathbf{v}} \!-\! \mathbf{g}\right)\|^2\mathbf I_3 \!-\! \lfloor\mathbf{v}_a\rfloor  \left(\dot{\mathbf{v}} \!-\! \mathbf{g}\right)\left(\dot{\mathbf{v}} \!-\! \mathbf{g}\right)^T\lfloor\mathbf{v}_a\rfloor}{\|\lfloor\mathbf{v}_a\rfloor \left(\dot{\mathbf{v}} \!-\! \mathbf{g}\right)\|^3} \nonumber\\
				&\hspace{10mm} \frac{\partial \left(\lfloor\mathbf{v}_a\rfloor \left(\dot{\mathbf{v}} - \mathbf{g}\right)\right)}{\partial \mathcal P} \\
				&\frac{\partial \left(\lfloor\mathbf{v}_a\rfloor \left(\dot{\mathbf{v}} - \mathbf{g}\right)\right)}{\partial \mathcal P} = -\lfloor\dot{\mathbf v} - \mathbf g\rfloor \frac{\partial \mathbf v}{\partial \mathcal P} + \lfloor\mathbf v_a\rfloor \frac{\partial \dot{\mathbf v}}{\partial \mathcal P} \\
				&\frac{\partial \mathbf x_b}{\partial \mathcal P} = \lfloor\mathbf y_b \rfloor {\rm Exp}(\alpha \mathbf y_b)\frac{\mathbf v_a}{\|\mathbf v_a\|} \frac{\partial \alpha}{\partial \mathcal P} + \left.\frac{\partial {\rm Exp}(\alpha \mathbf y_b)\boldsymbol \xi}{\partial \mathbf y_b}\right\rvert_{\boldsymbol \xi = \frac{\mathbf v_a}{\|\mathbf v_a\|}}   \nonumber\\
				&\hspace{10mm} + {\rm Exp}(\alpha \mathbf y_b) \frac{\|\mathbf v_a\|^2\mathbf I_3 - \mathbf v_a \mathbf v_a^T}{\|\mathbf v_a\|^3} \frac{\partial \mathbf v}{\partial \mathcal P}\\
				&\frac{\partial {\rm Exp}(\alpha \mathbf y_b) \boldsymbol \xi}{\partial \mathbf y_b} \!=\! \frac{\partial }{\partial \mathbf y_b} \!\!\left(\mathbf I_3 \!+\! \lfloor\mathbf y_b\rfloor \sin\alpha \!+\! \lfloor\mathbf y_b\rfloor^2(1\!-\!\cos\alpha) \right)\boldsymbol \xi \nonumber\\
				&\hspace{5.5mm} = -\lfloor\boldsymbol \xi\rfloor\sin\alpha - \left(\lfloor \lfloor\mathbf y_b\rfloor \boldsymbol \xi\rfloor + \lfloor\mathbf y_b\rfloor\lfloor\boldsymbol \xi \rfloor \right) (1- \cos\alpha) \\
				&\frac{\partial \mathbf z_b}{\mathcal P} = \lfloor\mathbf x_b\rfloor \frac{\partial \mathbf y_b}{\partial \mathcal P} - \lfloor\mathbf y_b\rfloor \frac{\partial \mathbf x_b}{\partial \mathcal P}\\
                &\frac{\partial}{\partial \mathcal P} \left(\frac{\partial \mathbf f_a}{\partial \mathbf v_a^{\mathcal B}} \boldsymbol \xi\right) = \frac{\rho S}{2} \left(\left(2\frac{\partial \mathbf{c}}{\alpha} \mathbf{v}_a^{\mathcal{B}^T} + \frac{\partial^2 \mathbf{c}}{\partial \alpha^2} \mathbf{v}_a^{\mathcal{B}^T} \lfloor \mathbf{e}_2 \rfloor  \right.\right.\nonumber\\
				&\hspace{5.5mm} \left.\left. + V \frac{\partial^2 \mathbf{c}}{\partial \beta \partial \alpha}\mathbf{e}_2^T \right)\boldsymbol \xi \frac{\partial \alpha}{\partial \mathcal P} + \left( 2\mathbf{c} \boldsymbol \xi^T - \frac{\partial \mathbf{c}}{\partial \alpha} \boldsymbol \xi^T \lfloor \mathbf{e}_2 \rfloor \right.\right.\nonumber\\
				&\hspace{5.5mm} \left.\left. + \frac{\partial \mathbf{c}}{\partial \beta}\frac{\mathbf{e}_2^T \boldsymbol \xi \mathbf v_a^{\mathcal B^T}}{V}\right) \frac{\partial \mathbf v_a^{\mathcal B}}{\partial \mathcal P} \right) \\
				&\frac{\partial \mathbf v_a^{\mathcal B}}{\partial \mathcal P} =  \left.\frac{\partial \mathbf R^T \boldsymbol{\xi}}{\partial \mathcal P}\right\rvert_{\boldsymbol \xi = \mathbf v_a} + \mathbf R^T \frac{\partial \mathbf v_a}{\partial \mathcal P}\\
				&\frac{\partial \mathbf N \mathbf b}{\partial \mathcal P} = \begin{bmatrix}
					\mathbf b^T \frac{\partial \mathbf N_1^T}{\partial \mathcal P} \\
					\mathbf b_1 \frac{\partial \mathbf N_{21}}{\partial \mathcal P} + \sum_{i = 2}^4 \left(\mathbf b_i \frac{\partial \mathbf N_{22}\mathbf e_{i-1}}{\partial \mathcal P}\right)
				\end{bmatrix} \\
				&\frac{\partial \mathbf N_1^T}{\partial \mathcal P} = \begin{bmatrix}
					\mathbf 0 & \left(-\lfloor\mathbf e_2\rfloor \frac{\partial \mathbf v_a^{\mathcal B}}{\partial \mathcal P}\right)^T
				\end{bmatrix}^T \\
			 	&\frac{\partial \mathbf N_{21}}{\partial \mathcal P} = \frac{\partial \mathbf x_b}{\partial \mathcal P} \\
			 	&\frac{\partial \mathbf N_{22}\mathbf e_j}{\partial \mathcal P} = \frac{\partial \mathbf R}{\partial \mathcal P} \mathbf R^T\mathbf N_{22}\mathbf e_j + \mathbf R\left( -\lfloor\mathbf e_1\rfloor\mathbf e_j \frac{\partial a_T}{\partial \mathcal P} + \right. \nonumber\\
			 	&\hspace{5.5mm} \left. \frac{1}{m}\left(\lfloor\mathbf e_j\rfloor \frac{\partial \mathbf f_a}{\partial \mathcal P} + \frac{\partial^2 \mathbf f_a}{\partial \mathbf v_a^{\mathcal B} \partial \mathcal P} \lfloor\mathbf v_a^{\mathcal B}\rfloor \mathbf e_j - \frac{\partial \mathbf f_a}{\partial \mathbf v_a^{\mathcal B}} \lfloor\mathbf e_j\rfloor \frac{\partial \mathbf v_a^{\mathcal B}}{\partial \mathcal P}\right) \right)
			\end{align} 
		\end{subequations}
		
		\subsection{When in singularity condition $\|\mathbf v_a\| = 0 $}
		\label{app_grad_sing_va0}
		The flatness functions are rewritten in Section \ref{sec_sing_va0} when $\|\mathbf v_a\| = 0 $. The corresponding modified gradients that are different from Appendix \ref{app_grad_coordinated_flight} are given as follows:
		\begin{subequations}
			\begin{align}
				&\frac{\partial a_T}{\partial \mathcal P} = \frac{\partial \|\dot{\mathbf v}-\mathbf g\|}{\partial \mathcal P}\\
				&\frac{\partial \mathbf x_b}{\partial \mathcal P} = \frac{\|\dot{\mathbf v}-\mathbf g\|^2\mathbf I_3 - (\dot{\mathbf v}-\mathbf g)(\dot{\mathbf v}-\mathbf g)^T}{\|\dot{\mathbf v}-\mathbf g\|^3} \frac{\partial \dot{\mathbf v}}{\partial \mathcal P} \\ 
				&\frac{\partial \mathbf y_b}{\partial \mathcal P} = \frac{\|\lfloor \mathbf z_b^{\rm fix}\rfloor(\dot{\mathbf v}-\mathbf g)\|^2 \mathbf I_3 - (\lfloor \mathbf z_b^{\rm fix}\rfloor(\dot{\mathbf v}-\mathbf g))(\lfloor \mathbf z_b^{\rm fix}\rfloor(\dot{\mathbf v}-\mathbf g))^T}{\|\lfloor \mathbf z_b^{\rm fix}\rfloor(\dot{\mathbf v}-\mathbf g)\|^3} \nonumber\\
				&\hspace{10mm}  \lfloor \mathbf z_b^{\rm fix}\rfloor \frac{\partial \dot{\mathbf v}}{\partial \mathcal P} \\
				&\frac{\partial \mathbf h_1}{\partial \mathcal P} = \left(\lfloor \mathbf{z}_b^{\rm fix} \rfloor \ddot{\mathbf{v}}\right)^T\frac{\partial \mathbf{z}_b}{\partial \mathcal P} + \mathbf{z}_b^T\lfloor \mathbf{z}_b^{\rm fix} \rfloor \frac{\partial \ddot{\mathbf v}}{\partial \mathcal P} \\ 
				&\frac{\partial \mathbf h_2}{\partial \mathcal P} = \frac{\partial \ddot{\mathbf v}}{\partial \mathcal P} \\
				&\frac{\partial \mathbf N_1^T}{\partial \mathcal P} \begin{bmatrix}
					\mathbf 0 & \left(-\mathbf e_1\frac{(\dot{\mathbf{v}} - \mathbf{g})^T \lfloor\mathbf{z}_b^{\rm fix}\rfloor^2}{\|\lfloor\mathbf{z}_b^{\rm fix}\rfloor (\dot{\mathbf{v}} - \mathbf{g})\|} \frac{\partial \dot{\mathbf v}}{\partial \mathcal P}\right)^T
				\end{bmatrix}^T \\
				&\frac{\partial \mathbf N_{22}\mathbf e_j}{\partial \mathcal P} = -a_T\frac{\partial \mathbf R}{\partial \mathcal P} \lfloor \mathbf{e}_1 \rfloor \mathbf e_j + \mathbf R \lfloor \mathbf{e}_1 \rfloor \mathbf e_j \frac{\partial a_T}{\partial \mathcal P}
			\end{align}
		\end{subequations} 
	
		\subsection{When in singularity condition $|\gamma| = 0 $}
		The flatness functions are rewritten in Section \ref{sec_sing_va0} when $|\gamma| = 0 $. Since $a_T$, $\mathbf h_2$ and $\mathbf N_2$ are the same as those presented in Section \ref{sec_flat_func}, while $\mathbf h_1$ and $\mathbf N_1$ are the same as those in Section \ref{sec_sing_va0}, their gradients are the identical to those given respectively in Appendix \ref{app_grad_coordinated_flight} and \ref{app_grad_sing_va0}. 

        \section{Proof of Theorem \ref{theorem_error_state} (the error-state dynamics)}
        \label{app_error_sys}
        The dynamics of (\ref{e_d_delta_p}) and (\ref{e_d_delta_v}) simply take the time derivative to (\ref{e_delta_p}) and (\ref{e_delta_v}), respectively. Denoting $\boldsymbol{\theta} = {\rm Log}(\mathbf{R})$, the exponential map holds $\dot{\boldsymbol{\theta}} = \mathbf{A}^T(\boldsymbol{\theta}) \boldsymbol{\omega}$, where $\boldsymbol{\omega} = (\mathbf{R}^T\dot{\mathbf{R}})^\vee$, $(\cdot)^\vee$ the inverse of $\lfloor \cdot \rfloor$ that maps a skew-symmetric matrix to a vector, and $\mathbf{A}(\cdot)$ denotes the Jacobian of the exponential coordinates of $SO(3)$ \citep{bullo1995proportional}:
        \begin{equation}
            \mathbf{A}(\mathbf{\boldsymbol{\theta}}) \!=\! \mathbf{I}_3 \!+\! \left(\frac{1 \!-\! \cos\Vert\boldsymbol{\theta}\Vert}{\Vert\boldsymbol{\theta}\Vert}\right) \! \frac{\lfloor\boldsymbol{\theta}\rfloor}{\Vert\boldsymbol{\theta}\Vert} \!+\! \left(1 \!-\! \frac{\sin\Vert\boldsymbol{\theta}\Vert}{\Vert\boldsymbol{\theta}\Vert} \right) \! \frac{\lfloor\boldsymbol{\theta}\rfloor^2}{\Vert\boldsymbol{\theta}\Vert^2}
        \end{equation}
        By substituting (\ref{e_delta_p}) into the above rules, we have
        \begin{align}
            \delta \dot{\boldsymbol{\theta}} &= \mathbf{A}^T (\delta \boldsymbol{\theta})\left( \left( \mathbf{R}^T\mathbf{R}_d \right)^T \frac{d}{dt} \left( \mathbf{R}^T\mathbf{R}_d \right) \right)^\vee \nonumber\\
            &=\mathbf{A}^T (\delta \boldsymbol{\theta}) \left(\mathbf{R}_d^T\mathbf{R} \left(-\lfloor\boldsymbol{\omega}\rfloor \mathbf{R}^T\mathbf{R}_d + \mathbf{R}^T\mathbf{R}_d \lfloor\boldsymbol{\omega}_d \rfloor\right)\right)^\vee \nonumber\\      
            &=\mathbf{A}^T(\delta\boldsymbol{\theta}) \left(-\mathbf{R}_d^T\mathbf{R}\boldsymbol{\omega} + \boldsymbol{\omega}_d \right)
        \end{align}
        which is the error attitude dynamics in (\ref{e_d_delta_R}). 

	\section{Proof of Lemma \ref{lemma_linear_error_sys} (the linearized  error-state dynamics)}
	\label{app_linear_error_sys}
		The position error dynamics in (\ref{e_d_delta_p}) is linear, and the velocity error dynamics in (\ref{e_d_delta_v}) can be linearized along the reference trajectory. Specifically, since $\mathbf{R}^T\mathbf{R}_d = {\rm Exp}(\delta \boldsymbol{\theta}) \approx \mathbf{I}_3 + \lfloor \delta \boldsymbol{\theta}\rfloor$, (\ref{e_error_state_def}) implies:
		\begin{equation}
			\label{e_vel_err_dyn_tmp}	
			\begin{aligned}
				\delta \dot{\mathbf{v}}
				&= \left({a}_{T_d}\mathbf{R}_d\mathbf{e}_1 + \frac{1}{m}\mathbf{R}_d \mathbf{f}_{a_d}\right) - \left({a}_{T}\mathbf{R}\mathbf{e}_1 + \frac{1}{m}\mathbf{R} \mathbf{f}_{a}\right) \\
				&\approx \left({a}_{T_d}\mathbf{R}_d - ({a}_{T_d} - \delta a_T)\mathbf{R}_d (\mathbf{I}_3 +  \lfloor \delta \boldsymbol{\theta} \rfloor )^T\right) \mathbf{e}_1 \\
				&\qquad + \frac{1}{m}\mathbf{R}_d \left(\mathbf{f}_{a_d} - (\mathbf{I}_3 +  \lfloor \delta \boldsymbol{\theta} \rfloor )^T (\mathbf{f}_{a_d} - \delta \mathbf{f}_a) \right) \\
				&\approx \mathbf{R}_d\mathbf{e}_1 \delta a_T - \mathbf{R}_d \left(a_{T_d} \lfloor \mathbf{e}_1 \rfloor  +   \lfloor \frac{\mathbf{f}_{a_d}}{m} \rfloor \right)  \delta \boldsymbol{\theta} + \mathbf{R}_d \frac{\delta \mathbf{f}_a}{m} 
			\end{aligned}
		\end{equation}
		where 
        \begin{subequations}
			\begin{align}
				\delta \mathbf{f}_a &= \mathbf{f}_{a_d} - \mathbf{f}_a \approx \frac{\partial \mathbf{f}_{a_d}}{\partial \mathbf{v}_{a_d}^\mathcal{B}} \delta\mathbf{v}_a^\mathcal{B}
				\label{e_delta_f_a} \\
				\delta \mathbf{v}^{\mathcal{B}}_a & = \mathbf v_{a_d}^{\mathcal{B}} - \mathbf v_a^{\mathcal{B}} = \mathbf{R}_d^T\mathbf{v}_{a_d} - \mathbf{R}^T\mathbf{v}_a \nonumber\\
				&\approx \mathbf{R}_d^T\mathbf{v}_{a_d} - (\mathbf{I}_3 +  \lfloor \delta \boldsymbol{\theta} \rfloor ) \mathbf{R}_d^T (\mathbf{v}_{a_d} - \delta \mathbf{v}_a ) \nonumber\\
				&\approx  -  \lfloor \delta \boldsymbol{\theta} \rfloor  \mathbf{R}_d^T \mathbf{v}_{a_d} + \mathbf{R}_d^T \delta \mathbf{v}_a  \nonumber\\
				&\approx  \lfloor \mathbf{v}^{\mathcal{B}}_{a_d} \rfloor  \delta \boldsymbol{\theta} + \mathbf{R}_d^T \delta \mathbf{v}_a 
				\label{e_delta_v_ab}
			\end{align}
        \end{subequations}
        and the partial derivative $\partial \mathbf{f}_{a_d} / \partial \mathbf{v}_{a_d}^\mathcal{B}$ in (\ref{e_delta_f_a}) is given in (\ref{e_pfa_pvb_2}): 
        \begin{equation}
            \frac{\partial \mathbf{f}_{a_d}}{\partial \mathbf{v}_{a_d}^\mathcal{B}} = \left. \frac{\partial \mathbf{f}_{a}}{\partial \mathbf{v}_{a}^\mathcal{B}} \right|_{ \mathbf v_{a_d}^{\mathcal{B}}} \quad \quad \text{(see Equation (\ref{e_pfa_pvb_2}))} 
        \end{equation}
        
        In (\ref{e_delta_v_ab}), $\mathbf v_a = \mathbf v - \mathbf w$ is the actual air velocity and $\mathbf v_{a_d} = \mathbf v_d - \widebar{\mathbf w}$ is the air velocity used to calculate the reference trajectory. Hence, 
        \begin{subequations}
                    \begin{align}
                        \delta \mathbf v_a = \mathbf v_a - \mathbf v_{a_d} = \delta \mathbf v - \delta \mathbf w;\quad \delta \mathbf w = \mathbf w - \bar{\mathbf w}
                    \end{align}
        \end{subequations}
        and
        \begin{subequations}
                    \begin{align}
                        \delta \mathbf{f}_a & \approx \frac{\partial \mathbf{f}_{a_d}}{\partial \mathbf{v}_{a_d}^\mathcal{B}} \left( \lfloor \mathbf{v}^{\mathcal{B}}_{a_d} \rfloor  \delta \boldsymbol{\theta} + \mathbf{R}_d^T (\delta \mathbf{v} - \delta \mathbf{w}) \right)
                    \end{align}
                    \label{e_delta_v_ab_w}
        \end{subequations}
        By substituting (\ref{e_delta_f_a}) and  (\ref{e_delta_v_ab_w}) into (\ref{e_vel_err_dyn_tmp}), the velocity error dynamics can be given by
		\begin{equation}
			\delta \dot{\mathbf{v}} = \mathbf{M}_{T} \delta a_T + \mathbf{M}_\mathbf{v} \delta \mathbf{v} + \mathbf{M}_\mathbf{R} \delta \boldsymbol{\theta} + \mathbf{M}_\mathbf{w} \mathbf{w}
		\end{equation}
		where
        \begin{subequations}
			\begin{align}
				\mathbf{M}_{T} &= \mathbf{R}_d \mathbf{e}_1 \\
				\mathbf{M}_\mathbf{v} &= \frac{1}{m} \mathbf{R}_d \frac{\partial \mathbf{f}_{a_d}}{\partial \mathbf{v}_{a_d}^\mathcal{B}} \mathbf{R}_d^T \\
				\mathbf{M}_\mathbf{R} &= \mathbf{R}_d  \left(- a_{T_d} \lfloor \mathbf{e}_1 \rfloor  -  \lfloor \frac{\mathbf{f}_{a_d}}{m} \rfloor  + \frac{\partial \mathbf{f}_{a_d}}{\partial \mathbf{v}_{a_d}^\mathcal{B}}  \lfloor \frac{\mathbf{v}^{\mathcal{B}}_{a_d}}{m} \rfloor \right)  \\
				\mathbf{M}_\mathbf{w} &= -\frac{1}{m}\mathbf{R}_d \frac{\partial \mathbf{f}_{a_d}}{\partial \mathbf{v}_{a_d}^\mathcal{B}}\mathbf{R}_d^T
			\end{align}
        \end{subequations}
		
		To linearize the attitude error dynamics, we substitute (\ref{e_delta_R}) in (\ref{e_d_delta_R}) and approximate $\mathbf{A}(\delta \mathbf{R}) \approx \mathbf{I}_3$. Thus we have
		\begin{align}
			\dot{\delta \mathbf{R}} &\approx  -\mathbf{R}_d^T\mathbf{R}\boldsymbol{\omega} + \boldsymbol{\omega}_d \nonumber\\
			&\approx -\left(\mathbf{I} + \lfloor \delta\mathbf{R} \rfloor \right)^T \left(\boldsymbol{\omega}_d - \delta\boldsymbol{\omega}\right) + \boldsymbol{\omega}_d\nonumber\\
			&\approx  \delta\boldsymbol{\omega} -  \lfloor \boldsymbol{\omega}_d \rfloor  \delta \mathbf{R} 
		\end{align}

	\end{appendices}
	
\end{document}